\PassOptionsToPackage{table,dvipsnames}{xcolor}
\PassOptionsToPackage{nameinlink,capitalise}{cleveref}

\documentclass[runningheads]{llncs}

\usepackage{eccv}
\usepackage{amsmath,amsfonts,bm}

\def\eqref#1{equation~\ref{#1}}
\def\1{\bm{1}}

\def\vtheta{{\bm{\theta}}}

\def\ve{{\bm{e}}}

\def\vx{{\bm{x}}}

\DeclareMathAlphabet{\mathsfit}{\encodingdefault}{\sfdefault}{m}{sl}
\SetMathAlphabet{\mathsfit}{bold}{\encodingdefault}{\sfdefault}{bx}{n}

\def\gX{{\mathcal{X}}}
\def\gY{{\mathcal{Y}}}

\DeclareMathOperator*{\argmin}{arg\,min}

\usepackage{eccvabbrv}

\usepackage{amsmath,amsfonts,amssymb}
\usepackage{bm}
\usepackage{url}
\usepackage{booktabs}
\usepackage{multirow}
\usepackage{graphicx}
\usepackage{caption}
\usepackage{subcaption}
\usepackage{algorithm}
\usepackage{algpseudocode}
\usepackage{tikz}
\usepackage{pifont}
\usepackage{wrapfig}
\usepackage{placeins}

\newcommand{\method}{\textsc{ReTeX}}

\usepackage[accsupp]{axessibility}

\definecolor{linkred}{RGB}{190,40,40}
\definecolor{citegreen}{RGB}{0,140,90}
\definecolor{urlmagenta}{RGB}{200,0,130}

\usepackage[
    colorlinks=true,
    linkcolor=linkred,
    citecolor=citegreen,
    urlcolor=urlmagenta
]{hyperref}

\usepackage{orcidlink}

\titlerunning{Learning to Recover Task Experts}
\authorrunning{J. Jung et al.}

\begin{document}

% ===============================================================
% Main Paper
% ===============================================================

\title{Learning to Recover Task Experts from a Multi-Task Merged Model}

\author{
Jinwook Jung\inst{1}
\and Taegyu Kim\inst{1}
\and Kumju Jo\inst{1}
\and Sungyong Baik\inst{1,2}\thanks{Corresponding author.}
}

\institute{
$^1$Department of Artificial Intelligence,
$^2$Department of Data Science\\
Hanyang University, Republic of Korea\\
\email{
jjw970517@hanyang.ac.kr,
sa090180@gmail.com,
juice0630@hanyang.ac.kr,
dsybaik@hanyang.ac.kr
}
}

\maketitle

\begin{abstract}

Multi-task model merging aims to consolidate several task-specific experts into a unified model, yet static merging consistently suffers from parameter interference. 
% From the perspective of experts, this interference introduces parameter offsets that leave a persistent performance gap compared to original experts. 
% While dynamic merging can bridge this gap, it relies on the costly storage and loading of redundant expert components at inference. 
While dynamic merging models aim to bridge this gap, many works rely on the costly storage and loading of redundant expert components at inference. 
% In this work, we propose \textbf{Re}cover \textbf{T}ask \textbf{eX}pert (\textbf{\method}), a framework that shifts the paradigm from expert composition to expert recovery. 
% In this work, we propose \textbf{Re}cover \textbf{T}ask \textbf{eX}pert (\textbf{\method}), a framework that shifts the attention to expert recovery. 
In this work, from the perspective of task expert, we view parameter interference as parameter perturbation introduced to each expert during merging process.
We show that such parameter perturbations can be modeled as affine transformation, which can be approximated as additive offsets.
Motivated by these, we propose \textbf{Re}cover \textbf{T}ask \textbf{eX}pert (\textbf{\method{}}), a framework that predicts those offsets, in order to \textit{undo} parameter interference and recover task-expert performance from a single merged checkpoint.
%shifts the attention to expert recovery. 
% We mathematically derive that the merged-to-expert mapping is an affine transformation where the scaling factor typically collapses to identity, motivating a lightweight, offset-only recovery rule. 
% By predicting task-conditioned low-rank parameter offsets, \method ``\textit{undoes}'' the interference inherent in the merged model without storing original expert weights.
% ReTeX trains a compact recovery module using parameter supervision only: given any task ID, it predicts a task-conditioned parameter offset that removes interference from the merged model, without requiring task datasets or input–output examples during ReTeX training.
% By predicting task-conditioned, low-rank parameter offsets, ReTeX undoes the interference inherent in the merged model without storing original expert weights. 
% Furthermore, to enable ReTeX to identify a task and thus recover corresponding expert for a given input, we introduce a router-free identifier based on SVD subspace projection residuals. 
% Furthermore, to enable ReTeX to identify a task and thus recover corresponding expert for a given input, we introduce a router-free identifier based on SVD subspace signatures computed before model merging process.
% At inference, we select the task whose subspace yields the smallest projection residual for an input-dependent signal, avoiding router training and mitigating data storage and privacy concerns.
To recover the appropriate expert when task identity is unknown, we introduce a router-free task identifier based on SVD subspace signatures computed offline before inference.
%from existing expert parameter updates (e.g., $\Delta\theta_t$), avoiding labeled router training and retaining no task data.
At inference, the identifier selects the task whose subspace yields the smallest projection residual for a given input.
%n input-dependent signal.
%derived from~[activations / features / probes].
%
%
%projection residuals. 
% Crucially, we show that ReTeX enables emergent expert generation for out-of-distribution (OOD) tasks. 
% By relaxing task-identity constraints to soft probabilistic distributions, ReTeX adaptively interpolates seen expert knowledge to handle unseen tasks—even when their corresponding experts were unavailable during the merging process.
As a result, \method{} recovers over 95\% of individual-expert performance in both vision and NLP domains, while significantly improving generalization to unseen tasks.
% ReTeX recovers over 99\% of individual-expert performance across 30 tasks in both vision and NLP domains, while significantly improving generalization to unseen tasks.
Crucially, we also show that the parameter offset prediction leads to emergent adaptive interpolation of expert knowledge for out-of-distribution (OOD) tasks. 
\method{} adaptively interpolates seen expert knowledge to handle unseen tasks.
Our code is available at \url{https://github.com/BAIKLAB/ReTeX}


\end{abstract}

\section{Introduction}
Since the advent of foundation models~\cite{achiam2023gpt, saab2024capabilities, ding2023parameter} pre-trained on large-scale data, deep learning has achieved remarkable success across diverse domains by fine-tuning such large pre-trained models for individual downstream tasks.
This paradigm naturally yields a growing collection of task-specific experts.
However, storing an expert for each task separately is costly to deploy at scale due to storage, memory traffic, and management overhead.

Multi-task model merging~\cite{ilharco2023editing, yadav2023ties, yang2024adamerging} has emerged as a promising solution to consolidate multiple experts into a single model by directly manipulating their parameters.
Most early approaches construct a single merged model by weighted combinations of parameters or updates, where task coefficients are often selected by heuristics or hyperparameter tuning~\cite{ilharco2023editing, jin2023dataless, matena2022fisher, yadav2023ties, yang2024adamerging, tang2023concrete, yu2024language}.
However, static merging methods frequently underperform task experts due to parameter interference: task-specific updates collide in weight space, producing model parameters that are suboptimal for multiple tasks.
%deviating from the parameters of task experts.
%producing a single checkpoint that underperforms on many tasks.
% However, such static merging methods struggle to match the performance of the original experts due to parameter interference: task-specific updates collide in weight space, producing a single checkpoint that underperforms on many tasks.
% Several works have circumvented the challenge by finetuning coefficients and thus producing a separate merged model for each task, which however leads to requirement for task identity of each input and storage overhead.

Recently, dynamic merging methods~\cite{oh2025dawin, tang2024merging, lu2024twin, muqeeth2023soft} have been proposed to improve accuracy by selecting or composing task-specific components conditioned on each input, but typically require storing and reading additional task-specific parameters at inference.
%and may incur extra computation for input-dependent routing or coefficients. 
Moreover, several existing methods assume that the task identity is known at inference~\cite{ilharco2023editing, yadav2023ties, huang2024emr},  or train a separate router using substantial labeled task data~\cite{lu2024twin, tang2024merging, ye2025dynamic}, which can introduce additional storage overhead and privacy issues.
\iffalse
To narrow this gap, recent works have shifted attention to dynamic model merging~\cite{oh2025dawin, tang2024merging, lu2024twin, muqeeth2023soft}.
These methods dynamically compose task-specific components (e.g., experts, masks, or task-conditioned branches) based on each input, substantially improving accuracy.
However, this comes with practical costs: the system must store and read multiple task-specific components at inference, and often incurs extra computation to obtain input-dependent coefficients.
Moreover, even with dynamic merging, a non-trivial gap to the original experts can remain.

\textcolor{red}{
Furthermore, several multi-task model merging methods either make task-identity-known assumption or train a router with a large amount of data to provide such task identity.
Leveraging given/found task identity of each input, previous models select the appropriate output head, choose task coefficients for each task, or activate task-specific modules.
However, task identity may be often unknown in practice, while the training of router with a lot of data raises concerns regarding data storage memory overhead and privacy issues.
%and , leading to data o memory overhead for storing data during model merging process, causing .
}
\fi

\begin{figure*}[t!]
    \centering
    % Update the file name to the actual figure path in the project.
    \includegraphics[width=\linewidth]{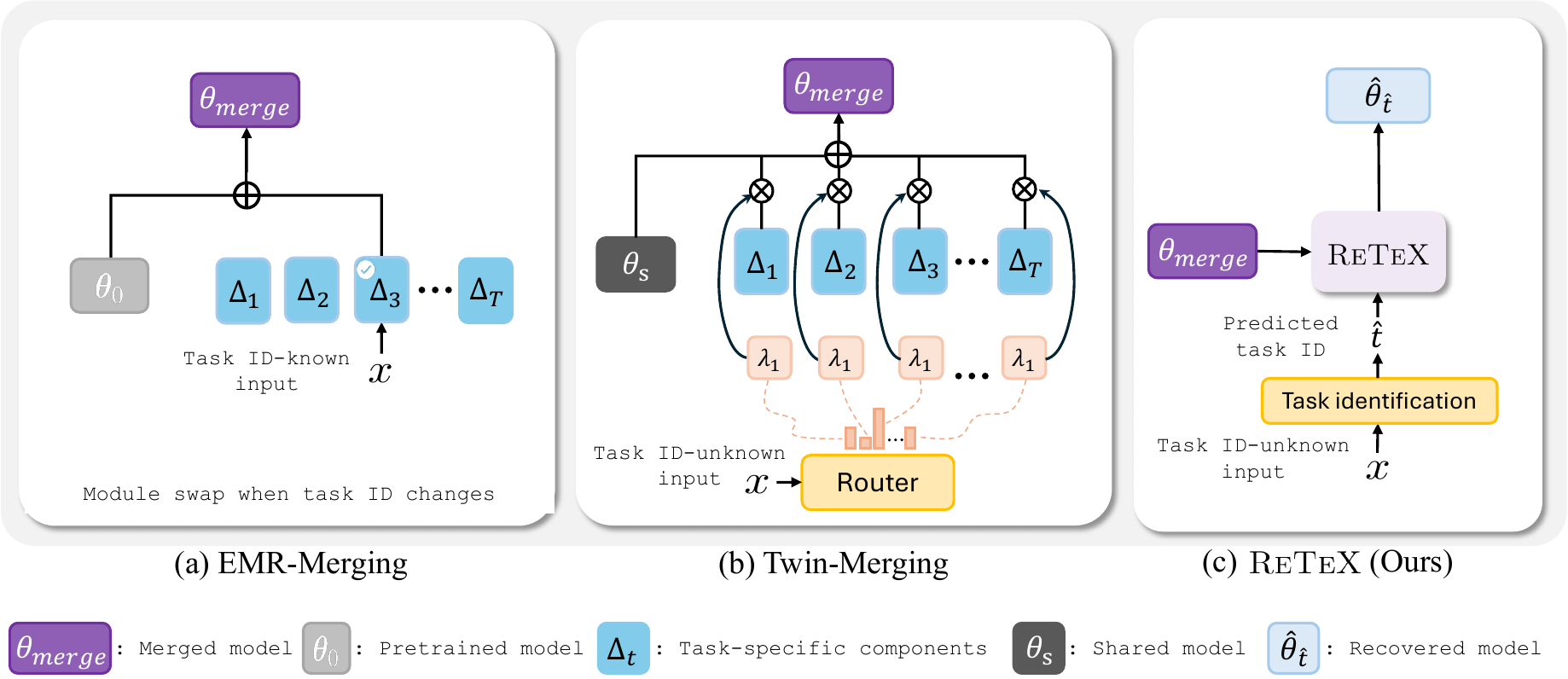}
    \caption{\textbf{Component requirements for multi-task deployment.}
    (a) \emph{EMR-Merging}~\cite{huang2024emr} assumes that the task ID is known and activates task-specific components on top of a shared backbone $\vtheta_0$.
    (b) \emph{Twin-Merging}~\cite{lu2024twin} handles task ID-unknown inputs by using a router to predict input-dependent mixing coefficients over task components, combined with a shared backbone $\vtheta_s$.
    (c) \textcolor{black}{
    \method{} handles task ID-unknown inputs by first predicting the task ID $\hat{t}$ and then recovering $\hat{\vtheta}_{\hat{t}}$ from a single merged checkpoint $\vtheta_{\text{merge}}$ using the predicted offset $\hat{\bm{\beta}}_{\hat{t}}$, without storing expert checkpoints at inference.}
    }
    \label{fig:deployment_comparison}
\end{figure*}

We take an alternative view: we interpret parameter interference as parameter perturbation introduced to each task expert during merging process.
Upon this perspective, we aim to recover task expert performance from a single merged checkpoint. 
We empirically find that the dominant effect of interference can be modeled as an additive offset. 
Thus, we aim to ``undo'' parameter interference and recover task-expert performance by predicting and removing additive offsets.
% Our starting point is the mathematical derivation, which shows that the parameters of each expert parameters can be modeled as an affine function of the merged parameters. 
% We further empirically find that the multiplicative term is often close to identity, suggesting that the dominant effect of interference can be modeled as an additive offset. 
% Intuitively, from the perspective of a task expert, merging process injects task-irrelevant perturbations originating from other experts; recovering the expert corresponds to removing these perturbations.

Based on this insight, we propose \method{}, a plug-and-play expert recovery framework. \method{} augments a merged model with a lightweight recovery module that predicts a task-conditioned parameter offset and applies it to the merged weights to reconstruct task-specialized behavior. 
Importantly, \method{} is trained using parameter supervision only: the supervision signal is provided by existing expert checkpoints (targets are expert parameters), and recovery training does not require access to task datasets or input–output examples.
%, in contrast to the training of standard routers. 
This design keeps recovery training simple and avoids additional data collection or storage.

A key challenge is determining which expert to recover when task identity is unknown.
To this end, we introduce a router-free task identifier based on SVD subspace signatures (singular values and vectors).
Notably, SVD subspace signatures are computed offline before inference.
At inference, we select the task whose subspace yields the smallest projection residual for each input, removing the need for training a router or storing task experts.
% During inference, we identify the task by selecting the subspace that yields the smallest projection residual, eliminating the need for training a router or storing task experts.
Surprisingly, adaptive interpolation of expert knowledge emerges from the proposed parameter offset prediction, when faced with unseen or out-of-distribution tasks.

We further note that our proposed method is a flexible plug-and-play framework that can be seamlessly applied on top of various merged models.
% Because \method{} operates as a post-hoc recovery layer, it can be seamlessly applied on top of various merged models, leveraging stronger base merges to improve recovery especially under tight low-rank budgets.
% Extensive experiments on computer vision and NLP benchmarks demonstrate that \method{} greatly improves the performance of various model merging methods, greatly improving the performance on both seen and unseen tasks and achieving 95\% of individual-expert performance.
Extensive experiments on computer vision and NLP benchmarks demonstrate that \method{} substantially enhances the performance of various model merging methods, greatly improving the performance on both seen and unseen tasks and achieving 95\% of individual-expert performance.
%while scaling to 30 tasks.

We summarize our contributions as follows:
\begin{enumerate}
    \item We introduce {\method}, an expert recovery framework that reconstructs task expert performance from a single merged checkpoint by ``undoing'' parameter interference via {offset prediction}.
    \item We propose a {router-free} task identification method based on {SVD subspace projection residuals}, removing the need for labeled router training and reducing data storage requirements.
    \item We show that \method{} is {plug-and-play} on top of various merged models across vision and NLP, recovering up to {95\%} of expert performance while improving robustness to unseen/OOD tasks.
\end{enumerate}

\section{Related work}
\label{sec:related_work}

\noindent\textbf{Model merging.}
Model merging aims to consolidate multiple task experts into a single merged model by directly combining or editing expert parameters, avoiding joint multi-task retraining.
Early works have employed weighted combination of expert parameters~\cite{matena2022fisher,jin2023dataless} or corresponding task vectors~\cite{ilharco2023editing}, where the goal becomes finding a set of coefficients such that a resulting merged model achieve strong performance on all tasks.
However, simple weighted averaging often results in parameter interference and thus performance degradation, which few works have attempted to tackle by sparsifying or masking task vectors to mitigate conflicts~\cite{yadav2023ties,wang2024localizing}.
\textcolor{black}{A complementary line of work addresses this issue before merging by modifying the fine-tuning stage itself, encouraging task experts to become more mergeable or less interfering in parameter space~\cite{lee2025mitigating, ortiz2023task, iurada2025efficient}.}
Despite efforts, a single shared model still often falls short of per-task experts.
% Weight-space ensembling and averaging provide simple yet strong baselines~\cite{wortsman2022modelsoups}, and static merging is further improved with principled fusion rules such as Fisher-weighted merging~\cite{matena2022fisher} and RegMean~\cite{regmean}, as well as task-vector composition via Task Arithmetic~\cite{ilharco2023editing}.
% To mitigate parameter interference when many tasks are combined, methods such as TIES-Merging~\cite{yadav2023ties} and Consensus TA/TALL-mask~\cite{wang2024localizing} sparsify or mask task updates, but a single shared checkpoint can still fall short of per-task experts as task diversity grows.

\noindent\textbf{Task-conditioned adaptation and meta-learning.}
\textcolor{black}{
HyperNetworks and FiLM generate task- or input-conditioned weights and feature-wise affine modulations~\cite{ha2017hypernetworks,perez2018film}.
Few-shot conditional adaptation methods, such as CNAPs and its follow-up variants, infer task-conditioned parameters from support-set statistics or limited supervision~\cite{requeima2019cnaps,bateni2020improved,bateni2022enhancing,li2022crossdomain,triantafillou2021universal}.
}
Related meta-learning methods learn task-adaptive initialization, attenuation, hyperparameters, losses, or online/test-time adaptation mechanisms~\cite{MAML,baik2020learning,baik2022learning,baik2020alfa,baik2024learning,baik2021metal,choi2020scene,choi2022testtime,choi2022visual}.
\textcolor{black}{
Although these works are related in that they produce task-adaptive behavior, they primarily adapt from support examples, task observations, or gradient-based updates.
}
Although similarly task-adaptive, \method{} has a different goal: it recovers task-expert behavior from an already merged checkpoint using expert-parameter supervision, rather than adapting from support examples.

% \noindent\textbf{Dynamic model merging.}
% Dynamic model merging aims to reduce parameter interference by formulating input-adaptive weighted combination by predicting coefficients for each input via routing at inference time~\cite{muqeeth2023soft, lu2024twin, tang2024merging, oh2025dawin}.
% While these methods have greatly improved the performance, they need a large amount of task datasets to train a routing module or a large number of forward passes to estimate the task ID of each sample.
% Further, they need to save and load multiple task components during inference, increasing memory traffic and complexity that increases with the number of tasks.
% % While often accurate, such methods usually require inference-time access to multiple task components and additional routing computation, increasing memory traffic and system complexity.
% Meanwhile, other works have proposed to store lightweight task-expert, which is achieved by compressing full parameters via masking or sparsification~\cite{huang2024emr, wang2024localizing, gargiulo2024task}.
% These approaches reduce storage for storing task experts, but they assume the knowledge of task ID of each input, limiting the applicability in practice.
\noindent\textbf{Dynamic model merging.}
Dynamic model merging aims to reduce parameter interference by formulating input-adaptive weighted combination by predicting coefficients for each input via routing at inference time~\cite{muqeeth2023soft,lu2024twin,tang2024merging,oh2025dawin}.
While these methods have greatly improved performance, they require large task datasets to train a routing module or a large number of forward passes to estimate the task ID of each sample.
\textcolor{black}{
A related training-free direction, SiM~\cite{sim}, shows that SVD-based task manifolds constructed from a small support set can provide reliable task signals by routing each test input through projection residuals, without additional router training.
Building on this observation, \method{} uses the SVD residual-based task signal not merely to select among stored components, but to condition task-expert recovery from a single merged checkpoint.
}
However, these dynamic or routing-based approaches still require storing and loading multiple task components during inference, increasing memory traffic and complexity as the number of tasks grows.
Meanwhile, other works have proposed storing lightweight task-expert components by compressing full parameters via masking or sparsification~\cite{huang2024emr,wang2024localizing,gargiulo2024task}.
These approaches reduce storage for task experts, but they assume knowledge of the task ID of each input, limiting their applicability in practice.
In contrast to previous dynamic model merging methods that require either task ID of each input or multiple experts during inference, our proposed approach directly recovers task-expert performance from a single merged model.
\section{Background}
\label{sec:background}

\noindent\textbf{Problem setting.}
Given a pre-trained model $f:\gX\times\Theta\to\gY$ with parameters $\vtheta_0\in\Theta$, we consider $T$ downstream tasks indexed by $t\in\{1,\dots,T\}$.
Since many finetuned models have become available in HuggingFace with the advent of pretraining-finetuning paradigm, we make the following assumption, similar to previous works~\cite{muqeeth2023soft, lu2024twin, tang2024merging, oh2025dawin,ilharco2023editing}: for each task $t$, a task expert $f_{\vtheta_t}$ is assumed to be available prior to merging process.
Prior to model merging process, we assume each task expert has been obtained by fine-tuning $f_{\vtheta_0}$ on each task dataset $\mathcal{D}_t=\{(\vx_t^{i}, y_t^{i})\}_{i=1}^{N_t}$ with $\vx_t^{i}\in \gX$ and $y_t^{i}\in\gY$.

Given task experts, multi-task model merging aims to find a single merged model $\vtheta_{\text{merge}}$ that can generalize across all tasks by directly operating on the expert parameters $\{\vtheta_t\}_{t=1}^T$:
\begin{equation}
\vtheta_{\text{merge}} = \sum_{t=1}^T \alpha_t \vtheta_t.
\end{equation}
% such that the resulting model performs well across all tasks.

\noindent\textbf{Task arithmetic.}
To better facilitate the management of task expert knowledge, Task Arithmetic (TA)~\cite{ilharco2023editing} proposes to isolate task-specific knowledge from prior knowledge; $\bm{\tau}_t := \vtheta_t - \vtheta_0$.
Using task vectors, TA and subsequent works formulate merging as
\begin{equation}
    \label{eq:task_arithmetic}
    \vtheta_{\text{merge}} = \vtheta_0 + \sum_{t=1}^T \lambda_t \bm{\tau}_t,
\end{equation}
where $\lambda_t$ is a task coefficient.

\noindent\textbf{Dynamic model merging.}
Despite previous efforts, static model merging is susceptible to parameter interference~\cite{ilharco2023editing}, as it is forced to decide which task-specific model parameter will contribute to the final merged model parameters, which are then fixed for all input data.
However, each data sample can come from varying tasks and thus require varying model expertise.
%Arguing that the accurate prediction for each data sample $\vx$ requires varying model expertise, recent works have shifted their attention to dynamic model merging~\cite{kang2024self, li2023merge, muqeeth2023soft,lu2024twin,oh2025dawin}.
Recently, several works have shifted their attention to dynamic model merging~\cite{kang2024self, li2023merge, muqeeth2023soft,lu2024twin,oh2025dawin}.
Dynamic model merging aims to find input-wise model expertise; in other words, the goal is to find the input-adaptive task coefficients $\lambda_t(\vx)$:
\begin{equation}
    \label{eq:dynamic_merging}
    \vtheta_\text{merge} =  \vtheta_0 + \sum_{t=1}^T \lambda_t(\vx)\bm{\tau}_t.
    %\vtheta' =  \vtheta_0 + \sum_{t=1}^T \alpha_t\bm{\tau}_t,
\end{equation}    
\section{Learning to Recover Task Experts}
\label{sec:learning_recovery_final_corrected}

Previous dynamic model merging methods have greatly improved the performance of a merged model by dynamically composing experts with input-adaptive coefficients at inference time~\cite{oh2025dawin, tang2024merging, lu2024twin, muqeeth2023soft}.
% aim to find task coefficients that improve performance on each task.
% Recent works further improve accuracy by dynamically composing experts with input-adaptive coefficients at inference time~\cite{oh2025dawin, tang2024merging, lu2024twin, muqeeth2023soft}.
While dynamic model merging methods greatly mitigated parameter interference, previous methods often require the knowledge of task ID of each input example~\cite{huang2024emr, wang2024localizing, gargiulo2024task} or require additional training with abundant amount of training examples~\cite{muqeeth2023soft, lu2024twin, tang2024merging}.
% it typically requires storing and reading multiple task-specific components at run time, which increases memory traffic and computation.
% Moreover, a non-trivial performance gap between a merged model and the original task experts often remains.

In this work, we approach model merging from the perspective of each task expert: each task expert experiences parameter perturbations that other experts inject during merging, resulting in parameter interference.
Thus, in this work, we aim to ``undo'' such parameter perturbations and recover a task expert from a single merged model.
% \method{} attributes this persisting gap to structured parameter perturbations that other experts inject during merging.
% Instead of repeatedly composing experts at inference, \method{} targets expert recovery from a single merged model.
As delineated in Appendix~\cref{app:offset_justification}, we show that each task expert parameter $\vtheta_t$ can be expressed as an affine transformation of $\vtheta_\text{merge}$:
\begin{equation}
\label{eq:theta_affine_form}
\vtheta_t
=
\gamma_t\,\vtheta_{\text{merge}}
\;+\;
\bm{\beta}_t,
\end{equation}
% The starting point is an affine recovery rule in parameter space.
% Let the merged task vector satisfy $\bm{\tau}_{\text{merge}}:=\vtheta_{\text{merge}}-\vtheta_0$.
% Then expert parameters admit the affine form
% \begin{equation}
% \label{eq:theta_affine_form}
% \vtheta_t
% \;\approx\;
% \vtheta_0
% \;+\;
% \gamma_t\,\bm{\tau}_{\text{merge}}
% \;+\;
% \bm{\beta}_t,
% \end{equation}
where $\gamma_t$ and $\bm{\beta}_t$ denote task-conditioned scaling and shift terms, which are composed of task coefficients and other task expert parameters.
This form suggests that expert recovery reduces to predicting $(\gamma_t,\bm{\beta}_t)$ conditioned on task identity.
However, our empirical analysis indicates that $\gamma_t$ concentrates near an identity $1$, which motivates an offset-only recovery rule:
\begin{equation}
\label{eq:theta_offset_form}
\vtheta_t
\;\approx\;
\vtheta_{\text{merge}}
\;+\;
\bm{\beta}_t,
\end{equation}
where we interpret $\bm{\beta}_t$ as ``undoing'' parameter perturbations introduced by other task experts to reduce parameter interference and recover task-expert performance for task $t$.

Thus, as illustrated in \cref{fig:overall_framework}, \method{} operates in two stages: (1) task identification that estimates a task ID $\hat{t}$ by nearest-subspace matching using a compact task memory bank (\cref{sec:task_identification}), and (2) task expert recovery that reconstructs inference parameters from $\vtheta_{\text{merge}}$ conditioned on $\hat{t}$ using offset prediction (\cref{sec:task_expert_recovery}).
% \method{} targets task-agnostic inference and follows two stages:
% (1) task identification that estimates a task ID $\hat{t}$ by nearest-subspace matching using a compact task memory bank, and
% (2) task expert recovery that reconstructs inference parameters from $\vtheta_{\text{merge}}$ conditioned on $\hat{t}$ using offset prediction.
% Figure~\ref{fig:overall_framework} illustrates the overall pipeline.

\begin{figure*}[t]
    \centering
    \includegraphics[width=\linewidth]{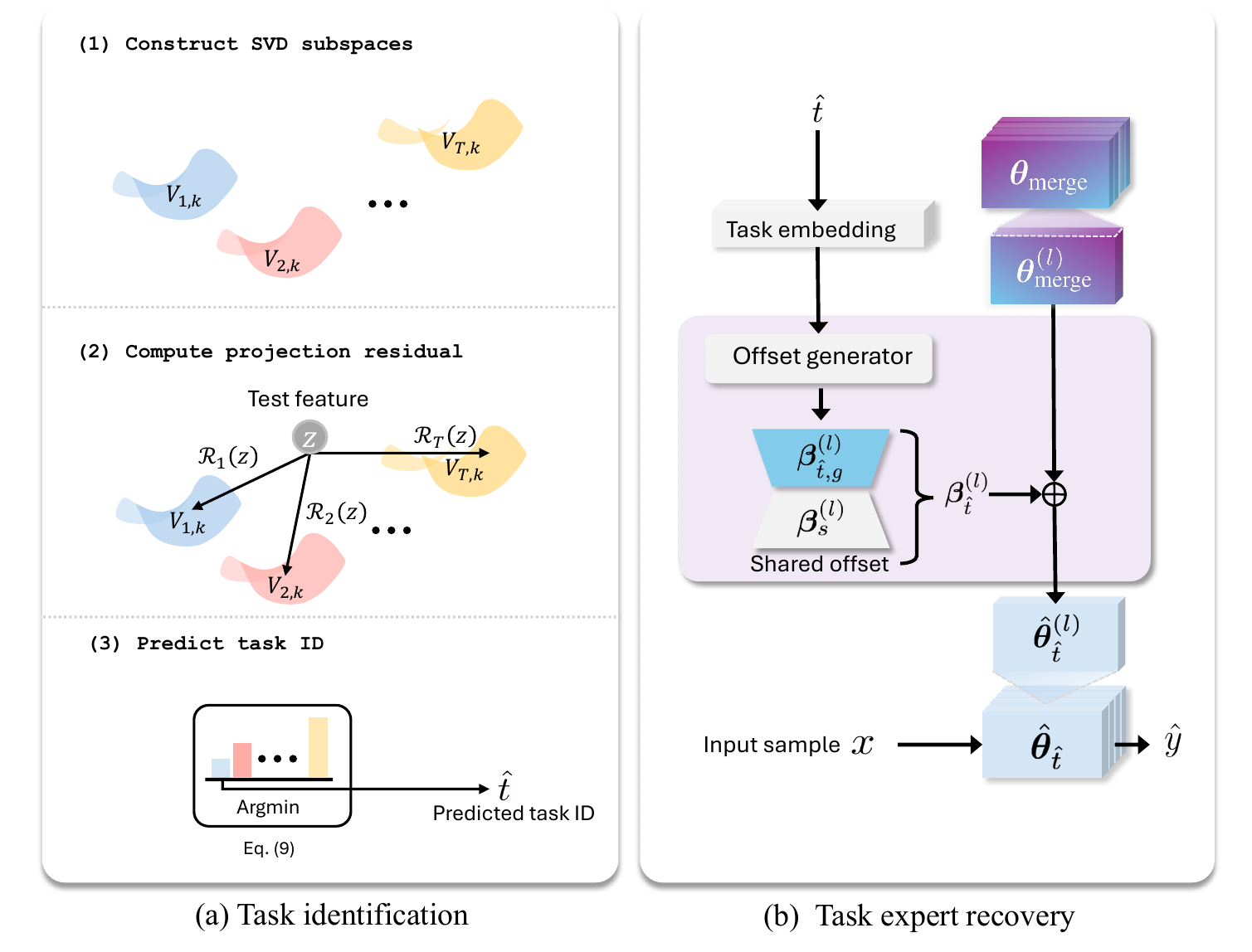}
    \caption{\textbf{Overall \method{} framework.}
    (a) Task identification constructs a task memory bank by fitting a low-dimensional SVD subspace to merged-model representations for each task.
    Given a test representation $\mathbf{z}$, \method{} computes projection residuals $\{\mathcal{R}_t(\mathbf{z})\}_{t=1}^{T}$ and outputs a task ID estimate $\hat{t}$ by an argmin rule.
    (b) Task expert recovery conditions a task embedding on $\hat{t}$ and predicts low-rank offsets across layers.
    Adding the offsets to merged parameters yields inference parameters for prediction.}
    \label{fig:overall_framework}
\end{figure*}

\iffalse
\cref{sec:task_identification} describes task identification and \cref{sec:task_expert_recovery} describes task expert recovery conditioned on the task ID estimate.
\fi

\subsection{Task identification}
\label{sec:task_identification}
To estimate task ID of each, we take inspirations from the nearest-subspace viewpoint used in the Nearest Subspace Classifier (NSC)~\cite{chi2012nsc}.
% Task-agnostic inference requires a task ID estimate for each input.
% The design follows the nearest-subspace viewpoint used in the Nearest Subspace Classifier (NSC)~\cite{chi2012nsc}, which assigns a sample to the class subspace with the smallest projection residual.
In particular, \method{} represents each task with a low-dimensional subspace estimated from merged-model representations.
%and stores these subspaces in a task memory bank.
At inference, \method{} selects the task that minimizes the projection residual.
We note that these subspaces incur much less memory overhead, compared to full task expert parameters or low-rank experts, as demonstrated in experimental results.

\noindent\textbf{Constructing task SVD subspaces.}
Let $\mathbf{z}_t^{i}=f_{\text{merge}}(\vx)\in\mathbb{R}^{d}$ denote an intermediate representation extracted by a merged model $f_{\vtheta_{\text{merge}}}$.
For each task $t\in\{1,\ldots,T\}$, we utilize $N$ reference inputs from original training set 
$\mathcal{D}_t^{\text{mb}}=\{\vx_t^{1},\vx_t^{2},\ldots,\vx_t^{N}\}$
and compute the empirical mean of its representations $\mathbf{z}_t^{i}$:
% \method{} computes the empirical mean
\begin{equation}
\label{eq:mean}
\bm{\mu}_t
=
\frac{1}{N}\sum_{i=1}^{N}\mathbf{z}_t^{i}.
\end{equation}
Then, we compute the feature matrix $\mathbf{X}_t\in\mathbb{R}^{N\times d}$ whose $i$-th row equals
$(\mathbf{z}_t^{i}-\bm{\mu}_t)^\top$,
and compute the SVD
\begin{equation}
\label{eq:singluar}
\mathbf{X}_t
=
\mathbf{U}_t\mathbf{\Sigma}_t\mathbf{V}_t^\top.
\end{equation}
\method{} keeps the top-$k$ right singular vectors as the task subspace basis
$\mathbf{V}_{t,k}\in\mathbb{R}^{d\times k}$ with $k\ll d$.
The set $\bm{\mu}_t + \mathrm{span}(\mathbf{V}_{t,k})$ defines an affine subspace for task $t$.
\method{} stores $\{\bm{\mu}_t,\mathbf{V}_{t,k}\}_{t=1}^{T}$ as task signatures.
Unless otherwise specified, \method{} uses $N=64$ reference inputs per task.

\noindent\textbf{Task identification.}
Given a test input $\vx$, and its feature $\mathbf{z}=f_{\text{merge}}(\vx)$.
% Given a test input $\vx$, \method{} computes $\mathbf{z}=f_{\text{merge}}(\vx)$.
\method{} measures alignment to each task subspace using the projection residual
\begin{equation}
\label{eq:residual}
\mathcal{R}_t(\mathbf{z})
=
\left\|
\left(\mathbf{I}-\mathbf{V}_{t,k}\mathbf{V}_{t,k}^\top\right)
(\mathbf{z}-\bm{\mu}_t)
\right\|_2,
\end{equation}
where $\mathbf{I}\in\mathbb{R}^{d\times d}$ denotes the identity matrix.
The task ID estimate satisfies
\begin{equation}
\label{eq:task_inference}
\hat{t}
=
\argmin_{t\in\{1,\ldots,T\}}
\mathcal{R}_t(\mathbf{z}).
\end{equation}
% For a batch of test samples, \method{} applies \cref{eq:task_inference} independently to each sample representation.

\subsection{Task expert recovery}
\label{sec:task_expert_recovery}

Given the task ID estimate $\hat{t}$ from \cref{sec:task_identification}, \method{} recovers the corresponding task-expert performance by predicting low-rank offsets that would ``undo'' parameter interference as follows:
%and adding them to merged parameters.
\method{} defines inference parameters as
\begin{equation}
\label{eq:theta_hat_t_def}
\vtheta_{\hat{t}}
\;:=\;
\vtheta_{\text{merge}} + \bm{\beta}_{\hat{t}},
\end{equation}
where $\bm{\beta}_{\hat{t}}$ denotes the offset predicted, conditioned on estimated task ID $\hat{t}$.

\noindent\textbf{Task embedding.}
Each task $t$ is represented with a learnable embedding $\ve_t\in\mathbb{R}^{d_{\mathrm{emb}}}$, where $d_{\mathrm{emb}}$ is the dimension of task embedding.
A one-hot task indicator for $\hat{t}$ maps to the embedding $\ve_{\hat{t}}$, which serves as input to the offset generator.

\noindent\textbf{Low-rank offset prediction.}
\method{} adopts a low-rank factorization for offsets independently for each layer $l$.
% \method{} adopts a LoRA-style low-rank factorization for recovery offsets~\cite{hu2022lora}.
% \method{} applies the same parameterization independently to each layer $l$; the following equations describe a generic layer $l$.
Let $\vtheta^{(l)}_{\text{merge}}\in\mathbb{R}^{a\times b}$ denote merged parameters at layer $l$, and choose rank $r<\min(a,b)$.
For the further efficiency, we decompose an offset as follows: 
\begin{equation}
\label{eq:beta_low_rank}
\bm{\beta}^{(l)}_{\hat{t}}
=
\bm{\beta}^{(l)}_{\hat{t},g}\,\bm{\beta}^{(l)}_{s},
\end{equation}
where $\bm{\beta}^{(l)}_{s}\in\mathbb{R}^{r\times b}$ is a learnable offset shared across all task experts, and $\bm{\beta}^{(l)}_{\hat{t},g}\in\mathbb{R}^{a\times r}$ is a task-specific offset predicted by single-layer offset generator $h^{(l)}$ conditioned on $\ve_{\hat{t}}$.
% A lightweight linear predictor $h^{(l)}$ maps $\ve_{\hat{t}}$ to a task-conditioned factor
% $\bm{\beta}^{(l)}_{\hat{t},g}\in\mathbb{R}^{a\times r}$.
% \method{} combines this factor with a shared learnable basis $\bm{\beta}^{(l)}_{s}\in\mathbb{R}^{r\times b}$ to form the layer offset:
Thus, for each layer, task-expert parameters at layer $l$ are approximated as $\vtheta^{(l)}_{\hat{t}}=\vtheta^{(l)}_{\text{merge}}+\bm{\beta}^{(l)}_{\hat{t}}.$ 
% \begin{equation}
% \label{eq:theta_recover_layer}
% \vtheta^{(l)}_{\hat{t}}
% =
% \vtheta^{(l)}_{\text{merge}}+\bm{\beta}^{(l)}_{\hat{t}}.
% \end{equation}
Stacking layers yields $\vtheta_{\hat{t}}=\vtheta_{\text{merge}}+\bm{\beta}_{\hat{t}}$, consistent with \cref{eq:theta_hat_t_def}.

\noindent\textbf{Optimization.}
We note that the training of the offset prediction module only uses parameter supervision, without the need for task datasets of input-output samples or test samples.
The training process of \method{} only requires a randomly sampled task ID $t$ and its corresponding task expert parameters.
In other words, the training objective for \method{} is

% Similar to standard model merging that operates on expert checkpoints, \method{} uses expert parameters as supervision.
% \method{} samples a task ID $t$, forms
% $\vtheta^{(l)}_{t,\mathrm{rec}}:=\vtheta^{(l)}_{\text{merge}}+\bm{\beta}^{(l)}_{t}$,
% and minimizes the discrepancy to corresponding expert parameters:
% \begin{equation}
% \label{eq:retex_training_loss_offset}
% \mathcal{L}
% =
% \sum_{t=1}^{T}\sum_{l=1}^{L}
% \left\|
% \textcolor{blue}{\vtheta^{(l)}_{\hat{t}}}-\vtheta^{(l)}_{t}
% \right\|_{2}^{2}.
% \end{equation}

\textcolor{black}{
\begin{equation}
\label{eq:retex_training_loss_offset}
\mathcal{L}
=
\mathbb{E}_{t \sim \mathcal{U}(1, T)}
\left[
\left\|
(\vtheta_{\text{merge}} + \hat{\bm{\beta}}_{t}) - \vtheta_{t}
\right\|_{2}^{2}
\right],
\end{equation}
where $\mathcal{U}(1, T)$ denotes a uniform distribution over task indices; $\hat{\bm{\beta}}_t$ denotes the predicted task-specific offset for the sampled task $t$; and $\vtheta_t$ acts as ground-truth. 
}

At inference, \method{} does not require storage of task experts: \method{} predicts a task ID $\hat{t}$ and recovers $\vtheta_{\hat{t}}$ on demand from $\vtheta_{\text{merge}}$.
% In this section, we first study efficiency improvements of \method{}, and then evaluate its multi-task merging performance under two inference scenarios: \textit{task-known} and \textit{task-unknown}.
% In the \textit{task-known} scenario (task ID available at inference), which is a standard multi-task model merging setting (to know which output layer to use), we compare with methods that construct a merged model conditioned on a given task: Task Arithmetic~\cite{ilharco2023editing}, TIES-Merging~\cite{yadav2023ties}, and EMR-Merging~\cite{huang2024emr}.
% In the \textit{task-unknown} scenario (task ID not given), we compare with methods that operate without task identity: Weight Averaging, Twin-Merging~\cite{lu2024twin}, and DaWin~\cite{oh2025dawin}.
% \method{} supports both settings: conditioned on given (task-known) or predicted (task-unknown), \method predicts a task offset to recover the corresponding task-expert.
\section{Experiments}
\label{sec:experiments}

We evaluate \method{} in a \textit{task-agnostic} multi-task deployment, where test inputs arrive as a mixed-task stream and per-sample task identity is not provided.
Accordingly, we do not assume task-specific validation or calibration at inference; \method{} predicts task identity and recovers the corresponding expert behavior on demand from a \emph{single} merged checkpoint.
% We report results on vision, NLP, and multimodal, and additionally study unseen-task generalization and efficiency.
We report results on vision, NLP, and additionally study unseen-task generalization and efficiency.

\noindent\textbf{Baselines and comparison scope.}
We group baselines by inference-time component requirements.
\textit{Static model merging} produces a single merged checkpoint and applies it to all inputs, without storing expert checkpoints at inference.
This group includes Weight Averaging, Task Arithmetic~\cite{ilharco2023editing}, TIES-Merging~\cite{yadav2023ties}, Consensus TA~\cite{wang2024localizing} and AdaMerging\cite{yang2024adamerging}.
\textit{Dynamic model merging} performs input-adaptive merging at inference and typically requires access to task-specific components (e.g., expert checkpoints or equivalent representations) during inference.
%to construct the effective model.
This group includes Twin-Merging~\cite{lu2024twin}, DaWin~\cite{oh2025dawin}, MoW-Merging~\cite{ye2025dynamic}, and WEMoE~\cite{tang2024merging}.
We also report Zero-shot (pre-trained model) and Individual experts (original task experts) as lower and upper reference points.
Detailed baseline configurations and protocols are provided in Appendix~\cref{app:experiment_details}.

\noindent\textbf{Implementation details.}
\label{sec:training_setup}
Unless specified otherwise, the base merged model $\vtheta_{\text{merge}}$ is constructed via simple Weight Averaging of task experts, to highlight that \method{} can recover high-fidelity experts even from a simple static merge.
We train the task-conditioned offset predictor for 5000 iterations using Adam~\cite{kingma2014adam-key} with learning rate $2\times 10^{-4}$ and a cosine schedule with 600 warm-up steps.
We note that the training of the offset predictor in \method{} does not require original task datasets of input-output examples.
\method{} training requires parameter supervision only: given any task ID, it predicts a task-conditioned parameter offset to approximate the task expert parameters, which are used as ground-truth.
%that removes interference from the merged model, without requiring task datasets or input–output examples during ReTeX training.

%The objective function for training \method{} is the L2 loss between the reconstructed layer parameters $\hat{\vtheta}_t^{(l)}$ and the task expert layer parameters $\vtheta_t^{(l)}$, as defined in Equation~\ref{eq:retex_training_loss_offset}, summed over all tasks and layers.

\subsection{Vision tasks}
\label{sec:vision_tasks}

\noindent\textbf{Setting.}
We follow the multi-task CLIP merging protocol in~\cite{wang2024localizing}.
We fine-tune a separate expert per dataset on three CLIP~\cite{radford2021llearning} backbones (ViT-B/32, ViT-B/16, ViT-L/14), and evaluate merged models under task-agnostic inference.
The 8-task suite includes 
SUN397~\cite{xiao2016sun}, 
Cars~\cite{krause20133d}, 
RESISC45~\cite{cheng2017remote}, 
EuroSAT~\cite{helber2019eurosat}, 
SVHN~\cite{netzer2011reading}, 
GTSRB~\cite{stallkamp2011german}, 
MNIST~\cite{deng2012mnist}, 
and DTD~\cite{cimpoi2014describing}.
The 14-task suite further adds 
CIFAR100~\cite{krizhevsky2009learning}, 
STL10~\cite{coates2011analysis}, 
Flowers102~\cite{nilsback2008automated}, 
Oxford-IIIT-Pet~\cite{parkhi2012cats}, 
PCAM~\cite{veeling2018rotation}, 
and FER2013~\cite{goodfellow2013challenges}.
The 20-task suite further adds 
EMNIST~\cite{cohen2017emnist}, 
CIFAR10~\cite{krizhevsky2009learning}, 
Food101~\cite{bossard2014food}, 
FashionMNIST~\cite{xiao2017fashion}, 
RenderedSST2~\cite{socher2013recursive}, 
and KMNIST~\cite{clanuwat2018deep}.
To study scalability in the number of tasks, we additionally evaluate a 30-task suite on ViT-B/32 by augmenting the 20-task suite with ten more datasets:
Vegetables~\cite{ahmed2021dcnn}, 
Kvasir-v2~\cite{pogorelov2017kvasir}, 
Intel Images~\cite{bansal2019intel}, 
Weather~\cite{xiao2021classification}, 
Cats and dogs~\cite{cukierskidogs}, 
MangoLeafBD~\cite{ahmed2023mangoleafbd}, 
Beans~\cite{beansdata}, 
Landscape Recognition~\cite{Landscape}, 
Garbage Classification~\cite{cchang_2018}, 
and Fruits-360~\cite{muresan2018fruit}.
We report classification accuracy (\%) on all datasets.

\noindent\textbf{Results on 8/14/20 tasks.}
\begin{table*}[t!]
  \centering
  \caption{\textbf{Multi-task performance of merged models across different CLIP backbones and numbers of tasks.}
  Values in parentheses $_{(\cdot)}$ indicate normalized accuracy (merged / individual).
  All methods are evaluated on 8, 14, and 20 computer vision tasks.}
  \label{tab_vision_main}
  \setlength{\tabcolsep}{4pt}
  \resizebox{\textwidth}{!}{
  \begin{tabular}{l|ccc|ccc|ccc}
    \toprule
    \multirow{2}{*}{\textbf{Method}} 
      & \multicolumn{3}{c|}{\textbf{ViT-B/32}} 
      & \multicolumn{3}{c|}{\textbf{ViT-B/16}} 
      & \multicolumn{3}{c}{\textbf{ViT-L/14}} \\
    \cmidrule(lr){2-4}\cmidrule(lr){5-7}\cmidrule(lr){8-10}
      & \textbf{8 tasks} & \textbf{14 tasks} & \textbf{20 tasks} 
      & \textbf{8 tasks} & \textbf{14 tasks} & \textbf{20 tasks} 
      & \textbf{8 tasks} & \textbf{14 tasks} & \textbf{20 tasks} \\
    \midrule
    Zero-shot    & 48.3 & 57.2 & 56.1 & 55.3 & 61.3 & 59.7 & 64.7 & 68.2 & 65.2 \\
    Individual   & 92.9 & 90.9 & 91.4 & 94.7 & 92.8 & 92.8 & 95.9 & 94.3 & 94.8 \\
    \midrule

    \rowcolor[HTML]{EFEFEF}
    \multicolumn{10}{l}{\textbf{\textit{(Static model merging)}}} \\

    Weight Averaging                          & 66.3$_{(72.1)}$ & 64.3$_{(71.1)}$ & 61.0$_{(67.5)}$ & 72.2$_{(76.6)}$ & 69.5$_{(74.8)}$ & 65.3$_{(70.4)}$ & 79.6$_{(83.2)}$ & 76.7$_{(81.1)}$ & 71.6$_{(75.6)}$ \\
    Task Arithmetic~\cite{ilharco2023editing}  & 70.6$_{(76.3)}$ & 65.3$_{(72.0)}$ & 60.5$_{(66.7)}$ & 75.9$_{(74.4)}$ & 70.5$_{(75.9)}$ & 65.8$_{(70.7)}$ & 85.0$_{(88.6)}$ & 79.4$_{(83.9)}$ & 74.1$_{(78.1)}$ \\
    TIES-Merging~\cite{yadav2023ties}         & 74.5$_{(80.3)}$ & 65.1$_{(71.9)}$ & 62.3$_{(68.8)}$ & 79.4$_{(83.8)}$ & 70.1$_{(75.5)}$ & 66.6$_{(71.7)}$ & 85.8$_{(89.5)}$ & 77.3$_{(81.6)}$ & 72.9$_{(76.9)}$ \\
    Consensus TA~\cite{wang2024localizing}      & 75.0$_{(80.8)}$ & 70.4$_{(77.4)}$ & 65.4$_{(72.0)}$ & 79.4$_{(83.9)}$ & 74.4$_{(79.9)}$ & 69.8$_{(74.9)}$ & 86.3$_{(90.1)}$ & 82.2$_{(86.9)}$ & 79.0$_{(83.2)}$ \\
    AdaMerging~\cite{yang2024adamerging}      & 75.9$_{(81.7)}$ & 70.5$_{(77.3)}$ & 66.2$_{(72.7)}$ & 79.5$_{(83.8)}$ & 73.8$_{(79.2)}$ & 69.6$_{(74.6)}$ & 87.4$_{(94.6)}$ & 82.6$_{(87.2)}$ & 77.8$_{(82.0)}$\\
    \midrule
    \rowcolor[HTML]{EFEFEF}
    \multicolumn{10}{l}{\textbf{\textit{(Dynamic model merging)}}} \\

    Twin-Merging~\cite{lu2024twin}            & 84.0$_{(90.3)}$ & 70.0$_{(76.7)}$ & 57.5$_{(61.8)}$ & 91.4$_{(96.2)}$ & 78.4$_{(83.9)}$ & 63.1$_{(67.0)}$ & 93.7$_{(97.7)}$ & 86.2$_{(91.2)}$ & 74.8$_{(78.6)}$ \\
    DaWin~\cite{oh2025dawin}                  & 89.0$_{(95.3)}$ & 73.8$_{(80.3)}$ & 52.8$_{(57.7)}$ & 87.1$_{(91.9)}$ & 77.8$_{(83.5)}$ & 62.8$_{(67.3)}$ & 91.6$_{(95.5)}$ & 82.6$_{(87.2)}$ & 77.5$_{(81.8)}$ \\
    MoW-Merging~\cite{ye2025dynamic}          & 88.1$_{(94.8)}$ & 83.2$_{(91.5)}$ & 79.3$_{(86.8)}$ & 93.7$_{(98.9)}$ & 79.3$_{(85.5)}$ & 78.2$_{(84.3)}$ & 94.9$_{(99.0)}$ & 78.8$_{(83.6)}$ & 81.8$_{(86.3)}$ \\
    WEMoE~\cite{tang2024merging}            & 90.4$_{(97.3)}$ & 83.1$_{(90.7)}$ & 74.4$_{(81.2)}$ & 93.1$_{(98.2)}$ & 85.6$_{(92.9)}$ & 78.8$_{(84.1)}$ & 94.8$_{(98.8)}$ & 87.0$_{(91.6)}$ & 75.7$_{(79.5)}$ \\ 

    \midrule

    \rowcolor[HTML]{FFF2CC}
    \textbf{\method{} (Ours)}   
      & $\mathbf{92.2}_{(99.3)}$ & $\mathbf{89.8}_{(98.8)}$ & $\mathbf{89.8}_{(98.3)}$
      & $\mathbf{94.2}_{(99.4)}$ & $\mathbf{91.9}_{(99.0)}$ & $\mathbf{91.7}_{(98.4)}$
      & $\mathbf{95.4}_{(99.5)}$ & $\mathbf{93.5}_{(99.1)}$ & $\mathbf{93.6}_{(98.9)}$ \\
    \bottomrule
  \end{tabular}
  }
  
\end{table*}
\cref{tab_vision_main} reports multi-task accuracy across three CLIP backbones and 8/14/20 task suites.
Static merging baselines degrade substantially as the number of tasks increases, consistent with parameter interference under a single fixed checkpoint.
Dynamic merging baselines mitigate interference via input-adaptive composition, but still exhibit degradation under larger suites and require inference-time access to task-specific components.
Across all backbones and task counts, \method{} achieves the strongest performance and closely matches Individual experts, reaching around $98$--$99\%$ normalized accuracy while using a single merged checkpoint plus lightweight recovery parameters.

\noindent\textbf{Results on 30 tasks.}
\begin{table}[t]
\caption{\textbf{Multi-task performance on ViT-B/32 across different numbers of vision tasks.}
Values in parentheses $_{(\cdot)}$ indicate normalized accuracy (merged / individual).}
\centering
\label{tab:30_tasks_vision_vit_b_32}

\resizebox{\linewidth}{!}{
\begin{tabular}{l|cccc}
\toprule
\multicolumn{1}{l|}{\textbf{Method}} & \textbf{8 tasks} & \textbf{14 tasks} & \textbf{20 tasks} & \textbf{30 tasks} \\
\midrule
\multicolumn{1}{l|}{Zero-shot}  & 48.3 & 57.2 & 56.1 & 55.5 \\
\multicolumn{1}{l|}{Individual} & 92.9 & 90.9 & 91.4 & 93.1 \\
\midrule

\rowcolor[HTML]{EFEFEF}
\multicolumn{5}{l}{\textbf{\textit{(Static model merging)}}} \\
\multicolumn{1}{l|}{Weight Averaging}                         & 66.3$_{(72.1)}$ & 64.3$_{(71.1)}$ & 61.0$_{(67.5)}$ & 59.1$_{(64.2)}$ \\
\multicolumn{1}{l|}{Task Arithmetic~\cite{ilharco2023editing}} & 70.6$_{(76.3)}$ & 65.3$_{(72.0)}$ & 60.5$_{(66.7)}$ & 58.0$_{(62.8)}$ \\
\multicolumn{1}{l|}{TIES-Merging~\cite{yadav2023ties}}         & 74.5$_{(80.3)}$ & 65.1$_{(71.9)}$ & 62.3$_{(68.8)}$ & 59.6$_{(64.7)}$ \\
\multicolumn{1}{l|}{Consensus TA~\cite{wang2024localizing}}    & 75.0$_{(80.8)}$ & 70.4$_{(77.4)}$ & 65.4$_{(72.0)}$ & 63.4$_{(68.5)}$ \\
\multicolumn{1}{l|}{AdaMerging~\cite{yang2024adamerging}}      & 75.9$_{(81.7)}$ & 70.5$_{(77.3)}$ & 66.2$_{(72.7)}$ & 56.8$_{(61.3)}$ \\
\midrule

\rowcolor[HTML]{EFEFEF}
\multicolumn{5}{l}{\textbf{\textit{(Dynamic model merging)}}} \\
\multicolumn{1}{l|}{Twin-Merging~\cite{lu2024twin}}            & 84.0$_{(90.3)}$ & 70.0$_{(76.7)}$ & 57.5$_{(61.8)}$ & 60.1$_{(65.2)}$ \\
\multicolumn{1}{l|}{DaWin~\cite{oh2025dawin}}                  & 89.0$_{(95.3)}$ & 73.8$_{(80.3)}$ & 52.8$_{(57.7)}$ & 40.3$_{(42.9)}$ \\
\multicolumn{1}{l|}{MoW-Merging~\cite{ye2025dynamic}}          & 88.1$_{(94.8)}$ & 83.2$_{(91.5)}$ & 79.3$_{(86.8)}$ & 56.4$_{(60.6)}$ \\
\multicolumn{1}{l|}{WEMoE~\cite{tang2024merging}}              & 90.4$_{(97.3)}$ & 83.1$_{(90.7)}$ & 74.4$_{(81.2)}$ & 67.1$_{(72.4)}$ \\
\midrule

\rowcolor[HTML]{FFF2CC}
\multicolumn{1}{l|}{\textbf{\method{} (Ours)}}
& $\mathbf{92.0}_{(99.1)}$ & $\mathbf{89.8}_{(98.8)}$ & $\mathbf{89.4}_{(97.9)}$ & $\mathbf{91.2}_{(97.9)}$ \\
\bottomrule
\end{tabular}
}
\end{table}
\cref{tab:30_tasks_vision_vit_b_32} shows that performance of both static and dynamic merging baselines deteriorates as the task suite scales to 30 tasks.
In contrast, \method{} remains stable and continues to closely match Individual experts, indicating that offset-based expert recovery scales favorably to larger and more diverse task collections.

\subsection{Unseen-task generalization robustness}
\label{sec:ood_generalization}
\begin{table}[t!]
\centering
\caption{\textbf{Evaluation under OOD scenarios.}}
\label{tab:ood_robustness_table_3}

\setlength{\tabcolsep}{4pt}
\resizebox{\linewidth}{!}{
\begin{tabular}{lcc}
\toprule
\textbf{Method} & \textbf{Seen Task Accuracy} & \textbf{Unseen Task Accuracy} \\ 
\midrule
Weight Averaging & 63.7 & 57.8 \\
\rowcolor[HTML]{FFF2CC}
Weight Averaging + \method{} & $\mathbf{90.4}$ & $\mathbf{63.4}$ \\

\midrule
Task Arithmetic~\cite{ilharco2023editing} & 72.2 & 58.5 \\
\rowcolor[HTML]{FFF2CC}
Task Arithmetic + \method{} & $\mathbf{90.4}$ & $\mathbf{66.0}$ \\

\midrule
TIES-Merging~\cite{yadav2023ties} & 75.7 & 57.9 \\
\rowcolor[HTML]{FFF2CC}
TIES-Merging + \method{} & $\mathbf{89.8}$ & $\mathbf{68.7}$ \\

\midrule
AdaMerging~\cite{yang2024adamerging} & 76.6 & 66.8\\ 
\rowcolor[HTML]{FFF2CC}
AdaMerging + \method{} & $\mathbf{90.2}$ & $\mathbf{69.1}$ \\ 

\bottomrule
\end{tabular}
}
\end{table}

\noindent\textbf{Setting.}
This section evaluates whether \method{} maintains performance when evaluation inputs include tasks that are not used to build the merged checkpoint.
Following the protocol in AdaMerging~\cite{yang2024adamerging}, the evaluation uses eight vision datasets.
MNIST~\cite{deng2012mnist} and EuroSAT~\cite{helber2019eurosat} serve as unseen tasks, while the remaining six datasets (Cars~\cite{krause20133d}, DTD~\cite{cimpoi2014describing}, GTSRB~\cite{stallkamp2011german}, RESISC45~\cite{cheng2017remote}, SUN397~\cite{xiao2016sun}, and SVHN~\cite{netzer2011reading}) serve as seen tasks.
The merged checkpoint construction and \method{} recovery training use only seen-task experts, and no unseen-task expert checkpoint is available.
Appendix~\cref{app:experiment_details} provides further details.

\noindent\textbf{Soft recovery beyond one-hot task selection.}
The task identifier in \cref{sec:task_identification} produces a hard task ID by nearest-subspace matching.
For unseen tasks, a hard assignment to a single seen task becomes a brittle approximation.
To enable interpolation across seen experts, this section replaces the one-hot task indicator with a soft mixing vector over seen tasks.
Concretely, projection residuals $\{\mathcal{R}_t(\mathbf{z})\}_{t=1}^{T}$ are converted into mixture weights $\mathbf{p}(\mathbf{z})\in\mathbb{R}^{T}$ (for example, by applying a softmax to the negative residuals), and \method{} conditions offset prediction on $\mathbf{p}(\mathbf{z})$ instead of a one-hot task vector.
To support this behavior without using any task data, recovery training uses soft targets in parameter space: a mixing vector is sampled from a Dirichlet distribution and defines a convex combination of seen experts as the supervision target.

\noindent\textbf{Results.}
\cref{tab:ood_robustness_table_3} reports performance when \method{} is applied on top of multiple merged checkpoints.
\method{} consistently improves accuracy on both seen tasks and unseen tasks across all merging baselines.
Seen-task accuracy increases to around 90\% for every merged model.
Unseen-task accuracy also improves for every baseline, and the best unseen-task performance is obtained by combining \method{} with AdaMerging~\cite{yang2024adamerging}, reaching 69.1\%.
These results support the claim in the introduction that offset-based recovery enables interpolation over seen expertise and improves unseen-task generalization robustness.

\subsection{NLP tasks}
\label{sec:nlp_tasks}

\noindent\textbf{Setting.}
This section evaluates \method{} on a multi-task NLP suite using T5-large as a common pre-trained backbone.
A separate task expert is obtained by fine-tuning the backbone on each task, and merged methods are evaluated under task-agnostic inference.
The 7-task scenario comprises:
(1) QASC~\cite{khot2020qasc},
(2) WikiQA~\cite{yang2015wikiqa},
(3) QuaRTz~\cite{tafjord2019quartz} for question answering,
(4) PAWS~\cite{zhang2019paws} for paraphrase identification,
(5) Story Cloze~\cite{sharma2018tack} for sentence completion,
(6) Winogrande~\cite{sakaguchi2020wino},
and (7) WSC~\cite{levesque2012wsc} for coreference resolution.
All numbers are reported as percentages, and values in parentheses denote normalized accuracy (merged / fine-tuned), similar to \cref{tab:nlpacc}.

\noindent\textbf{Results.}
\begin{table*}[t]
\caption{\textbf{Multi-task performance of merged T5-large across seven NLP tasks.}
Numbers in parentheses $_{(\cdot)}$ denote normalized accuracy (merged / fine-tuned).
Bold values indicate the best performance among merged methods (excluding fine-tuned).}
\centering
\setlength\tabcolsep{4.5pt}
\renewcommand{\arraystretch}{1.2}
\footnotesize
\resizebox{\textwidth}{!}{
\begin{tabular}{l|ccccccc|c}
\toprule
\multicolumn{1}{l|}{\textbf{Method}} & PAWS & QASC & QuaRTz & StoryCloze & WikiQA & Winogrande & WSC & \textbf{Avg.} \\
\midrule
\multicolumn{1}{l|}{Fine-tuned} & 94.4 & 98.9 & 87.8 & 90.8 & 96.0 & 74.7 & 79.2 & 88.8 \\
\midrule

\rowcolor[HTML]{EFEFEF}
\multicolumn{9}{l}{\textbf{\textit{(Static model merging)}}} \\
\multicolumn{1}{l|}{Weight Averaging} & 61.3$_{(64.9)}$ & 82.6$_{(83.5)}$ & 70.5$_{(80.3)}$ & 53.7$_{(59.1)}$ & 63.2$_{(65.8)}$ & 49.7$_{(66.5)}$ & 36.1$_{(45.6)}$ & 59.6$_{(67.1)}$ \\
\multicolumn{1}{l|}{Task Arithmetic~\cite{ilharco2023editing}} & 77.8$_{(82.4)}$ & 96.0$_{(97.1)}$ & 78.6$_{(89.5)}$ & 86.4$_{(95.2)}$ & 59.1$_{(61.6)}$ & 62.3$_{(83.4)}$ & 52.8$_{(66.7)}$ & 73.3$_{(82.5)}$ \\
\multicolumn{1}{l|}{TIES-Merging~\cite{yadav2023ties}} & 81.5$_{(86.3)}$ & 96.2$_{(97.3)}$ & 80.1$_{(91.2)}$ & 83.6$_{(92.1)}$ & 64.9$_{(67.6)}$ & 66.5$_{(89.0)}$ & 65.3$_{(82.4)}$ & 76.9$_{(86.6)}$ \\
\multicolumn{1}{l|}{Consensus TA~\cite{wang2024localizing}} 
& 77.7$_{(82.3)}$ & 67.4$_{(68.1)}$ & 65.5$_{(74.6)}$ & 79.9$_{(88.0)}$ & 90.3$_{(94.1)}$ & 78.9$_{(105.6)}$ & 56.0$_{(70.7)}$ & 73.7$_{(83.0)}$ \\
\multicolumn{1}{l|}{AdaMerging~\cite{yang2024adamerging}}
& 58.5$_{(62.0)}$ & 76.0$_{(76.8)}$ & 75.5$_{(86.0)}$ & 82.1$_{(90.4)}$
& 93.5$_{(97.4)}$ & 65.0$_{(87.0)}$ & 64.4$_{(81.3)}$ & 73.6$_{(82.9)}$ \\
% \multicolumn{1}{l|}{TSV-M~\cite{gargiulo2024task}} & 85.0$_{(90.0)}$ & 92.4$_{(93.4)}$ & 66.0$_{(75.2)}$ & 81.1$_{(89.3)}$ & 88.4$_{(92.1)}$ & 87.0$_{(116.5)}$ & 57.0$_{(72.0)}$ & 79.6$_{(89.6)}$ \\
\midrule

\rowcolor[HTML]{EFEFEF}
\multicolumn{9}{l}{\textbf{\textit{(Dynamic model merging)}}} \\
\multicolumn{1}{l|}{Twin-Merging~\cite{lu2024twin}} & 93.2$_{(98.7)}$ & 96.0$_{(97.1)}$ & 87.1$_{(99.2)}$ & 85.6$_{(94.3)}$ & 77.1$_{(80.3)}$ & 70.8$_{(94.8)}$ & 69.7$_{(88.0)}$ & 82.8$_{(93.2)}$ \\
\multicolumn{1}{l|}{DaWin~\cite{oh2025dawin}} & 85.6$_{(90.7)}$ & 96.1$_{(97.2)}$ & 81.4$_{(92.7)}$ & 73.2$_{(80.6)}$ & 70.2$_{(73.1)}$ & 68.7$_{(92.0)}$ & 57.9$_{(73.1)}$ & 76.2$_{(85.8)}$ \\
\multicolumn{1}{l|}{WEMoE~\cite{tang2024merging}} & 93.4$_{(98.9)}$ & 95.4$_{(96.5)}$ & 74.5$_{(84.9)}$ & 79.6$_{(87.7)}$ & 94.1$_{(98.0)}$ & 90.4$_{(121.0)}$ & 57.0$_{(72.0)}$ & 83.4$_{(93.9)}$ \\
\multicolumn{1}{l|}{MoW-Merging~\cite{ye2025dynamic}} & 94.7$_{(100.3)}$ & 90.2$_{(91.2)}$ & 79.0$_{(90.0)}$ & 83.1$_{(91.5)}$ & 56.6$_{(59.0)}$ & 93.8$_{(125.6)}$ & 51.0$_{(64.4)}$ & 78.3$_{(88.2)}$ \\
\midrule

\rowcolor[HTML]{FFF2CC}
\textbf{\method{} (Ours)} 
& $\mathbf{96.1}_{(101.8)}$ & $\mathbf{97.4}_{(98.5)}$ & $\mathbf{83.5}_{(95.1)}$ & $\mathbf{90.3}_{(99.4)}$ & $\mathbf{95.7}_{(99.7)}$ & $\mathbf{95.0}_{(127.2)}$ & $\mathbf{48.0}_{(60.6)}$ & $\mathbf{86.6}_{(97.5)}$ \\
\bottomrule
\end{tabular}}
\label{tab:nlpacc}
\end{table*}
\cref{tab:nlpacc} reports multi-task performance of merged T5-large across the seven NLP tasks.
Static model merging baselines, including Task Arithmetic~\cite{ilharco2023editing}, TIES-Merging~\cite{yadav2023ties}, and TSV-M~\cite{gargiulo2024task}, show clear degradation relative to fine-tuned experts under multi-task interference.
Dynamic model merging baselines, including Twin-Merging~\cite{lu2024twin}, DaWin~\cite{oh2025dawin}, WEMoE~\cite{tang2024merging}, and MoW-Merging~\cite{ye2025dynamic}, improve over static merging by using input-adaptive composition, but they still fall short of expert-level performance on average.

Across the full suite, \method{} achieves the strongest average performance among merged methods, reaching $86.6$ average accuracy and $97.5\%$ normalized accuracy relative to fine-tuned experts.
\method{} matches or exceeds fine-tuned experts on several tasks such as PAWS and Winogrande, while WSC remains challenging and some baselines obtain higher scores on that task.
Overall, the results indicate that \method{} scales to multi-task NLP settings and recovers expert-level performance from a single merged model with high fidelity.

% \noindent\textbf{Shared Layer generation}

\subsection{Ablation study}
\label{sec:ablation}

\begin{figure*}[t]
  \centering
  \captionsetup[subfigure]{labelformat=empty}

  \begin{subfigure}{0.32\linewidth}
    \centering
    \includegraphics[width=\linewidth]{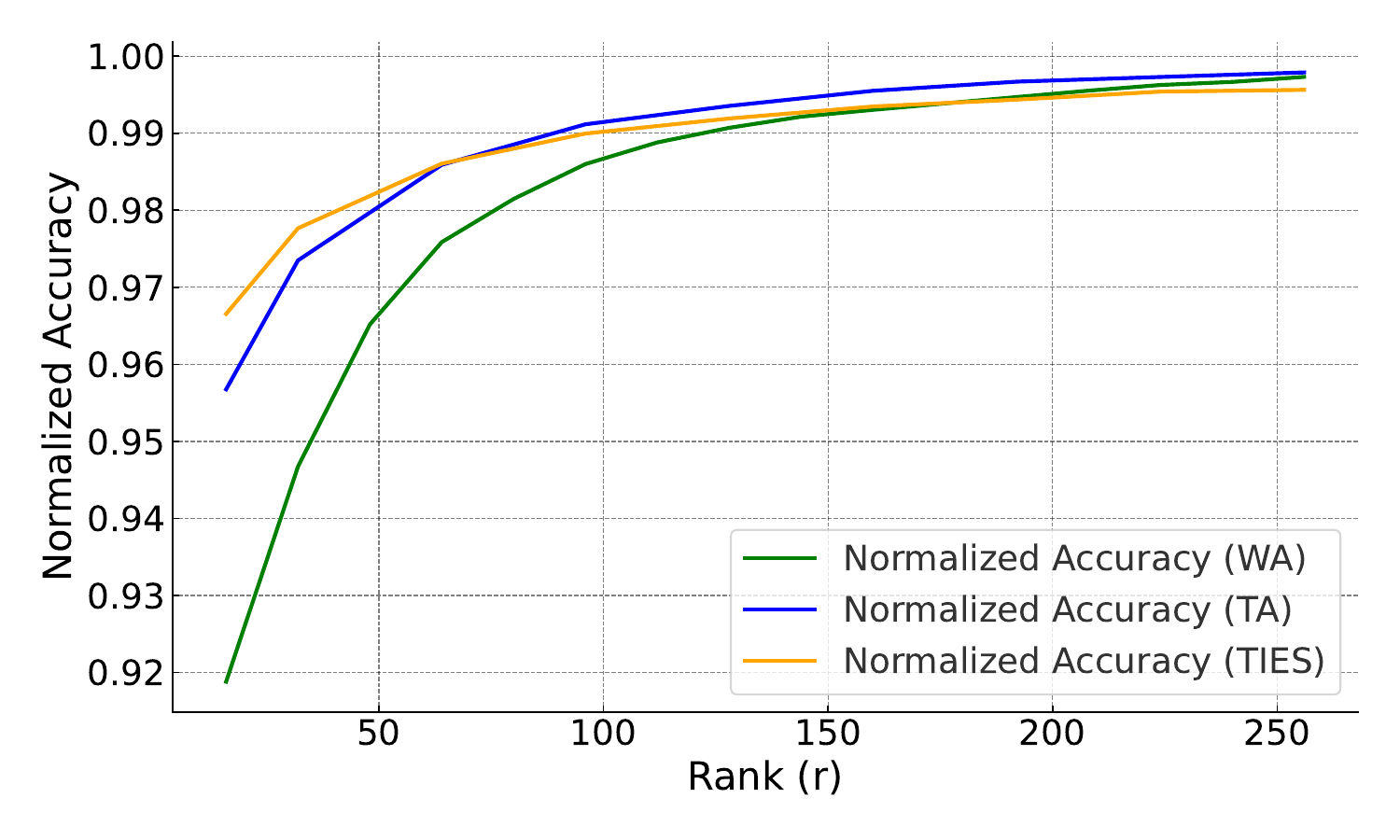}
    \caption{\small (a) Effect of the base merged model across different ranks $r$.}
    \label{fig:normalized_accuracy_methods}
  \end{subfigure}
  \hfill
  \begin{subfigure}{0.32\linewidth}
    \centering
    \includegraphics[width=\linewidth]{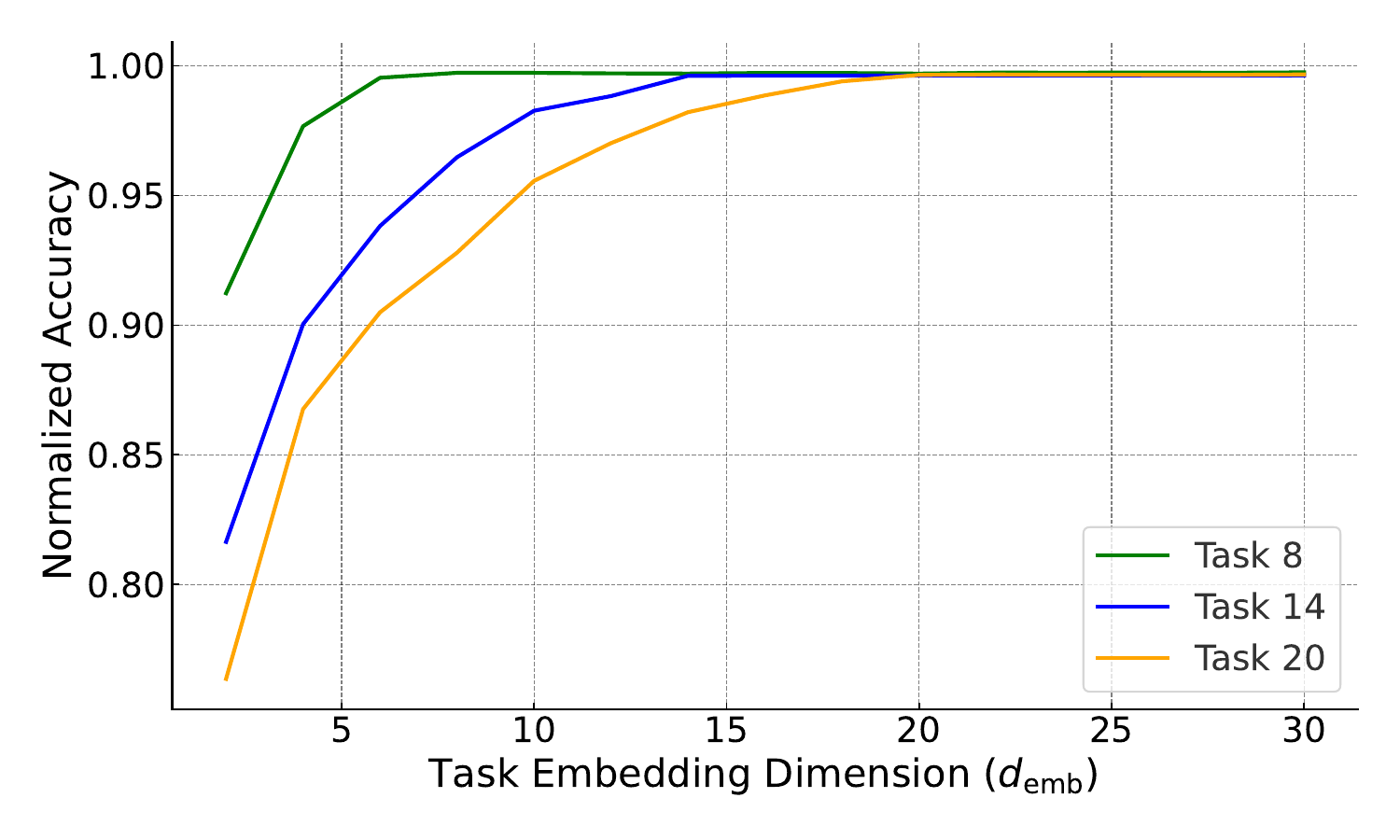}
    \caption{\small (b) Effect of task embedding dimension $d_{\mathrm{emb}}$ for 8/14/20 tasks.}
    \label{fig:task_embedding_vs_accuracy}
  \end{subfigure}
  \hfill
  \begin{subfigure}{0.32\linewidth}
    \centering
    \includegraphics[width=\linewidth]{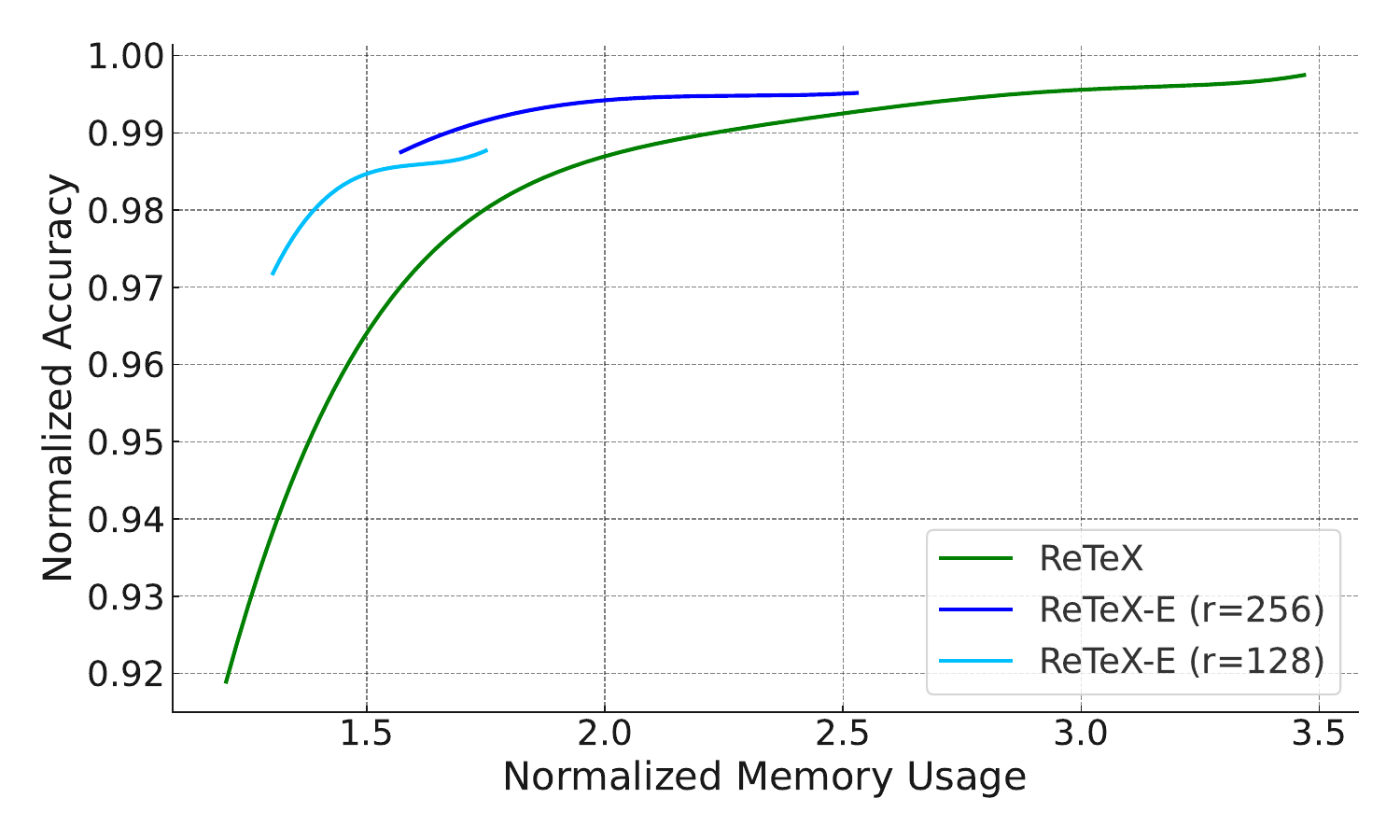}
    \caption{\small (c) Memory--accuracy trade-off of \method{} and \method{}-E at outer rank $r$.}
    \label{fig:advanced_efficiency}
  \end{subfigure}

  \caption{\small
    Performance and efficiency trade-offs of \method{}.
    (a) Normalized accuracy across ranks $r$ for three base merged models (WA, TA, TIES).
    (b) Normalized accuracy as a function of task embedding dimension $d_{\mathrm{emb}}$ for 8, 14, and 20 tasks with ViT-B/32 and fixed $r{=}256$.
    (c) Normalized accuracy vs.\ normalized memory for \method{} and \method{}-E at fixed outer rank $r$ while varying $r_g$.
    Normalized memory is computed as the storage of recovery parameters divided by the storage of all task experts.
  }
  \label{fig:ablation_overview}
\end{figure*}

We conduct ablation studies for further analysis.
Unless specified otherwise, this subsection follows the ViT-B/32 protocol from~\cite{wang2024localizing}.

\noindent\textbf{Integration of \method{} with different base merged models.}
This ablation evaluates whether recovery quality depends on the base merged checkpoint.
\method{} is trained on top of Weight Averaging, Task Arithmetic~\cite{ilharco2023editing}, and TIES-Merging~\cite{yadav2023ties}, while keeping the base merge rule fixed during recovery training.
\cref{fig:normalized_accuracy_methods} shows that using a stronger base merge improves normalized accuracy, with larger gains at smaller recovery ranks $r$.
This indicates that expert recovery benefits from a better initialization in weight space, especially under tight low-rank budgets.

\noindent\textbf{Task embedding dimension.}
This ablation studies the effect of task embedding dimension $d_{\mathrm{emb}}$ on recovery.
We investigate the influence of the task embedding dimension, $d_{\mathrm{emb}}$, on the recovery performance of \method{}.
For this analysis, we fix the rank $r{=}256$ and vary $d_{\mathrm{emb}}$ while evaluating on configurations with 8, 14, and 20 tasks.
\cref{fig:task_embedding_vs_accuracy} shows that normalized accuracy increases with $d_{\mathrm{emb}}$ and then saturates.
High recovery is achieved once $d_{\mathrm{emb}}$ is comparable to the number of tasks, and increasing $d_{\mathrm{emb}}$ beyond this point yields diminishing returns.

\textcolor{black}{\subsection{Efficient low-rank offset factorization}}
\label{sec:advanced_efficiency}

\noindent\textbf{Two-stage offset generation.}
\method{} uses a low-rank offset per layer,
$\bm{\beta}_{t}^{(l)}=\bm{\beta}_{t,g}^{(l)}\bm{\beta}_{s}^{(l)}$,
where $\bm{\beta}_{t,g}^{(l)}\in\mathbb{R}^{a\times r}$ is task-conditioned and
$\bm{\beta}_{s}^{(l)}\in\mathbb{R}^{r\times b}$ is shared.
The dominant parameter cost comes from producing $\bm{\beta}_{t,g}^{(l)}$.

\method{}-E reduces this cost by factorizing the generator output:
$\bm{\beta}_{t,g}^{(l)}=\bm{\beta}_{t,gA}^{(l)}\bm{\beta}_{gB}^{(l)}$ with
$\bm{\beta}_{t,gA}^{(l)}\in\mathbb{R}^{a\times r_g}$,
$\bm{\beta}_{gB}^{(l)}\in\mathbb{R}^{r_g\times r}$, and $r_g<r$.
The final offset becomes
$\bm{\beta}_{t}^{(l)}=(\bm{\beta}_{t,gA}^{(l)}\bm{\beta}_{gB}^{(l)})\bm{\beta}_{s}^{(l)}$.

\noindent\textbf{Results.}
\cref{fig:advanced_efficiency} fixes the outer rank $r$ and varies $r_g$.
\method{}-E improves the memory--accuracy trade-off and maintains high recovery quality over a wide range of $r_g$.
The corresponding inference memory usage is reported in \cref{tab:cost}.

\subsection{Computational cost}
\label{sec:cost}
\begin{table}[t]
  \centering
  \renewcommand{\arraystretch}{}
  \caption{\textbf{Inference cost with CLIP ViT-B/32.}
  Per-input inference latency and peak VRAM usage in the 8-task vision setting with task-ID unknown, together with average accuracy (Avg.) on ViT-B/32, measured on an NVIDIA GeForce RTX 3090.}
  \label{tab:cost}

  \resizebox{\linewidth}{!}{
  \begin{tabular}{l|ccc}
    \toprule
    \textbf{Method} & \textbf{Inference cost (per input)} & \textbf{VRAM (GB)} & \textbf{Avg.} \\
    \midrule

    \rowcolor[HTML]{EFEFEF}
    \multicolumn{4}{l}{\textbf{\textit{(Static model merging)}}} \\
    Weight Averaging                            & 0.0004s & 1.3 & 66.3 \\
    Task Arithmetic~\cite{ilharco2023editing}  & 0.0004s & 1.3 & 70.8 \\
    TIES-Merging~\cite{yadav2023ties}          & 0.0004s & 1.3 & 75.1 \\
    Consensus TA~\cite{wang2024localizing}     & 0.0004s & 1.3 & 75.0 \\
    \midrule

    \rowcolor[HTML]{EFEFEF}
    \multicolumn{4}{l}{\textbf{\textit{(Dynamic model merging)}}} \\
    Twin-Merging~\cite{lu2024twin}             & 0.03s  & 3.2 & 84.0 \\
    DaWin~\cite{oh2025dawin}                   & 0.63s  & 5.5 & 89.0 \\
    WEMoE~\cite{tang2024merging}               & 0.02s  & 4.3 & 90.4 \\
    MoW-Merging~\cite{ye2025dynamic}           & 0.008s & 4.9 & 88.1 \\
    \midrule

    \rowcolor[HTML]{EFEFEF}
    \multicolumn{4}{l}{\textbf{\textit{(Component selection)}}} \\
    ID-LoRA~\cite{hu2022lora}                  & 0.0009s & 1.9 & 89.3 \\
    ID-EMR-Merging~\cite{huang2024emr}         & 0.004s  & 1.9 & 90.5 \\
    \midrule

    \rowcolor[HTML]{FFF2CC}
    \textbf{\method{} (Sample-wise)}           & 0.05s   & 2.5 & 92.2 \\
    \rowcolor[HTML]{FFF2CC}
    \textbf{\method{} (Group: $B=64$)}         & 0.0015s & 2.7 & 92.2 \\
    \rowcolor[HTML]{FFF2CC}
    \textbf{\method{}-E (Sample-wise)}         & 0.07s   & 1.8 & 92.0 \\
    \rowcolor[HTML]{FFF2CC}
    \textbf{\method{}-E (Group: $B=64$)}       & 0.0015s & 2.0 & 92.0 \\
    \bottomrule
  \end{tabular}
  }
\end{table}

\noindent\textbf{Batch inference.}
\method{} supports efficient batched execution by grouping samples that share the same predicted task ID $\hat{t}$ (\cref{sec:task_identification}).
Given a batch of $B$ test inputs, task predictions $\{\hat{t}_i\}_{i=1}^{B}$ are computed and then:
(i) partition indices by task ID,
$\mathcal{I}_t = \{ i \in \{1,\dots,B\} : \hat{t}_i = t \}$,
and for each $t$ with $|\mathcal{I}_t| > 0$ recover the corresponding expert once from the single merged checkpoint;
(ii) run each sub-batch $X_{\mathcal{I}_t}$ through the recovered expert and scatter outputs back to the original order.
This reduces the number of recoveries to the number of unique task IDs in the batch, $K \le \min(B, T)$.

\noindent\textbf{Component selection.}
Component-based multi-task deployment stores a shared backbone together with task-specific components and selects components conditioned on task identity.
To test whether task identification plus component selection is sufficient for a task-unknown stream, per-task LoRA adapters and EMR-Merging components are evaluated by selecting components using predicted task IDs from \cref{sec:task_identification}.
\textcolor{black}{For a fair latency comparison, ID-LoRA and ID-EMR-Merging are also evaluated under the same grouped inference protocol as \method{}.}
\textcolor{black}{After predicting task IDs, samples with the same predicted task ID are grouped and processed together with batch size $B=64$.}
\textcolor{black}{The reported latency is the per-input inference cost, obtained by normalizing the measured grouped latency by the batch size.}
In \cref{tab:cost}, these baselines use the prefix ID-.

This comparison evaluates an alternative deployment path that combines task identification with compact components, and clarifies when expert recovery provides additional value beyond component selection.

\noindent\textbf{Results.}
\cref{tab:cost} reports latency and peak VRAM in the 8-task vision setting without task IDs with CLIP ViT-B/32.
Latency includes task identification, offset generation, and the forward pass with the recovered parameters.
Static merging baselines are fastest (0.0004s per input) with low VRAM (1.3GB), but accuracy remains low (66.3--75.1 Avg.).
Dynamic merging baselines improve accuracy (84.0--90.4) but incur higher VRAM (4.3--5.6GB) and additional latency (0.008--0.63s) due to inference-time access to task components.

Component selection baselines reduce VRAM to 1.9GB and run quickly (0.0009--0.004s), but accuracy remains below \method{} (89.3--90.5 vs.\ 92.2 Avg.) and deployment still depends on storing and selecting task-specific components.

\method{} achieves the best accuracy (92.2 Avg.) with 2.5GB VRAM in the sample-wise setting.
With grouped inference at $B{=}64$, \method{} reduces latency to 0.0015s per input while maintaining accuracy, with a small VRAM increase to 2.7GB.
\method{}-E further reduces VRAM (1.8GB sample-wise; 2.0GB grouped) with minor accuracy change.

\section{Conclusion}
%In this work, we aim to recover task-expert-level performance while reducing memory usage overhead.
In this work, instead of storing task experts during inference, we aim to recover task experts directly from a merged model, retrieving task-expert-level performance while removing the need for storing experts during inference.
We note that the task-specific offset between the task-expert parameters and the merged model parameters can be recovered from the merged model.
Building upon this, we introduce a new model merging approach that learns to \textbf{Re}cover \textbf{T}ask \textbf{eX}perts (\textbf{\method{}}) from a merged model by predicting these offsets.
Particularly, our framework first estimates the task identity for a given input.
Conditioned on the estimated task identity, our framework generates a low-rank, task-dependent offset that maps the merged parameters to the corresponding expert for that input.
Experimental results across vision and NLP domains highlight the effectiveness of \method{} in recovering task-expert-level performance while reducing memory overhead, compared to previous input-adaptive merging methods.
We hope that this work encourages future research on the relationship between a merged model and task-specific models.
%, and on more efficient approaches to model merging via offset recovery.

% ---------------------------------------------------------------
% Acknowledgements
% ---------------------------------------------------------------

\section*{Acknowledgements}

This work was supported by the Institute of Information and Communications
Technology Planning and Evaluation (IITP) grant
(No. RS-2025-25422680, No. RS-2020-II201373,
Artificial Intelligence Graduate School Program (Hanyang University))
and the National Research Foundation of Korea (NRF) grant funded by the
Korea government (MSIT) (No. RS-2025-24533064).

% ===============================================================
% References
% ===============================================================

\FloatBarrier

\bibliographystyle{splncs04}
\bibliography{main}

@String(CVPR  = {IEEE Conf. Comput. Vis. Pattern Recog.})

@String(ICCV  = {Int. Conf. Comput. Vis.})

@String(ECCV  = {Eur. Conf. Comput. Vis.})

@String(NeurIPS = {Adv. Neural Inform. Process. Syst.})

@String(ICML  = {Int. Conf. Mach. Learn.})

@String(ICLR  = {Int. Conf. Learn. Represent.})

@String(AAAI  = {AAAI})

@String(CVPR  = {CVPR})

@String(ICCV  = {ICCV})

@String(ECCV  = {ECCV})

@String(NeurIPS = {NeurIPS})

@String(ICML  = {ICML})

@String(ICLR  = {ICLR})

@inproceedings{ilharco2023editing,
title={Editing models with task arithmetic},
author={Ilharco, Gabriel and Ribeiro, Marco Tulio and Wortsman, Mitchell and Schmidt, Ludwig and Hajishirzi, Hannaneh and Farhadi, Ali},
booktitle={ICLR},
year={2023}
}

@inproceedings{yadav2023ties,
title={TIES-Merging: Resolving Interference When Merging Models},
author={Yadav, Prateek and Tam, Derek and Choshen, Leshem and Raffel, Colin and Bansal, Mohit},
booktitle={NeurIPS},
year={2023}
}

@inproceedings{matena2022fisher,
  title={Merging models with Fisher-weighted averaging},
  author={Matena, Michael S.~and Raffel, Colin A.},
  booktitle={NeurIPS},
  year={2022}
}

@inproceedings{yang2024adamerging,
  title={AdaMerging: Adaptive Model Merging for Multi-Task Learning},
  author={Yang, Enneng and Wang, Zhenyi and Shen, Li and Liu, Shiwei and Guo, Guibing and Wang, Xingwei and Tao, Dacheng},
  booktitle={ICLR},
  year={2024}
}

@inproceedings{huang2024emr,
  title={EMR-Merging: Tuning-Free High-Performance Model Merging},
  author={Huang, Chenyu and Ye, Peng and Chen, Tao and He, Tong and Yue, Xiangyu and Ouyang, Wanli},
  booktitle={NeurIPS},
  year={2024}
}

@article{achiam2023gpt,
  title={Gpt-4 technical report},
  author={Achiam, Josh and Adler, Steven and Agarwal, Sandhini and Ahmad, Lama and Akkaya, Ilge and Aleman, Florencia Leoni and Almeida, Diogo and Altenschmidt, Janko and Altman, Sam and Anadkat, Shyamal and others},
  journal={arXiv preprint arXiv:2303.08774},
  year={2023},
}

@article{saab2024capabilities,
    author = {Saab, Khaled and Tu, Tao and Weng, Wei-Hung and Tanno, Ryutaro and Stutz, David and Wulczyn, Ellery and Zhang, Fan and Strother, Tim and Park, Chunjong and Vedadi, Elahe and others},
    title = {Capabilities of gemini models in medicine},
    journal = {arXiv preprint arXiv:2404.18416},
    year = {2024}
}

@article{ding2023parameter,
  title={Parameter-efficient fine-tuning of large-scale pre-trained language models},
  author={Ding, Ning and Qin, Yujia and Yang, Guang and Wei, Fuchao and Yang, Zonghan and Su, Yusheng and Hu, Shengding and Chen, Yulin and Chan, Chi-Min and Chen, Weize and others},
  journal={Nature Machine Intelligence},
  year={2023},
}

@inproceedings{lu2024twin,
  title={Twin-merging: Dynamic integration of modular expertise in model merging},
  author={Lu, Zhenyi and Fan, Chenghao and Wei, Wei and Qu, Xiaoye and Chen, Dangyang and Cheng, Yu},
  booktitle={NeurIPS},
  year={2024}
}

@article{xiao2016sun,
  title={Sun database: Exploring a large collection of scene categories},
  author={Xiao, Jianxiong and Ehinger, Krista A. and Hays, James and Torralba, Antonio and Oliva, Aude},
  journal={International Journal of Computer Vision},
  year={2016},
}

@inproceedings{krause20133d,
  title={3d object representations for fine-grained categorization},
  author={Krause, Jonathan and Stark, Michael and Deng, Jia and Fei-Fei, Li},
  booktitle={ICCV Workshop},
  year={2013}
}

@article{cheng2017remote,
  title={Remote sensing image scene classification: Benchmark and state of the art},
  author={Cheng, Gong and Han, Junwei and Lu, Xiaoqiang},
  journal={Proceedings of the IEEE},
  year={2017},
}

@article{helber2019eurosat,
  title={Eurosat: A novel dataset and deep learning benchmark for land use and land cover classification},
  author={Helber, Patrick and Bischke, Benjamin and Dengel, Andreas and Borth, Damian},
  journal={IEEE Journal of Selected Topics in Applied Earth Observations and Remote Sensing},
  year={2019},
}

@inproceedings{netzer2011reading,
  title={Reading digits in natural images with unsupervised feature learning},
  author={Netzer, Yuval and Wang, Tao and Coates, Adam and Bissacco, Alessandro and Wu, Baolin and Ng, Andrew Y. and others},
  booktitle={NIPS Workshop},
  year={2011},
}

@inproceedings{stallkamp2011german,
  title={The German traffic sign recognition benchmark: a multi-class classification competition},
  author={Stallkamp, Johannes and Schlipsing, Marc and Salmen, Jan and Igel, Christian},
  booktitle={IJCNN},
  year={2011},
}

@article{deng2012mnist,
  title={The mnist database of handwritten digit images for machine learning research [best of the web]},
  author={Deng, Li},
  journal={IEEE signal processing magazine},
  year={2012},
}

@inproceedings{cimpoi2014describing,
  title={Describing textures in the wild},
  author={Cimpoi, Mircea and Maji, Subhransu and Kokkinos, Iasonas and Mohamed, Sammy and Vedaldi, Andrea},
  booktitle={CVPR},
  year={2014}
}

@inproceedings{coates2011analysis,
  title={An analysis of single-layer networks in unsupervised feature learning},
  author={Adam Coates and Andrew Ng and Honglak Lee},
  booktitle={AISTATS},
  year={2011}
}

@article{krizhevsky2009learning,
  title={Learning multiple layers of features from tiny images},
  author={Alex Krizhevsky and Geoffrey Hinton and others},
  year={2009},
  publisher={Toronto, ON, Canada}
}

@inproceedings{nilsback2008automated,
  title={Automated Flower Classification over a Large Number of Classes},
  author={Maria-Elena Nilsback and Andrew Zisserman},
  booktitle={ICVGIP},
  year={2008}
}

@inproceedings{parkhi2012cats,
  title={Cats and dogs},
  author={Omkar M Parkhi and Andrea Vedaldi and Andrew Zisserman and C. V. Jawahar},
  booktitle={CVPR},
  year={2012}
}

@inproceedings{veeling2018rotation,
  title={Rotation Equivariant CNNs for Digital Pathology},
  author={Bastiaan S. Veeling and Jasper Linmans and Jim Winkens and Taco Cohen and Max Welling},
  booktitle={MICCAI},
  year={2018}
}

@inproceedings{goodfellow2013challenges,
  title={Challenges in Representation Learning: A Report on Three Machine Learning Contests},
  author={Ian J. Goodfellow and Dumitru Erhan and Pierre Luc Carrier and Aaron Courville and Mehdi Mirza and Ben Hamner and Will Cukierski and Yichuan Tang and David Thaler and Dong-Hyun Lee and Yingbo Zhou and Chetan Ramaiah and Fangxiang Feng and Ruifan Li and Xiaojie Wang and Dimitris Athanasakis and John Shawe-Taylor and Maxim Milakov and John Park and Radu Ionescu and Marius Popescu and Cristian Grozea and James Bergstra and Jingjing Xie and Lukasz Romaszko and Bing Xu and Zhang Chuang and Yoshua Bengio},
  booktitle={ICONIP},
  year={2013}
}

@inproceedings{cohen2017emnist,
  title={EMNIST: Extending MNIST to handwritten letters},
  author={Gregory Cohen and Saeed Afshar and Jonathan Tapson and André van Schaik},
  booktitle={IJCNN},
  year={2017}
}

@inproceedings{bossard2014food,
  title = {Food-101 -- Mining Discriminative Components with Random Forests},
  author = {Lukas Bossard and Matthieu Guillaumin and Luc Van Gool},
  booktitle = {ECCV},
  year = {2014}
}

@article{xiao2017fashion,
  title={Fashion-MNIST: a Novel Image Dataset for Benchmarking Machine Learning Algorithms},
  author={Han Xiao and Kashif Rasul and Roland Vollgraf},
  journal={arXiv preprint arXiv:1708.07747},
  year={2017}
}

@inproceedings{socher2013recursive,
  title = {Recursive Deep Models for Semantic Compositionality Over a Sentiment Treebank},
  author = {Richard Socher and Alex Perelygin and Jean Wu and Jason Chuang and Christopher D. Manning and Andrew Ng and Christopher Potts},
  booktitle = {EMNLP},
  year = {2013}
}

@inproceedings{clanuwat2018deep,
    author={Tarin Clanuwat and Mikel Bober-Irizar and Asanobu Kitamoto and Alex Lamb and Kazuaki Yamamoto and David Ha},
    title={Deep Learning for Classical Japanese Literature},
    booktitle = {NeurIPS},
    year = {2018}
}

@inproceedings{wang2024localizing,
title={Localizing Task Information for Improved Model Merging and Compression},
author={Ke Wang and Nikolaos Dimitriadis and Guillermo Ortiz-Jimenez and François Fleuret and Pascal Frossard},
booktitle={ICML},
year={2024},
}

@inproceedings{perez2018film,
  title={Film: Visual reasoning with a general conditioning layer},
  author={Perez, Ethan and Strub, Florian and De Vries, Harm and Dumoulin, Vincent and Courville, Aaron},
  booktitle={AAAI},
  year={2018}
}

@inproceedings{shoemake1985animating,
author={Ken Shoemake},
title={Animating rotation with quaternion curves},
booktitle={SIGGRAPH},
year={1985}
}

@inproceedings{utans1996weight,
author={Joachim Utans},
title={Weight averaging for neural networks and local resampling schemes},
booktitle={AAAI Workshop},
year={1996},
}

@inproceedings{yu2024language,
  title={Language models are super mario: Absorbing abilities from homologous models as a free lunch},
  author={Yu, Le and Yu, Bowen and Yu, Haiyang and Huang, Fei and Li, Yongbin},
  booktitle={ICML},
  year={2024}
}

@article{muqeeth2023soft,
  title={Soft merging of experts with adaptive routing},
  author={Muqeeth, Mohammed and Liu, Haokun and Raffel, Colin},
  journal={Transactions on Machine Learning Research},
  year={2023}
}

@inproceedings{oh2025dawin,
title={DaWin: Training-free Dynamic Weight Interpolation for Robust Adaptation},
author={Changdae Oh and Yixuan Li and Kyungwoo Song and Sangdoo Yun and Dongyoon Han},
booktitle=ICLR,
year={2025},
}

@InProceedings{tang2024merging,
title = {Merging Multi-Task Models via Weight-Ensembling Mixture of Experts},
author = {Anke Tang and Li Shen and Yong Luo and Nan Yin and Lefei Zhang and Dacheng Tao},
booktitle = {ICML},
year = {2024},
}

@inproceedings{jin2023dataless,
title={Dataless Knowledge Fusion by Merging Weights of Language Models},
author={Jin, Xisen and Ren, Xiang and Preotiuc-Pietro, Daniel and Cheng, Pengxiang},
booktitle={ICLR},
year={2023}
}

@inproceedings{ortiz2023task,
title={Task arithmetic in the tangent space: Improved editing of pre-trained models},
author={Ortiz-Jimenez, Guillermo and Favero, Alessandro and Frossard, Pascal},
booktitle={NeurIPS},
year={2023}
}

@inproceedings{davari2024model,
  title={Model breadcrumbs: Scaling multi-task model merging with sparse masks},
  author={Davari, MohammadReza and Belilovsky, Eugene},
  booktitle={ECCV},
  year={2024},
}

@inproceedings{kang2024self,
    author={Kang, Junmo and Karlinsky, Leonid and Luo, Hongyin and Wang, Zhen and Hansen, Jacob and Glass, James and Cox, David and Panda, Rameswar and Feris, Rogerio and Ritter, Alan},
    title={Self-moe: Towards compositional large language models with self-specialized experts},
    booktitle = {ICLR},
    year = {2025}
}

@inproceedings{li2023merge,
    author={Li, Pingzhi and Zhang, Zhenyu and Yadav, Prateek and Sung, Yi-Lin and Cheng, Yu and Bansal, Mohit and Chen, Tianlong},
    title={Merge, then compress: Demystify efficient smoe with hints from its routing policy},
    booktitle = {ICLR},
    year = {2024}
}

@article{tang2023concrete,
  title={Concrete subspace learning based interference elimination for multi-task model fusion},
  author={Tang, Anke and Shen, Li and Luo, Yong and Ding, Liang and Hu, Han and Du, Bo and Tao, Dacheng},
  journal={arXiv preprint arXiv:2312.06173},
  year={2023}
}

@inproceedings{kingma2014adam-key,
  author={Kingma, Diederik P},
    title ={Adam: A method for stochastic optimization} ,
    booktitle ={ICLR} ,
    year = {2015}
}

@inproceedings{radford2021llearning,
    author={Radford, Alec and Kim, Jong Wook and Hallacy, Chris and Ramesh, Aditya and Goh, Gabriel and Agarwal, Sandhini and Sastry, Girish and Askell, Amanda and Mishkin, Pamela and Clark, Jack and Krueger, Gretchen and Sutskever, Ilya},
    title={Learning Transferable Visual Models From Natural Language Supervision},
    booktitle={ICML},
    year={2021} 
}

@article{xiong2024multi,
  title={Multi-task model merging via adaptive weight disentanglement},
  author={Xiong, Feng and Cheng, Runxi and Chen, Wang and Zhang, Zhanqiu and Guo, Yiwen and Yuan, Chun and Xu, Ruifeng},
  journal={arXiv preprint arXiv:2411.18729},
  year={2024}
}

@article{muresan2018fruit,
  title={Fruit recognition from images using deep learning},
  author={Muresan, Horea and Oltean, Mihai},
  journal={Acta Universitatis Sapientiae, Informatica},
  year={2018}
}

@misc{cchang_2018,
	title={Garbage Classification},
	howpublished={\url{https://www.kaggle.com/ds/81794}},
	DOI={10.34740/KAGGLE/DS/81794},
	publisher={Kaggle},
	author={CCHANG},
	year={2018},
    note = {Accessed: 2026-06-30}
}

@misc{Landscape,
  author = {DeepNets},
  title = {Landscape Recognition},
  publisher = {Kaggle},
  howpublished = {\url{https://www.kaggle.com/datasets/utkarshsaxenadn/landscape-recognition-image-dataset-12k-images}},
  note = {Accessed: 2026-06-30}
}

@misc{beansdata,
    author="Makerere AI Lab",
    title="Bean disease dataset",
    year="2020",
    url="https://github.com/AI-Lab-Makerere/ibean/",
    note = {Accessed: 2026-06-30}
}

@article{ahmed2023mangoleafbd,
  title={MangoLeafBD: A comprehensive image dataset to classify diseased and healthy mango leaves},
  author={Ahmed, Sarder Iftekhar and Ibrahim, Muhammad and Nadim, Md and Rahman, Md Mizanur and Shejunti, Maria Mehjabin and Jabid, Taskeed and Ali, Md Sawkat},
  journal={Data in Brief},
  year={2023},
}

@article{cukierskidogs,
  title={Dogs vs. cats, 2013},
  author={Cukierski, Will},
  journal={URL https://kaggle.com/competitions/dogs-vs-cats},
  note = {Accessed: 2026-06-30}
}

@article{xiao2021classification,
  title={Classification of weather phenomenon from images by using deep convolutional neural network},
  author={Xiao, Haixia and Zhang, Feng and Shen, Zhongping and Wu, Kun and Zhang, Jinglin},
  journal={Earth and Space Science},
  year={2021},
}

@article{bansal2019intel,
  title={Intel image classification},
  author={Bansal, Puneet},
  journal={Available on https://www. kaggle. com/puneet6060/intel-image-classification, Online},
  year={2019},
  note = {Accessed: 2026-06-30}
}

@inproceedings{pogorelov2017kvasir,
  title={Kvasir: A multi-class image dataset for computer aided gastrointestinal disease detection},
  author={Pogorelov, Konstantin and Randel, Kristin Ranheim and Griwodz, Carsten and Eskeland, Sigrun Losada and de Lange, Thomas and Johansen, Dag and Spampinato, Concetto and Dang-Nguyen, Duc-Tien and Lux, Mathias and Schmidt, Peter Thelin and others},
  booktitle={ACM MMSys},
  year={2017}
}

@inproceedings{ahmed2021dcnn,
  title={Dcnn-based vegetable image classification using transfer learning: A comparative study},
  author={Ahmed, M Israk and Mamun, Shahriyar Mahmud and Asif, Asif Uz Zaman},
  booktitle={ICCCSP},
  year={2021},
}

@inproceedings{bateni2020improved,
  title={Improved few-shot visual classification},
  author={Bateni, Peyman and Goyal, Raghav and Masrani, Vaden and Wood, Frank and Sigal, Leonid},
  booktitle={CVPR},
  year={2020}
}

@inproceedings{bateni2022enhancing,
  title={Enhancing few-shot image classification with unlabelled examples},
  author={Bateni, Peyman and Barber, Jarred and Van de Meent, Jan-Willem and Wood, Frank},
  booktitle={WACV},
  year={2022}
}

@inproceedings{MAML,
  title={Model-agnostic meta-learning for fast adaptation of deep networks},
  author={Finn, Chelsea and Abbeel, Pieter and Levine, Sergey},
  booktitle={ICLR},
  year={2017},
}

@inproceedings{
ha2017hypernetworks,
title={HyperNetworks},
author={David Ha and Andrew M. Dai and Quoc V. Le},
booktitle={ICLR},
year={2017},
}

@inproceedings{hu2022lora,
  title={LoRA: Low-Rank Adaptation of Large Language Models},
  author={Hu, Edward J. and Shen, Yelong and Wallis, Phillip and Allen-Zhu, Zeyuan and Li, Yuanzhi and Wang, Shean and Wang, Lu and Chen, Weizhu},
  booktitle={ICLR},
  year={2022}
}

@inproceedings{requeima2019cnaps,
  title={Fast and Flexible Multi-Task Classification using Conditional Neural Adaptive Processes},
  author={Requeima, James and Gordon, Jonathan and Bronskill, John and Nowozin, Sebastian and Turner, Richard E.},
  booktitle={NeurIPS},
  year={2019}
}

@inproceedings{li2022crossdomain,
  title={Cross-domain Few-shot Learning with Task-specific Adapters},
  author={Li, Wei-Hong and Liu, Xialei and Bilen, Hakan},
  booktitle={CVPR},
  year={2022}
}

@inproceedings{triantafillou2021universal,
  title={Learning a Universal Template for Few-shot Dataset Generalization},
  author={Triantafillou, Eleni and Larochelle, Hugo and Zemel, Richard and Dumoulin, Vincent},
  booktitle={ICML},
  year={2021}
}

@inproceedings{chi2012nsc,
  title={Connecting the Dots in Multi-Class Classification: From Nearest Subspace to Collaborative Representation},
  author={Chi, Yuejie and Porikli, Fatih},
  booktitle={CVPR},
  year={2012},
  pages={3602--3609}
}

@article{ye2025dynamic,
  title={Dynamic model merging with mixture of weights},
  author={Ye, Peng and Huang, Chenyu and Shen, Mingzhu and Chen, Tao and Huang, Yongqi and Ouyang, Wanli},
  journal={IEEE Transactions on Circuits and Systems for Video Technology},
  year={2025},
}

@inproceedings{khot2020qasc,
title={Qasc: A dataset for question
answering via sentence composition},
author={Tushar Khot and Peter Clark and Michal Guerquin and Peter Jansen and Ashish Sabharwal},
booktitle= {AAAI},
year={2020}
}

@inproceedings{yang2015wikiqa,
title={WikiQA: A challenge dataset for open-domain question
answering},
author={Yi Yang and Wen-tau Yih and Christopher Meek},
booktitle={ACL},
year={2015}
}

@inproceedings{tafjord2019quartz,
title={Quartz: An open-domain dataset of qualitative
relationship questions},
author={Oyvind Tafjord and Matt Gardner and Kevin Lin and Peter Clark},
booktitle={EMNLP},
year={2019}
}

@inproceedings{zhang2019paws,
title={PAWS: Paraphrase Adversaries from Word Scrambling},
author={Yuan Zhang and Jason Baldridge and Luheng He},
booktitle={NAACL},
year={2019}
}

@inproceedings{sharma2018tack,
title={Tackling the story ending biases
in the story cloze test},
author={Rishi Sharma and James Allen and Omid Bakhshandeh and Nasrin Mostafazadeh},
booktitle={ACL},
year={2018}
}

@inproceedings{sakaguchi2020wino,
title={Winogrande: An adversarial winograd
schema challenge at scale},
author={Keisuke Sakaguchi and Ronan Le Bras and Chandra Bhagavatula and Yejin Choi},
booktitle={AAAI},
year={2020}
}

@inproceedings{levesque2012wsc,
title={The winograd schema challenge},
author={Hector Levesque and Ernest Davis and Leora Morgenstern},
booktitle={KR},
year={2012}
}

@inproceedings{gargiulo2024task,
    author={Gargiulo, Antonio Andrea and Crisostomi, Donato and Bucarelli, Maria Sofia and Scardapane, Simone and Silvestri, Fabrizio and Rodolà, Emanuele},
    title={Task Singular Vectors: Reducing Task Interference in Model Merging},
    booktitle = {CVPR},
    year = {2025}
}

@article{mohanty2016plantdisease,
  author= {Sharada P. Mohanty and David P. Hughes and Marcel Salath{\'e}},
  title = {Using Deep Learning for Image-Based Plant Disease Detection},
  journal= {Frontiers in Plant Science},
  year = {2016}
}

@article{kermany2018identifying,
  author    = {Daniel S. Kermany and Michael Goldbaum and Wenjia Cai and Carolina C. S. Valentim and Huiying Liang and Sally L. Baxter and Alex McKeown and Ge Yang and Xiaokang Wu and Fangbing Yan and Justin Dong and Madeleine Prasadha and Jacqueline Pei and Maggie Ting and Jie Zhu and Christopher Li and Sierra Hewett and Jason Dong and Ian Ziyar and Alexander Shi and Runze Zhang and Kashyap Gupta and Roger M. Y. Wong and Louis A. T. C. Lam and Josephine Cheung and Markus Tsoi and Valerie Wu and Christina Yan and Congjiao Huang and Dennis H. Lee and Yanning Zhang and Jung W. Wong and K. Maggie Li and Dan L. Zhang and Xun Xu and Yefeng Zheng and Lily Zhang},
  title     = {Identifying Medical Diagnoses and Treatable Diseases by Image-Based Deep Learning},
  journal   = {Cell},
  year      = {2018}
}

@misc{grassknoted_asl_alphabet,
  author= {Akash Nagaraj},
  title= {ASL Alphabet},
  howpublished = {\url{https://www.kaggle.com/datasets/grassknoted/asl-alphabet}},
  note = {Accessed: 2026-06-30}
}

@article{xia2017aid,
  author    = {Gui-Song Xia and Jingwen Hu and Fan Hu and Baoyuan Shi and Xiang Bai and Yongchao Zhong and Liangpei Zhang and Xiaoqiang Lu},
  title     = {AID: A Benchmark Data Set for Performance Evaluation of Aerial Scene Classification},
  journal   = {IEEE Transactions on Geoscience and Remote Sensing},
  year      = {2017},
}

@techreport{wah2011cub,
  author      = {Catherine Wah and Steve Branson and Peter Welinder and Pietro Perona and Serge Belongie},
  title       = {The Caltech-UCSD Birds-200-2011 Dataset},
  institution = {California Institute of Technology},
  year        = {2011},
}

@techreport{griffin2007caltech256,
  author      = {Gregory Griffin and Alex Holub and Pietro Perona},
  title       = {Caltech-256 Object Category Dataset},
  institution = {California Institute of Technology},
  year        = {2007},
}

@misc{corrado_animals10,
  author       = {Corrado Alessio},
  title        = {Animals-10},
  howpublished = {\url{https://www.kaggle.com/datasets/alessiocorrado99/animals10}},
  note = {Accessed: 2026-06-30}
}

@article{cuadros2009eyepacs,
  author  = {Jorge Cuadros and George Bresnick},
  title   = {EyePACS: An Adaptable Telemedicine System for Diabetic Retinopathy Screening},
  journal = {Journal of Diabetes Science and Technology},
  year    = {2009},
}

@inproceedings{khosla2011fgvc,
  author    = {Aditya Khosla and Nityananda Jayadevaprakash and Bangpeng Yao and Li Fei-Fei},
  title     = {Novel Dataset for Fine-Grained Image Categorization},
  booktitle = {CVPR},
  year      = {2011}
}

@inproceedings{quattoni2009indoor,
  author    = {Ariadna Quattoni and Antonio Torralba},
  title     = {Recognizing Indoor Scenes},
  booktitle = {CVPR},
  year      = {2009},
}

@misc{howard2019imagenette,
  author       = {Jeremy Howard},
  title        = {Imagenette: A Smaller Subset of 10 Easily Classified Classes from ImageNet},
  year         = {2019},
  howpublished = {\url{https://github.com/fastai/imagenette}},
  note = {Accessed: 2026-06-30}
}

@article{fei2007caltech101,
  author  = {Li Fei-Fei and Rob Fergus and Pietro Perona},
  title   = {Learning Generative Visual Models from Few Training Examples: An Incremental Bayesian Approach Tested on 101 Object Categories},
  journal = {Computer Vision and Image Understanding},
  year    = {2007},
}

@inproceedings{li2017reliable,
  author    = {Shan Li and Weihong Deng and JunPing Du},
  title     = {Reliable Crowdsourcing and Deep Locality-Preserving Learning for Expression Recognition in the Wild},
  booktitle = {CVPR},
  year      = {2017},
}

@article{maji2013aircraft,
  author  = {Subhransu Maji and Esa Rahtu and Juho Kannala and Matthew Blaschko and Andrea Vedaldi},
  title   = {Fine-Grained Visual Classification of Aircraft},
  journal = {arXiv preprint arXiv:1306.5151},
  year    = {2013}
}

@article{olsen2019deepweeds,
  author  = {Alex Olsen and Dmitry A. Konovalov and Bronson Philippa and Peter Ridd and Jake C. Wood and Jamie Johns and Wesley Banks and Benjamin Girgenti and Owen Kenny and James Whinney and Brendan Calvert and Mostafa {Rahimi Azghadi} and Ronald D. White},
  title   = {DeepWeeds: A Multiclass Weed Species Image Dataset for Deep Learning},
  journal = {Scientific Reports},
  year    = {2019}
}

@inproceedings{vanhorn2015nabirds,
  author    = {Grant Van Horn and Steve Branson and Ryan Farrell and Scott Haber and Jessie Barry and Panos Ipeirotis and Pietro Perona and Serge Belongie},
  title     = {Building a Bird Recognition App and Large Scale Dataset With Citizen Scientists: The Fine Print in Fine-Grained Dataset Collection},
  booktitle = {CVPR},
  year      = {2015},
}

@inproceedings{kumar2012leafsnap,
  author    = {Neeraj Kumar and Peter N. Belhumeur and Arijit Biswas and David W. Jacobs and W. John Kress and Ida C. Lopez and Jo{\~a}o V. B. Soares},
  title     = {Leafsnap: A Computer Vision System for Automatic Plant Species Identification},
  booktitle = {ECCV},
  year      = {2012}
}

@inproceedings{choi2020starganv2,
  author    = {Yunjey Choi and Youngjung Uh and Jaejun Yoo and Jung-Woo Ha},
  title     = {StarGAN v2: Diverse Image Synthesis for Multiple Domains},
  booktitle = {CVPR},
  year      = {2020},
}

@article{liao2022artbench,
  author  = {Peiyuan Liao and Xiuyu Li and Xihui Liu and Kurt Keutzer},
  title   = {The ArtBench Dataset: Benchmarking Generative Models with Artworks},
  journal = {arXiv preprint arXiv:2206.11404},
  year    = {2022}
}

@inproceedings{lee2025mitigating,
  author    = {Yeoreum Lee and Jinwook Jung and Sungyong Baik},
  title     = {Mitigating Parameter Interference in Model Merging via Sharpness-Aware Fine-Tuning},
  booktitle = {ICLR},
  year      = {2025},
}

@inproceedings{iurada2025efficient,
  author    = {Leonardo Iurada and Marco Ciccone and Tatiana Tommasi},
  title     = {Efficient Model Editing with Task-Localized Sparse Fine-tuning},
  booktitle = {ICLR},
  year      = {2025},
}

@inproceedings{baik2020learning,
  title={Learning to Forget for Meta-Learning},
  author={Baik, Sungyong and Hong, Seokil and Lee, Kyoung Mu},
  booktitle={CVPR},
  year={2020}
}

@article{baik2024learning,
  title={Learning to Learn Task-Adaptive Hyperparameters for Few-Shot Learning},
  author={Baik, Sungyong and Choi, Myungsub and Choi, Janghoon and Kim, Heewon and Lee, Kyoung Mu},
  journal={IEEE Transactions on Pattern Analysis and Machine Intelligence},
  year={2024}
}

@inproceedings{baik2021metal,
  title={Meta-Learning with Task-Adaptive Loss Function for Few-Shot Learning},
  author={Baik, Sungyong and Choi, Janghoon and Kim, Heewon and Cho, Dohee and Min, Jaesik and Lee, Kyoung Mu},
  booktitle={ICCV},
  year={2021}
}

@inproceedings{choi2020scene,
  title={Scene-Adaptive Video Frame Interpolation via Meta-Learning},
  author={Choi, Myungsub and Choi, Janghoon and Baik, Sungyong and Kim, Tae Hyun and Lee, Kyoung Mu},
  booktitle={CVPR},
  year={2020}
}

@article{choi2022testtime,
  title={Test-Time Adaptation for Video Frame Interpolation via Meta-Learning},
  author={Choi, Myungsub and Choi, Janghoon and Baik, Sungyong and Kim, Tae Hyun and Lee, Kyoung Mu},
  journal={IEEE Transactions on Pattern Analysis and Machine Intelligence},
  year={2022}
}

@article{choi2022visual,
  title={Visual Tracking by Adaptive Continual Meta-Learning},
  author={Choi, Janghoon and Baik, Sungyong and Choi, Myungsub and Kwon, Junseok and Lee, Kyoung Mu},
  journal={IEEE Access},
  year={2022}
}

@inproceedings{baik2020alfa,
  title={Meta-Learning with Adaptive Hyperparameters},
  author={Baik, Sungyong and Choi, Myungsub and Choi, Janghoon and Kim, Heewon and Lee, Kyoung Mu},
  booktitle={NeurIPS},
  year={2020}
}

@article{baik2022learning,
  title={Learning to Forget for Meta-Learning via Task-and-Layer-Wise Attenuation},
  author={Baik, Sungyong and Hong, Seokil and Lee, Kyoung Mu},
  journal={IEEE Transactions on Pattern Analysis and Machine Intelligence},
  year={2022}
}

@article{sim,
    author = {Jungyong Son and Jinwook Jung and Sungyong Baik},
    title = {Training-free Task Classification for Multi-Task Model Merging},
    journal = {arXiv preprint arXiv:2606.22589},
    year={2026}
}

\clearpage

% ===============================================================
% Supplementary Material / Appendix
% ===============================================================

\setcounter{section}{0}
\setcounter{figure}{0}
\setcounter{table}{0}
\setcounter{equation}{0}

\renewcommand{\thesection}{\Alph{section}}
\renewcommand{\thefigure}{\Alph{figure}}
\renewcommand{\thetable}{\Alph{table}}
\renewcommand{\theequation}{\Alph{equation}}

% Unique hyperref identifiers for supplementary material
\renewcommand{\theHsection}{appendix.\Alph{section}}
\renewcommand{\theHfigure}{appendix.\Alph{figure}}
\renewcommand{\theHtable}{appendix.\Alph{table}}
\renewcommand{\theHequation}{appendix.\Alph{equation}}

\appendix

\section*{Appendix} 

\section*{Table of contents}

We provide the following items in this Appendix:
\begin{itemize}
    % \item Clarifications on main paper
    % \begin{itemize}
    %     \item Recovered-model notation in Main Fig. \textcolor{red}{1}
    %     \item Group inference for component-selection baselines in Tab \textcolor{red}{5}
    % \end{itemize}

    \item (\cref{app:offset_justification}) Justification for offset-based recovery
    \begin{itemize}
        \item (\cref{app:affine_transformation_based_recovery}) Affine transformation-based expert recovery
        \item (\cref{app:learning_and_simplyfing}) Learning and simplifying the affine rule in \method{}
        \item (\cref{app:theoretical_justification}) Theoretical justification via low-rank offsets
    \end{itemize}

    \item (\cref{app:more_related_work}) More related work

    \item (\cref{app:experiment_details}) Experiment details
    \begin{itemize}
        \item (\cref{app:pseudocode}) Pseudocode of \method{}
        \item (\cref{app:module_architecture}) Module architecture and parameters
        \item (\cref{app:baseline_details}) Baseline details
        \item (\cref{app:computational_resources_training}) Computational resources and training time
    \end{itemize}

    \item (\cref{app:additional_experiments}) Additional experiments
    \begin{itemize}
        \item (\cref{app:dirichlet_training}) Dirichlet-sampling training for unseen-task generalization
        \item (\cref{app:more_ood_experiments}) More unseen-task generalization results
        \item (\cref{app:larger_task_suites}) Scalability to larger task suites
        \item (\cref{app:fine_grained_similarity}) Fine-grained tasks with high inter-task similarity
        \item (\cref{app:robustness_under_corrupted_input}) Robustness under corrupted inputs
    \end{itemize}

    \item (\cref{app:more_ablation_study}) More ablation study
    \begin{itemize}
        \item (\cref{app:representation_source_and_layer}) Representation source and feature layer
        \item (\cref{app:bank_size}) Number of reference inputs and subspace dimension
        \item (\cref{app:rank_ablation}) Rank
        \item (\cref{app:random_seed}) Random seed
        \item (\cref{app:cosine_objective}) Cosine similarity as an objective
        \item (\cref{app:factorization_order}) Impact of factorization order in offset generation
        \item (\cref{app:shared_layers}) Application with shared layers
    \end{itemize}

    \item (\cref{app:additional_analyses}) Additional analyses
    \begin{itemize}
        \item (\cref{app:soft_recovery_interpolation}) Soft-recovery interpolation mechanism
        \item (\cref{app:small_reference_sets}) Task identification under small reference sets
        \item (\cref{app:wsc_flip_analysis}) Prediction-flip analysis for WSC
    \end{itemize}

    \item (\cref{app:limitations}) Limitations
    
\end{itemize}

\FloatBarrier

% \input{supple/1_Clarifications}
% \FloatBarrier

\section{Justification for offset-based recovery}
\label{app:offset_justification}

This appendix provides additional justification for the expert recovery rule used in the main paper.
% We first derive an affine recovery rule directly in parameter space, consistent with \textcolor{linkred}{Eq.~(4)} and \textcolor{linkred}{Eq.~(5)} in the main paper.
We first derive an affine recovery rule directly in parameter space, consistent with \cref{eq:theta_affine_form} and \cref{eq:theta_offset_form} in the main paper.
We then explain how \method{} learns this affine rule and why it simplifies to an offset-only form in practice.
Finally, we provide a theoretical justification for using low-rank offsets based on spectral concentration.

\subsection{Affine transformation--based expert recovery}
\label{app:affine_transformation_based_recovery}

Starting from the task-arithmetic formulation in the main paper,
\begin{equation}
\label{eq:task_arithmetic_app}
\vtheta_{\text{merge}}
\;=\;
\vtheta_0 \;+\; \sum_{i=1}^{T} \lambda_i\,\bm{\tau}_i,
\qquad
\bm{\tau}_i := \vtheta_i - \vtheta_0,
\end{equation}
we substitute $\bm{\tau}_i=\vtheta_i-\vtheta_0$ into the expression for $\vtheta_{\text{merge}}$:
\begin{align}
\vtheta_{\text{merge}}
&=
\vtheta_0 + \sum_{i=1}^{T}\lambda_i(\vtheta_i-\vtheta_0) \\
&=
\Bigl(1-\sum_{i=1}^{T}\lambda_i\Bigr)\vtheta_0
\;+\;
\sum_{i=1}^{T}\lambda_i\vtheta_i.
\label{eq:theta_merge_expand}
\end{align}
Separating the target task $t$ gives
\begin{equation}
\label{eq:theta_merge_target}
\vtheta_{\text{merge}}
=
\lambda_t \vtheta_t
\;+\;
\Bigl(1-\sum_{i=1}^{T}\lambda_i\Bigr)\vtheta_0
\;+\;
\sum_{i\neq t}\lambda_i \vtheta_i.
\end{equation}
Assuming $\lambda_t \neq 0$, we isolate $\vtheta_t$:
\begin{equation}
\label{eq:theta_exact_app}
\vtheta_t
=
\lambda_t^{-1}\vtheta_{\text{merge}}
-
\lambda_t^{-1}
\left[
\Bigl(1-\sum_{i=1}^{T}\lambda_i\Bigr)\vtheta_0
+
\sum_{i\neq t}\lambda_i \vtheta_i
\right].
\end{equation}

\cref{eq:theta_exact_app} shows that recovering the target expert from the merged model consists of a multiplicative term applied to $\vtheta_{\text{merge}}$ and an additive correction term that compensates for contributions from the pre-trained model and other experts.
\cref{eq:theta_exact_app} can therefore be rewritten exactly in the affine form
\begin{equation}
\label{eq:affine_theta_rule_app}
\vtheta_t
=
\gamma_t \vtheta_{\text{merge}}
\;+\;
\bm{\beta}_t,
\end{equation}
with
\begin{equation}
\label{eq:gamma_beta_exact_app}
\gamma_t
=
\lambda_t^{-1},
\qquad
\bm{\beta}_t
=
-
\lambda_t^{-1}
\left[
\Bigl(1-\sum_{i=1}^{T}\lambda_i\Bigr)\vtheta_0
+
\sum_{i\neq t}\lambda_i \vtheta_i
\right].
\end{equation}
Therefore, if $\gamma_t$ and $\bm{\beta}_t$ are generated from task identity $t$, then the target expert parameters $\vtheta_t$ can be recovered directly from $\vtheta_{\text{merge}}$ via an affine transformation.
This matches the affine formulation in \cref{eq:theta_affine_form} of the main paper.

\subsection{Learning and simplifying the affine rule in \method{}}
\label{app:learning_and_simplyfing}

In this appendix, when the task index $t$ is given, we denote the recovered parameters by $\hat{\vtheta}_t$.
This notation is distinct from the main-paper notation $\vtheta_{\hat{t}}$, which denotes inference-time parameters conditioned on a predicted task ID $\hat{t}$.

Guided by \cref{eq:affine_theta_rule_app}, a preliminary affine variant of \method{} can learn, at each layer $l$, a scalar $\gamma_t^{(l)}$ together with a task-dependent offset $\bm{\beta}_t^{(l)}$:
\begin{equation}
\label{eq:full_retex_recovery_appendix_b}
\hat{\vtheta}_{t}^{(l)}
\;=\;
\gamma_{t}^{(l)}\vtheta_\text{merge}^{(l)}
\;+\;
\bm{\beta}_{t}^{(l)}.
\end{equation}
In the actual implementation, $\bm{\beta}_{t}^{(l)}$ is parameterized in low-rank form, but we suppress that factorization here in order to focus on the affine simplification itself.

Empirically, the task-averaged $\gamma_t^{(l)}$ for many layer types converges close to $1$, indicating that most task-specific adjustment is captured by the offset term rather than the scaling term.
\cref{fig:gamma_convergence_appendix_b} visualizes this convergence trend over training.

\begin{figure}[t]
    \centering
    \includegraphics[width=0.98\linewidth]{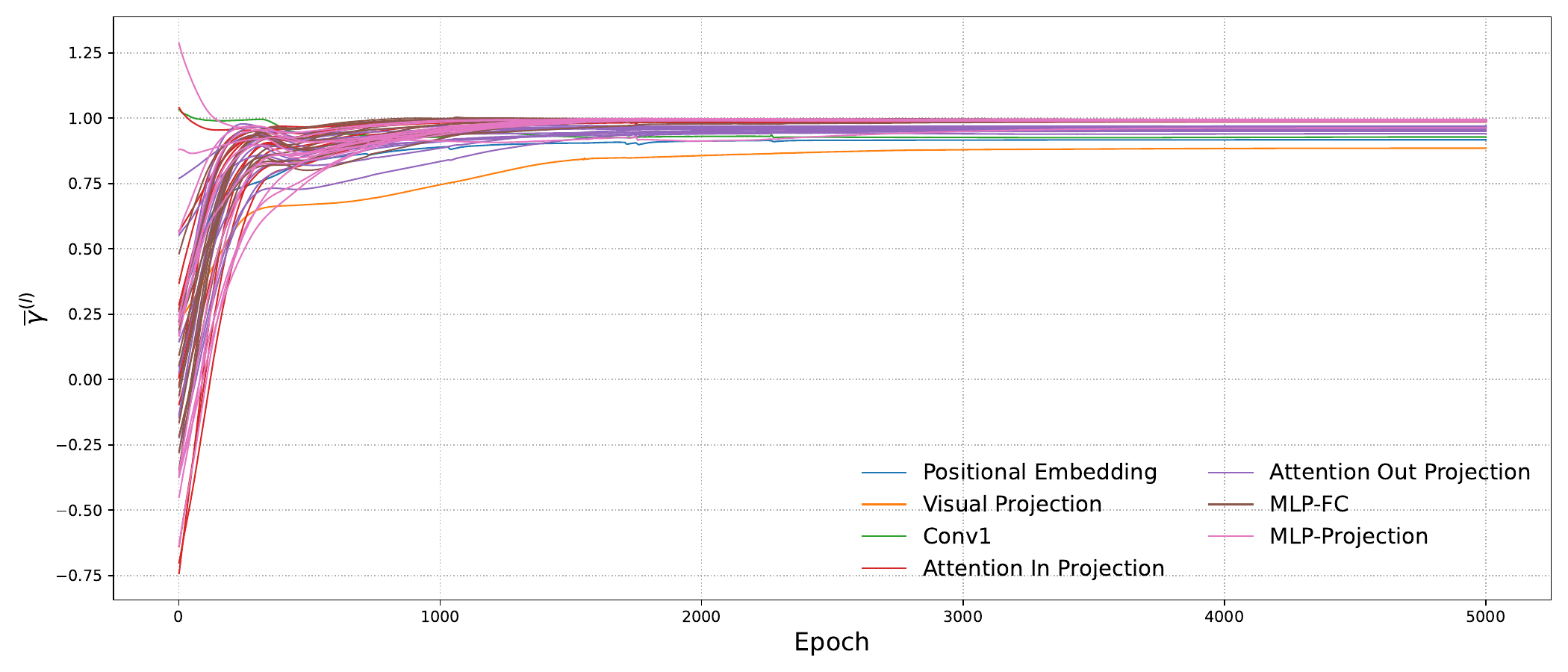}
    \caption{Convergence behavior of task-averaged $\gamma_t^{(l)}$ values for various 2D layer types during \method{} training, where $\gamma_t^{(l)}$ is learnable. Experiments on ViT-B/32 with 8 vision tasks show that many layers converge to $\gamma_t^{(l)}\!\approx\!1$.}
    \label{fig:gamma_convergence_appendix_b}
\end{figure}

\noindent\textbf{Offset-only rule via fixed scaling.}
Motivated by the above observation, we fix $\gamma_t^{(l)}\!=\!1$ for all tasks and layers so that \method{} generates only the low-rank offset:
\begin{align}
\hat{\vtheta}_{t}^{(l)}
\;=\;
\vtheta_\text{merge}^{(l)} \;+\; \hat{\bm{\beta}}_{t}^{(l)}.
\label{eq:offset_only_rule_app}
\end{align}
This corresponds to the offset-only recovery form used in \cref{eq:theta_affine_form} of the main paper, namely
\(
\hat{\vtheta}_{t}\approx\vtheta_{\text{merge}}+\hat{\bm{\beta}}_t
\).
\cref{fig:results_fixed_gamma_b} shows that fixing $\gamma_t^{(l)}\!=\!1$ closely matches the performance of the learned-$\gamma$ model across a wide range of ranks $r$.
These results indicate that the multiplicative term is not essential in practice and support the use of a simpler offset-only recovery rule.

\begin{figure}[t]
    \centering
    \includegraphics[width=0.62\linewidth]{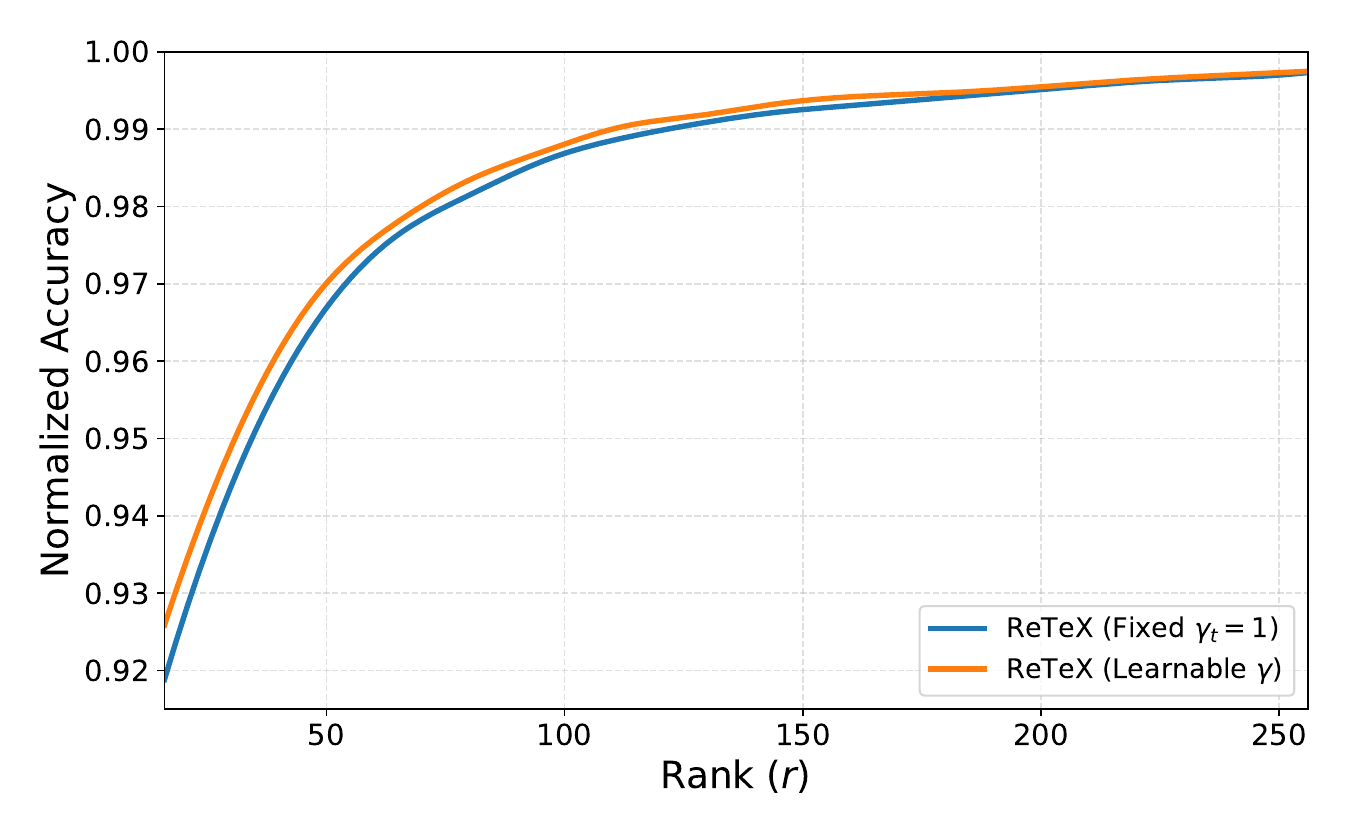}
    \caption{Normalized accuracy versus rank $r$ on 8 vision tasks (ViT-B/32). Fixing $\gamma_t^{(l)}\!=\!1$ closely matches the performance of the learned-$\gamma$ model across a wide range of ranks, supporting the offset-only recovery view.}
    \label{fig:results_fixed_gamma_b}
\end{figure}

\subsection{Theoretical justification via low-rank offsets}
\label{app:theoretical_justification}

We now justify why a low-rank offset is often sufficient in practice.
The key idea is that if the expert--merged difference is spectrally concentrated, then a low-rank approximation can recover most of its energy with small error.
To formalize this, we analyze approximation error for the expert--merged difference at a single layer.

Consider one 2D weight matrix at layer $l$ and define the expert--merged difference, which we also refer to here as the adjusted task vector, by
\[
\bm{\tau}'^{(l)}_t
\;:=\;
\vtheta_t^{(l)} - \vtheta_{\mathrm{merge}}^{(l)}
\in \mathbb{R}^{m\times m},
\]
where $m$ denotes the matrix dimension.
For notational simplicity, we use a square shape.
Let $\bm{\beta}^{(l)}_t$ be a rank-$r$ matrix intended to approximate $\bm{\tau}'^{(l)}_t$.
We measure the relative prediction error by
\begin{equation}
\label{eq:relative_error}
\varepsilon^{(l)}
\;:=\;
\frac{
  \bigl\|\bm{\tau}'^{(l)}_t - \bm{\beta}_t^{(l)}\bigr\|_{F}
}{
  \bigl\|\bm{\tau}'^{(l)}_t\bigr\|_{F}
}.
\end{equation}

\begin{theorem}[Worst-case error bound for truncated SVD]
\label{thm:prediction_error_bound}
Let $\bm{\tau}'^{(l)}_t \in \mathbb{R}^{m \times m}$ have singular values
$\sigma_1 \ge \sigma_2 \ge \dots \ge \sigma_m \ge 0$.
Let $\bigl(\bm{\tau}'^{(l)}_t\bigr)_r$ be the rank-$r$ matrix obtained by truncating the SVD of $\bm{\tau}'^{(l)}_t$ to its top $r$ singular values.
If we set $\bm{\beta}^{(l)}_t := \bigl(\bm{\tau}'^{(l)}_t\bigr)_r$, then the relative error $\varepsilon^{(l)}$ satisfies
\begin{equation}
\label{eq:prediction_error_bound}
\varepsilon^{(l)}
\;\le\;
\sqrt{\frac{m - r}{m}} \,.
\end{equation}
\end{theorem}

\begin{proof}
Let $\bm{\tau}'^{(l)}_t = \mathbf{U}\mathbf{\Sigma}\mathbf{V}^\top$ be the SVD with
$\mathbf{\Sigma}=\mathrm{diag}(\sigma_1,\dots,\sigma_m)$.
The residual $\bm{\tau}'^{(l)}_t-(\bm{\tau}'^{(l)}_t)_r$ has singular values
$\sigma_{r+1},\dots,\sigma_m$, so
\[
\bigl\|\bm{\tau}'^{(l)}_t - (\bm{\tau}'^{(l)}_t)_r\bigr\|_{F}^{2}
=
\sum_{i=r+1}^{m} \sigma_i^{2},
\qquad
\bigl\|\bm{\tau}'^{(l)}_t\bigr\|_{F}^{2}
=
\sum_{i=1}^{m} \sigma_i^{2}.
\]
Since $\sigma_1^2 \ge \dots \ge \sigma_m^2$, the average of the top-$r$ squared singular values is at least the global average:
\[
\sum_{i=1}^{r} \sigma_i^{2}
\;\ge\;
\frac{r}{m}\sum_{i=1}^{m}\sigma_i^{2}.
\]
Thus
\[
\sum_{i=r+1}^{m}\sigma_i^{2}
=
\sum_{i=1}^{m}\sigma_i^{2} - \sum_{i=1}^{r}\sigma_i^{2}
\;\le\;
\frac{m-r}{m}\sum_{i=1}^{m}\sigma_i^{2}.
\]
Dividing by $\|\bm{\tau}'^{(l)}_t\|_F^2$ and taking square roots yields \cref{eq:prediction_error_bound}.
\end{proof}

\cref{thm:prediction_error_bound} gives a worst-case guarantee.
In practice, expert--merged differences often exhibit spectral concentration, where a small number of leading singular directions capture most Frobenius energy.
The next result quantifies how such concentration tightens the approximation guarantee.

\begin{theorem}[Effect of spectral concentration]
\label{thm:spectral_concentration}
Let $\bm{\tau}'^{(l)}_t \in \mathbb{R}^{m \times m}$ have singular values
$\sigma_1 \ge \dots \ge \sigma_m \ge 0$.
For $q \in \{1,\dots,m\}$, define the top-$q$ energy ratio
\begin{equation}
\label{eq:rho_q_def}
\rho_q
\;:=\;
\frac{
  \sum_{i=1}^{q} \sigma_i^{2}
}{
  \sum_{i=1}^{m} \sigma_i^{2}
}
\;\in\; (0,1].
\end{equation}
If $r \ge q$ and we again take $\bm{\beta}^{(l)}_t := (\bm{\tau}'^{(l)}_t)_r$, then
\begin{equation}
\label{eq:spectral_concentration_bound}
\varepsilon^{(l)}
\;\le\;
\sqrt{1 - \rho_q} \,.
\end{equation}
\end{theorem}

\begin{proof}
For the truncated SVD approximation,
\[
\varepsilon^{(l)2}
=
\frac{
  \bigl\|\bm{\tau}'^{(l)}_t - (\bm{\tau}'^{(l)}_t)_r\bigr\|_{F}^{2}
}{
  \bigl\|\bm{\tau}'^{(l)}_t\bigr\|_{F}^{2}
}
=
1-\rho_r,
\qquad
\rho_r
=
\frac{
  \sum_{i=1}^{r}\sigma_i^{2}
}{
  \sum_{i=1}^{m}\sigma_i^{2}
}.
\]
If $r\ge q$, then $\rho_r \ge \rho_q$, hence
$\varepsilon^{(l)2}=1-\rho_r \le 1-\rho_q$, which gives \cref{eq:spectral_concentration_bound}.
\end{proof}

To verify that real expert--merged differences exhibit strong spectral concentration, we compute the top-$q$ energy ratio $\rho_q$ in \cref{eq:rho_q_def} for each square 2D layer of $\bm{\tau}'^{(l)}_t$ and then average over layers for each task.
\cref{fig:top_rho_adjust} reports the resulting curves on our vision tasks.
The strong concentration, namely large $\rho_q$ for modest $q$, supports the effectiveness of low-rank offset prediction in \method{}.

\begin{figure}[t]
  \centering
  \includegraphics[width=0.7\linewidth]{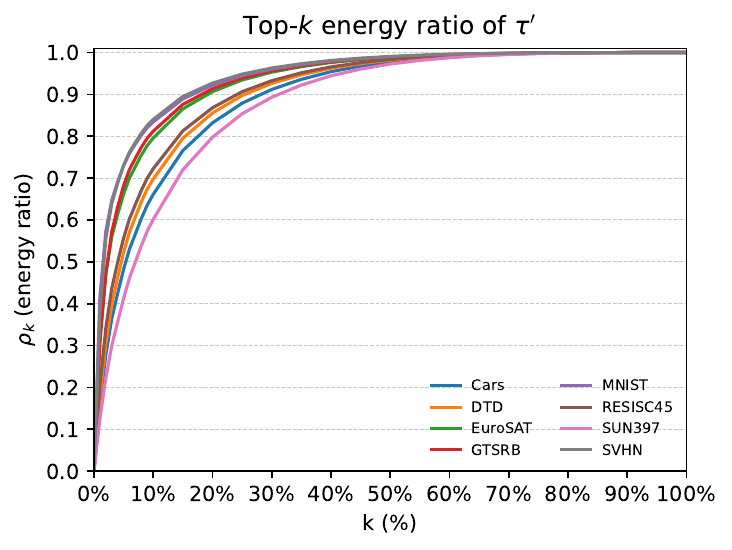}
  \caption{\textbf{Top-$q$ energy ratio of expert--merged differences.}
  For each task $t$ and each square 2D layer $l$,
  we compute the top-$q$ energy ratio $\rho_q$ of the adjusted task vector
  $\bm{\tau}'^{(l)}_t = \vtheta_t^{(l)} - \vtheta_{\mathrm{merge}}^{(l)}$
  and average over layers.
  The curves quickly saturate toward $1.0$, indicating strong spectral concentration and supporting the effectiveness of low-rank offset prediction in \method{}.}
  \label{fig:top_rho_adjust}
\end{figure}
\FloatBarrier

\section{More related work}
\label{app:more_related_work}
HyperNetworks generate the weights of one network using another network~\cite{ha2017hypernetworks}, and FiLM performs feature-wise affine modulation conditioned on side information~\cite{perez2018film}. 
These ideas underpin a large body of few-shot and conditional adaptation methods that infer task-conditioned parameters from support examples. 
A representative example is CNAPs~\cite{requeima2019cnaps}, which encodes a support set and predicts task-conditioned parameters for rapid adaptation. 
Follow-up work strengthens this line by refining conditional adaptation and exploiting richer task statistics, including Mahalanobis-distance based refinement~\cite{bateni2020improved,bateni2022enhancing}, task-specific adapters for cross-domain transfer~\cite{li2022crossdomain}, and universal templates specialized to new datasets with limited supervision~\cite{triantafillou2021universal}.
% Related meta-learning methods such as MAML~\cite{MAML} also target fast adaptation from small support sets, although they rely on gradient-based adaptation rather than direct parameter generation.

Related optimization-based meta-learning methods also adapt the task-learning process itself, for example by learning task- and layer-wise attenuation~\cite{baik2020learning,baik2022learning}, adaptive hyperparameters~\cite{baik2020alfa,baik2024learning}, or task-adaptive loss functions~\cite{baik2021metal}. 
Meta-learning has also been applied to test-time or online adaptation in vision tasks, such as scene-adaptive video frame interpolation~\cite{choi2020scene,choi2022testtime} and adaptive continual visual tracking~\cite{choi2022visual}.

% \textcolor{blue}{
% Another related direction studies training-free task identification for model merging.
% For instance, SiM constructs SVD-based task manifolds from a small support set and routes each test input by measuring projection residuals, avoiding additional router training~\cite{sim}.
% }

\method{} shares the high-level idea of task-conditioned generation, but the setting and objective are different.
Prior methods adapt a shared backbone to a new task using support examples, whereas \method{} starts from an already merged checkpoint and recovers the behavior of an existing task expert.
In addition, prior hypernetwork-based methods usually generate affine feature modulation or task-conditioned classifier parameters from support-set statistics, while \method{} predicts low-rank offsets directly in parameter space.
Recovery training in \method{} uses only expert checkpoints as parameter supervision and does not require task datasets or support examples, because the goal is not few-shot generalization but compensating for merging-induced interference.
\FloatBarrier

\section{Experiment details}
\label{app:experiment_details}

This section provides implementation details that complement \cref{sec:learning_recovery_final_corrected} and \cref{sec:experiments} of the main paper.
We first summarize the training and inference procedures of \method{} in pseudocode form.
We then describe the offset generator and parameterization, the baseline protocols used in the experiments, and the computational resources used in our study.

\subsection{Pseudocode of \method{}}
\label{app:pseudocode}

This subsection summarizes the recovery-module training and task-agnostic inference procedures of \method{} using the notation in the main paper.
The task memory bank
$\{(\bm{\mu}_t,\mathbf{V}_{t,k})\}_{t=1}^{T}$
is constructed offline using \cref{eq:mean} and \cref{eq:singluar}, and only the resulting task signatures are retained afterward.
% Once training is finished, task-agnostic inference requires only the merged checkpoint $\vtheta_{\text{merge}}$ and the trained \method{}, where the trained \method{} includes the task memory bank together with the learned task embeddings, offset generators, and shared offsets.

Once training is finished, task-agnostic inference requires only the merged checkpoint $\vtheta_{\text{merge}}$ and the trained \method{}, where the trained \method{} includes the task memory bank together with its trainable \method{} parameters: the task embeddings $\{\ve_t\}_{t=1}^T$, offset generators $\{h^{(l)}\}_{l=1}^L$, and shared offsets $\{\bm{\beta}^{(l)}_s\}_{l=1}^L$.

\cref{alg:retex_train} describes recovery-module training with parameter supervision only.
\cref{alg:retex_inference} describes task-agnostic inference with task identification and task expert recovery.

\begin{algorithm}[h]
\small
\caption{Training of \method{}}
\label{alg:retex_train}
\begin{algorithmic}[1]
% \Require Merged checkpoint $\vtheta_{\text{merge}}$, expert checkpoints $\{\vtheta_t\}_{t=1}^{T}$, trainable \method{}, number of training iterations $N_{\mathrm{iter}}$
% \Ensure Trained \method{}

\Require Merged checkpoint $\vtheta_{\text{merge}}$, expert checkpoints $\{\vtheta_t\}_{t=1}^{T}$, trainable \method{} parameters ($\{\ve_t\}_{t=1}^T$, $\{h^{(l)}\}_{l=1}^L$, $\{\bm{\beta}^{(l)}_s\}_{l=1}^L$), number of training iterations $N_{\mathrm{iter}}$
\Ensure Trained \method{}

\Statex \textbf{\method{} training with parameter supervision only}
\For{$n = 1$ to $N_{\mathrm{iter}}$}
    \State Sample task index \(t \sim \mathcal{U}(1,T)\)
    \State Obtain task embedding \(\ve_t\) from \method{}
    \For{$l = 1$ to $L$}
        \State Generate task-conditioned factor \(\bm{\beta}^{(l)}_{t,g} \gets h^{(l)}(\ve_t)\)
        \State Construct layer offset \(\hat{\bm{\beta}}^{(l)}_{t} \gets \bm{\beta}^{(l)}_{t,g}\bm{\beta}^{(l)}_{s}\)
        \State Recover layer parameters \(\hat{\vtheta}^{(l)}_{t} \gets \vtheta^{(l)}_{\text{merge}} + \hat{\bm{\beta}}^{(l)}_{t}\)
    \EndFor
    \State Stack recovered layer parameters into \(\hat{\vtheta}_{t}\)
    \State Compute reconstruction loss \(\mathcal{L} \gets \left\| \hat{\vtheta}_{t} - \vtheta_{t} \right\|_{2}^{2}\)
    \State Update trainable \method{} parameters \(\{\ve_t\}_{t=1}^T\), \(\{h^{(l)}\}_{l=1}^L\), and \(\{\bm{\beta}^{(l)}_s\}_{l=1}^L\) with Adam
\EndFor
\end{algorithmic}
\end{algorithm}

\begin{algorithm}[h]
\small
\caption{Inference with \method{}}
\label{alg:retex_inference}
\begin{algorithmic}[1]
\Require Input batch $\mathbf{X}=\{\vx_i\}_{i=1}^{B}$, merged checkpoint $\vtheta_{\text{merge}}$, trained \method{}
\Ensure Task ID estimates $\{\hat{t}_i\}_{i=1}^{B}$, recovered parameters $\{\hat{\vtheta}_{t}\}_{t:\,|\mathcal{I}_t|>0}$, predictions $\hat{\mathbf{Y}}$

\Statex \textbf{Before inference: offline task signatures are already stored in trained \method{}}
\Statex \textbf{Task identification}
\State Compute merged-model representations \(\mathbf{Z} \gets f_{\text{merge}}(\mathbf{X})\)

\For{$i = 1$ to $B$}
    \For{$t = 1$ to $T$}
        \State Compute projection residual \(\mathcal{R}_t(\mathbf{z}_i) \gets \left\| \left(\mathbf{I} - \mathbf{V}_{t,k}\mathbf{V}_{t,k}^{\top}\right)(\mathbf{z}_i - \bm{\mu}_t) \right\|_2\)
    \EndFor
    \State Estimate task ID \(\hat{t}_i \gets \arg\min_{t \in \{1,\ldots,T\}} \mathcal{R}_t(\mathbf{z}_i)\)
\EndFor

\Statex \textbf{Grouped task expert recovery}
\State Partition sample indices by predicted task ID \(\mathcal{I}_t \gets \{\, i \in \{1,\ldots,B\} : \hat{t}_i = t \,\}\)

\For{each task $t$ with $|\mathcal{I}_t| > 0$}
    \State Obtain task embedding \(\ve_t\) from \method{}
    \For{$l = 1$ to $L$}
        \State Generate task-conditioned factor \(\bm{\beta}^{(l)}_{t,g} \gets h^{(l)}(\ve_t)\)
        \State Construct layer offset \(\hat{\bm{\beta}}^{(l)}_{t} \gets \bm{\beta}^{(l)}_{t,g}\bm{\beta}^{(l)}_{s}\)
        \State Recover layer parameters \(\hat{\vtheta}^{(l)}_{t} \gets \vtheta^{(l)}_{\text{merge}} + \hat{\bm{\beta}}^{(l)}_{t}\)
    \EndFor
    \State Stack recovered layer parameters into \(\hat{\vtheta}_{t}\)
    \State Run the corresponding sub-batch \(\hat{\mathbf{Y}}_{\mathcal{I}_t} \gets f_{\hat{\vtheta}_{t}}(\mathbf{X}_{\mathcal{I}_t})\)
\EndFor

\State Return \(\{\hat{t}_i\}_{i=1}^{B}\), \(\{\hat{\vtheta}_{t}\}_{t:\,|\mathcal{I}_t|>0}\), and \(\hat{\mathbf{Y}}\)
\end{algorithmic}
\end{algorithm}

\subsection{Module architecture and parameters}
\label{app:module_architecture}

This subsection details the module architecture and parameterization of \method{} in the offset-only configuration justified in \cref{app:offset_justification}.
In this setting, \method{} generates a task-specific offset tensor $\bm{\beta}^{(l)}_{t}$ for each layer $l$ and applies it to the merged parameters.
The offset is factorized as
\begin{equation}
\bm{\beta}^{(l)}_{t}
=
\bm{\beta}^{(l)}_{t,g}\bm{\beta}^{(l)}_{s},
\end{equation}
where $\bm{\beta}^{(l)}_{t,g}$ is produced by a lightweight offset generator $h^{(l)}$ conditioned on a task embedding $\ve_t$, and $\bm{\beta}^{(l)}_{s}$ is a learnable offset shared across all task experts.
In implementation, $h^{(l)}$ is a single linear layer whose output dimension matches the required low-rank factorization for layer $l$.

The core components involved in offset generation, their shapes for a 2D layer, and their PyTorch-like forms are summarized in \cref{tab:module_architecture}.
We denote the number of tasks as $T$, the task embedding dimension as $d_{\mathrm{emb}}$, the dimensions of a 2D parameter matrix as $(a,b)$, and the target rank as $r$.

\begin{table}[t]
\centering
\caption{\textbf{Core components for offset generation in \method{} (offset-only), exemplified for a 2D layer.}}
\label{tab:module_architecture}
\begin{tabular}{lll}
\toprule
\textbf{Component} & \textbf{Shape} & \textbf{PyTorch-like form} \\
\midrule
Task embedding $\ve_t$ & $(T, d_{\mathrm{emb}})$ & \texttt{nn.Embedding(T, d\_emb)} \\
Offset generator $h^{(l)}$ & $(d_{\mathrm{emb}}, a\cdot r)$ & \texttt{nn.Linear(d\_emb, a*r)} \\
Generated $\bm{\beta}^{(l)}_{t,g}$ & $(a,r)$ & reshaped output of $h^{(l)}(\ve_t)$ \\
Shared $\bm{\beta}^{(l)}_{s}$ & $(r,b)$ & \texttt{nn.Parameter} \\
Effective offset $\bm{\beta}^{(l)}_{t}$ & $(a,b)$ & $\bm{\beta}^{(l)}_{t,g}\bm{\beta}^{(l)}_{s}$ \\
\bottomrule
\end{tabular}
\end{table}

\noindent\textbf{Offset generator.}
For each layer $l$, the offset generator $h^{(l)}$ takes the task embedding $\ve_t$ as input and outputs a vector in $\mathbb{R}^{a\cdot r}$, which is reshaped into the task-conditioned factor $\bm{\beta}^{(l)}_{t,g}\in\mathbb{R}^{a\times r}$ for a 2D layer of shape $(a,b)$.
The shared offset $\bm{\beta}^{(l)}_{s}\in\mathbb{R}^{r\times b}$ is a learnable parameter initialized with zeros and shared across tasks for layer $l$.
The final offset $\bm{\beta}^{(l)}_{t}$ is given by the matrix product $\bm{\beta}^{(l)}_{t,g}\bm{\beta}^{(l)}_{s}$.

\noindent\textbf{Parameter handling for different dimensionalities.}
The shapes of the task-conditioned factor $\bm{\beta}^{(l)}_{t,g}$ and the shared offset $\bm{\beta}^{(l)}_{s}$ are adapted according to the dimensionality of the original parameter tensor $\vtheta^{(l)}$:

\begin{itemize}
    \item \textbf{0D parameters.}
    $\bm{\beta}^{(l)}_{t,g}$ has shape $(1)$ and $\bm{\beta}^{(l)}_{s}$ is a learnable scalar.
    The resulting offset is computed by element-wise multiplication.

    \item \textbf{1D parameters.}
    For an original parameter of shape $(D)$, $\bm{\beta}^{(l)}_{t,g}$ has shape $(D)$ and $\bm{\beta}^{(l)}_{s}$ is a learnable scalar that scales $\bm{\beta}^{(l)}_{t,g}$ element-wise.

    \item \textbf{2D parameters.}
    For an original parameter of shape $(a,b)$ and rank $r$, $\bm{\beta}^{(l)}_{t,g}$ has shape $(a,r)$ and $\bm{\beta}^{(l)}_{s}$ has shape $(r,b)$.

    \item \textbf{4D parameters.}
    For an original parameter of shape $(c_{\mathrm{out}}, c_{\mathrm{in}}, k_h, k_w)$, we reshape it to $(c_{\mathrm{out}}, c_{\mathrm{in}}k_hk_w)$ and apply the same 2D decomposition.
    In this case, $\bm{\beta}^{(l)}_{t,g}$ has shape $(c_{\mathrm{out}}, r)$ and $\bm{\beta}^{(l)}_{s}$ has shape $(r, c_{\mathrm{in}}k_hk_w)$.
    The resulting offset is reshaped back to the original 4D form.
\end{itemize}

\noindent\textbf{Rank adjustment.}
\method{} uses a common target rank $r$ across layers.
For 2D parameters of shape $(d_1,d_2)$, or reshaped 4D parameters, if $r \ge \min(d_1,d_2)$, the effective rank is reduced to $\lfloor \min(d_1,d_2)/2 \rfloor$.
Otherwise, the target rank $r$ is used.
This prevents degenerate factorizations in small layers.

\subsection{Baseline details}
\label{app:baseline_details}
The baseline approaches employed for comparative evaluation in our experiments are detailed as follows:
\begin{itemize}
    \item \textbf{Individual Models}: This represents the standard performance benchmark where a distinct, fine-tuned model is dedicated to each specific task. These models operate independently and are not designed for multi-task execution.

    \item \textbf{Weight Averaging}~\cite{shoemake1985animating,utans1996weight}: As a foundational technique in model merging, this method directly computes an average of the parameters from all constituent task-specific models. It operates under the simplifying assumption that all tasks contribute equally, hence applying uniform weighting to each model.

    \item \textbf{Task Arithmetic}~\cite{ilharco2023editing}: This approach first defines a task vector \( \boldsymbol{\tau}_t \) for each task \( t \) as the parametric difference between the fine-tuned model \( \vtheta_t \) and the initial pre-trained model \( \vtheta_0 \) (i.e., \( \boldsymbol{\tau}_t = \vtheta_t - \vtheta_0 \)). A unified model \( \vtheta_\text{merge} \) is then constructed by adding a scaled sum of these task vectors to the pre-trained parameters, formulated as \( \vtheta_\text{merge} = \vtheta_0 + \lambda \sum_{t=1}^T \boldsymbol{\tau}_t \). The scaling factor \( \lambda \) is a hyperparameter selected from the range \( \{0.0, 0.1, \dots, 1.0\} \) to maximize average performance across all task validation sets.

    \item \textbf{TIES-Merging}~\cite{yadav2023ties}: This method refines task vectors before merging through a three-step process: Trim, Elect Sign, and Merge. In the Trim step, only the top 20\% of values by magnitude in each task vector are retained, with others zeroed out. The Elect Sign step (implicitly handled by the original task vector signs) and the subsequent Merge step proceed analogously to Task Arithmetic, including the hyperparameter tuning for the scaling factor.

    \item \textbf{AdaMerging}~\cite{yang2024adamerging}: This method proposed an adaptive model merging technique that learned merging coefficients using unlabeled test data. Operating as an unsupervised task arithmetic scheme, it utilized entropy minimization on the test samples as a surrogate objective function to iteratively update the coefficients. The approach adjusted these coefficients in either a task-wise or layer-wise manner to address parameter conflicts and task correlations.

    \item \textbf{Consensus TA}~\cite{wang2024localizing}: This technique first utilizes a multi-task model to derive binary masks that highlight parameters deemed critical for each task. The sparsity of these masks is controlled by a hyperparameter \( \lambda \), optimized over \( \{0.2, 0.3, 0.4, 0.5, 0.6\} \) using validation performance. Each task-specific mask is then applied to its corresponding task vector via an element-wise (Hadamard) product before the final merging, which follows the Task Arithmetic procedure.

    \item \textbf{EMR-Merging}~\cite{huang2024emr}: This approach begins by creating a consolidated "unified task vector" derived from the sign and magnitude of individual task vectors. It then computes task-specific binary masks and rescaling vectors for each task. The final merged model for a given task is obtained by an element-wise multiplication of this unified task vector with the corresponding task-specific mask and rescaler. This method is presented as hyperparameter-free.

    \item \textbf{Twin-Merging}~\cite{lu2024twin}: This method involves first training a router for dynamic task identification. A shared common expert is then established using Task Arithmetic with a predetermined scaling factor. Subsequently, exclusive knowledge vectors unique to each task are extracted, typically using Singular Value Decomposition (SVD) or a trimming procedure similar to TIES-Merging (with Trim reported as superior). At inference, the router assigns task-specific weights to an input. The final model output is derived by combining the shared expert with a weighted sum of these exclusive knowledge vectors, using the router-determined weights.

    \item \textbf{DaWin}~\cite{oh2025dawin}: This dynamic merging technique assigns an input-specific weight to each task model. These weights are calculated based on the Shannon entropy of the outputs from both the task-specific model and the pre-trained base model for the given input. To optimize inference speed, a Beta Mixture Model (BMM) can optionally be trained to approximate these dynamic weights, typically using \( K=3 \) mixture components.

    \item \textbf{MoW-Merging}~\cite{ye2025dynamic}: This method performs sample-wise weight fusion using a lightweight gating network. It predicts input-dependent coefficients and applies dynamic merging only to layers with strong interference, while layers with high similarity are handled by static merging.

    \item \textbf{WEMoE}~\cite{tang2024merging}: This method performs input-adaptive merging by decomposing model parameters into a shared component and task-specific components. A router predicts input-dependent weights, which are then used to reconstruct an effective model for the current input.
    
\end{itemize}

\subsection{Computational resources and training time}
\label{app:computational_resources_training}

All experimental procedures reported in this work, encompassing the training and inference of our proposed \method{} framework, as well as performance evaluations and computational cost measurements for baseline methods, were conducted on specific GPU hardware.
For experiments involving 14 tasks or fewer, NVIDIA GeForce RTX 3090 GPUs were utilized.
As a specific example of training duration, the training of \method{} for the 8-task vision benchmark typically completed in approximately 53 minutes and 49 seconds on a single NVIDIA GeForce RTX 3090 GPU.
For more extensive experiments involving 20 tasks or more, NVIDIA H100 80GB HBM3 GPUs were employed to accommodate the increased computational demands.
\FloatBarrier

\section{Additional experiments}
\label{app:additional_experiments}

This section provides additional experiments that complement the results in the main paper.
We first give a more detailed formulation of the Dirichlet-sampling training strategy used for soft recovery in unseen-task generalization as discussed in \textcolor{linkred}{Sec.~5.2}.
We then report additional unseen-task results under alternative unseen-task splits, study scalability on larger task suites, and finally evaluate \method{} on a fine-grained benchmark with high inter-task similarity.

\subsection{Dirichlet sampling training for unseen-task generalization}
\label{app:dirichlet_training}

The main paper studies unseen-task generalization through soft recovery over seen experts.
Here we provide a more explicit formulation of the additional training step used for that setting and an ablation that compares the base one-hot-indicator version of \method{} against its Dirichlet-sampling extension.

The base \method{} model is first trained with one-hot task indicators, exactly as in the standard task-expert recovery setting.
To further improve unseen-task generalization, we then fine-tune this already trained model with Dirichlet sampling.
Specifically, a continuous mixing vector
$\mathbf{w}=[w_1,\ldots,w_T]^\top$
is sampled from a symmetric Dirichlet distribution,
$\mathbf{w}\sim\mathrm{Dir}(\alpha)$,
where $T$ is the number of seen tasks.
Using this mixing vector, we define a soft task embedding and a soft target expert as
\begin{equation}
\label{eq:soft_embedding_target}
\ve_{\mathbf{w}} = \sum_{t=1}^{T} w_t \ve_t,
\qquad
\vtheta_{\mathbf{w}} = \sum_{t=1}^{T} w_t \vtheta_t.
\end{equation}
Given $\ve_{\mathbf{w}}$, the offset generator predicts a soft offset $\hat{\bm{\beta}}_{\mathbf{w}}$, and the fine-tuning objective becomes
\begin{equation}
\label{eq:retex_soft_loss}
\mathcal{L}_{\text{soft}}
=
\mathbb{E}_{\mathbf{w}\sim\mathrm{Dir}(\alpha)}
\left[
\left\|
(\vtheta_{\text{merge}}+\hat{\bm{\beta}}_{\mathbf{w}}) - \vtheta_{\mathbf{w}}
\right\|_2^2
\right].
\end{equation}
This additional training step exposes \method{} to convex combinations of seen experts and encourages a smoother recovery manifold for unseen inputs.

At inference time, instead of hard nearest-subspace assignment, the SVD projection residuals $\{\mathcal{R}_t(\mathbf{z})\}_{t=1}^{T}$ are converted into soft mixture weights
\begin{equation}
\label{eq:soft_routing}
p_t(\mathbf{z})
=
\frac{\exp(-\mathcal{R}_t(\mathbf{z})/\tau)}
{\sum_{j=1}^{T}\exp(-\mathcal{R}_j(\mathbf{z})/\tau)},
\end{equation}
where $\tau$ is a temperature parameter.
The resulting soft embedding
\[
\ve_{\mathbf{p}(\mathbf{z})}
=
\sum_{t=1}^{T} p_t(\mathbf{z}) \ve_t
\]
is fed to the fine-tuned offset generator to recover input-adaptive parameters
\[
\vtheta_{\mathbf{p}(\mathbf{z})}
=
\vtheta_{\text{merge}} + \hat{\bm{\beta}}_{\mathbf{p}(\mathbf{z})}.
\]

\begin{table}[t!]
\centering
\caption{
Performance comparison on seen and unseen tasks (EuroSAT and MNIST).
Base \method{} improves over Weight Averaging (WA), and further training with Dirichlet sampling consistently improves unseen-task generalization while remaining robust across different $\alpha$ values.
}
\label{tab:ood_robustness}
\footnotesize
\setlength{\tabcolsep}{4pt}
\begin{tabular}{lcc}
\toprule
\textbf{Method} & \textbf{Seen Task Accuracy} & \textbf{Unseen Task Accuracy} \\
\midrule
Weight Averaging (WA) & 63.7 & 57.8 \\
\midrule
\rowcolor[HTML]{EFEFEF}
\multicolumn{3}{l}{\textbf{\textit{(One-hot indicator training)}}} \\
WA + \method{} & 90.4 & 63.4 \\
\midrule
\rowcolor[HTML]{EFEFEF}
\multicolumn{3}{l}{\textbf{\textit{(Dirichlet-sampling training)}}} \\
WA + \method{} ($\alpha = 0.01$) & 90.1 & 67.1 \\
WA + \method{} ($\alpha = 0.1$)  & 90.1 & 67.2 \\
WA + \method{} ($\alpha = 1.0$)  & 90.2 & 67.2 \\
WA + \method{} ($\alpha = 2.0$)  & 90.1 & 67.2 \\
WA + \method{} ($\alpha = 5.0$)  & 90.1 & 67.3 \\
\bottomrule
\end{tabular}
\end{table}

\cref{tab:ood_robustness} compares three stages: the Weight Averaging baseline, the base \method{} model trained with one-hot task indicators, and the further fine-tuned variants with Dirichlet sampling.
The base one-hot-indicator version already improves over Weight Averaging on both seen and unseen tasks, showing that expert recovery alone provides a stronger starting point.
Additional Dirichlet-sampling training then yields a clear gain on unseen tasks while preserving seen-task performance.
Moreover, the unseen-task improvement remains stable across a wide range of $\alpha$ values, indicating that the benefit of soft recovery does not depend strongly on the specific concentration parameter.

\subsection{More unseen-task generalization experiments}
\label{app:more_ood_experiments}

Beyond the EuroSAT/MNIST unseen-task setting above, we further evaluate the Dirichlet-sampling version of \method{} under alternative unseen-task compositions and larger unseen-task benchmarks.

\begin{table}[t!]{
\caption{\textbf{Comparison of OOD robustness and seen task preservation across different merging strategies.} 
Evaluating on unseen tasks (RESISC45, SVHN), \method{} consistently outperforms baseline merging methods (Weight Averaging, Task Arithmetic, TIES-Merging) in terms of generalization. Furthermore, it successfully recovers Seen task accuracy, which is often degraded in baselines due to parameter interference, to over 90\%.}
\label{tab:ood_type1}
%\textcolor{red}{
%Evaluation of soft target ReTeX applied to different merging methods 
%}
}

%\vspace{-1em}
% \label{tab:ood_robustness_table}
\centering
\setlength{\tabcolsep}{4pt}
% [수정] \resizebox로 tabular 감싸기
\resizebox{\columnwidth}{!}{%
\begin{tabular}{lcc}
\toprule
\textbf{Method} & \textbf{Seen Task Accuracy} & \textbf{Unseen Task Accuracy} \\ 
\midrule
Weight Averaging & 68.5 & 53.8 \\
\rowcolor[HTML]{FFF2CC}
Weight Averaging + \method{} & $\mathbf{91.3}$ & $\mathbf{54.5}$ \\

\midrule
Task Arithmetic~\cite{ilharco2023editing} & 74.7 & 50.8 \\
\rowcolor[HTML]{FFF2CC}
Task Arithmetic + \method{} & $\mathbf{91.2}$ & $\mathbf{55.9}$ \\

\midrule
TIES-Merging~\cite{yadav2023ties} & 77.3 & 49.7 \\
\rowcolor[HTML]{FFF2CC}
TIES-Merging + \method{} & $\mathbf{90.6}$ & $\mathbf{56.0}$ \\

\midrule
AdaMerging~\cite{yang2024adamerging} & 76.3 & 51.6\\ 
\rowcolor[HTML]{FFF2CC}
AdaMerging + \method{} & $\mathbf{91.3}$ & $\mathbf{55.2}$ \\ 
\bottomrule
\end{tabular}
} % \resizebox 닫기
%\vspace{-1em}
\end{table}

\noindent\textbf{Robust generalization across merging methods.}
As shown in \cref{tab:ood_type1}, \method{} demonstrated strong robustness by consistently enhancing the generalization capability on unseen tasks (RESISC45 and SVHN) across all evaluated merging methods.
Regardless of the base merging strategy—whether Weight Averaging, Task Arithmetic, TIES-Merging, or AdaMerging—the integration of \method{} consistently improves out-of-distribution (OOD) performance. 
% This indicates that the regularization effect induced by Dirichlet sampling allows the merged model to robustly adapt to novel domains, demonstrating its broad applicability across various merging mechanisms.
These results are consistent with the view that Dirichlet-sampling training improves unseen-task generalization across merging methods, demonstrating its broad applicability.

\begin{table}[t!]
\caption{\textbf{Comparison of robustness on a 30-task benchmark (20 seen, 10 unseen).}
We evaluate unseen-task generalization on a larger benchmark and compare Weight Averaging, Task Arithmetic, TIES-Merging, and AdaMerging with and without \method{}.
\method{} consistently improves unseen-task performance across all base merging methods while substantially recovering seen-task accuracy.}

\label{tab:ood_robustness_table}
\centering
\setlength{\tabcolsep}{4pt}
\resizebox{\columnwidth}{!}{%
\begin{tabular}{lcc}
\toprule
\textbf{Method} & \textbf{Seen Task Accuracy} & \textbf{Unseen Task Accuracy} \\ 
\midrule
Weight Averaging & 55.0 & 54.7 \\
\rowcolor[HTML]{FFF2CC}
Weight Averaging + \method{} & $\mathbf{88.4}$ & $\mathbf{55.0}$ \\ 
% 여기부터 추가
\midrule
\textcolor{black}{Task Arithmetic~\cite{ilharco2023editing}} & \textcolor{black}{61.0} & \textcolor{black}{54.8} \\
\rowcolor[HTML]{FFF2CC}
\textcolor{black}{Task Arithmetic + \method{}} & $\mathbf{\textcolor{black}{89.0}}$ & $\mathbf{\textcolor{black}{55.2}}$ \\

\midrule
\textcolor{black}{TIES-Merging~\cite{yadav2023ties}} & \textcolor{black}{57.8} & \textcolor{black}{51.0} \\
\rowcolor[HTML]{FFF2CC}
\textcolor{black}{TIES-Merging + \method{}} & $\mathbf{\textcolor{black}{85.5}}$ & $\mathbf{\textcolor{black}{55.2}}$ \\

\midrule
\textcolor{black}{AdaMerging~\cite{yang2024adamerging}} & \textcolor{black}{69.2} & \textcolor{black}{54.2}\\ 
\rowcolor[HTML]{FFF2CC}
\textcolor{black}{AdaMerging + \method{}} & $\mathbf{\textcolor{black}{89.0}}$ & $\mathbf{\textcolor{black}{55.4}}$ \\ 

\bottomrule
\end{tabular}
}
\end{table}

\noindent\textbf{Scalability on large-scale benchmark.}
To verify the scalability of our approach, we extended the evaluation to a larger benchmark consisting of 30 tasks (20 seen and 10 unseen). 
As presented in \cref{tab:ood_robustness_table}, \method{} demonstrated robust performance even in this challenging large-scale setting. 
Compared to the Weight Averaging baseline, \method{} significantly recovers the accuracy on seen tasks (from 55.0\% to 88.4\%) while maintaining superior generalization capability on unseen tasks, confirming that further training with Dirichlet sampling is effective regardless of the number of tasks.

\subsection{Scalability to larger task suites}
\label{app:larger_task_suites}

\noindent\textbf{Setting.}
To assess scalability to larger and more heterogeneous task collections, we extend the ViT-B/32 vision benchmark in \cref{sec:vision_tasks} of the main paper from 30 tasks to 50 tasks.
The 30-task suite augments the 20-task setting with ten additional datasets:
Vegetables~\cite{ahmed2021dcnn},
Kvasir-v2~\cite{pogorelov2017kvasir},
Intel Images~\cite{bansal2019intel},
Weather~\cite{xiao2021classification},
Cats and dogs~\cite{cukierskidogs},
MangoLeafBD~\cite{ahmed2023mangoleafbd},
Beans~\cite{beansdata},
Landscape Recognition~\cite{Landscape},
Garbage Classification~\cite{cchang_2018},
and Fruits-360~\cite{muresan2018fruit}.
The 50-task suite further adds twenty more datasets:
PlantVillage~\cite{mohanty2016plantdisease},
OCT2017~\cite{kermany2018identifying},
ASL Alphabet~\cite{grassknoted_asl_alphabet},
AID~\cite{xia2017aid},
CUB200-2011~\cite{wah2011cub},
ObjectCategories-256~\cite{griffin2007caltech256},
Animals-10~\cite{corrado_animals10},
EyePACS~\cite{cuadros2009eyepacs},
StanfordDogs~\cite{khosla2011fgvc},
MIT Indoor-67~\cite{quattoni2009indoor},
Imagenette2-320~\cite{howard2019imagenette},
Caltech-101~\cite{fei2007caltech101},
RAF-DB~\cite{li2017reliable},
FGVC-Aircraft~\cite{maji2013aircraft},
DeepWeeds~\cite{olsen2019deepweeds},
Imagewoof~\cite{howard2019imagenette},
NABirds~\cite{vanhorn2015nabirds},
Leafsnap~\cite{kumar2012leafsnap},
AFHQ~\cite{choi2020starganv2},
and ArtBench10~\cite{liao2022artbench}.
All methods are evaluated under the same task-agnostic protocol as in the main vision experiments.

\noindent\textbf{Results.}
\begin{table}[t]
\caption{\textbf{Multi-task performance on ViT-B/32 across increasing numbers of vision tasks.}
Values in parentheses $_{(\cdot)}$ indicate normalized accuracy (merged / individual).}
\centering
\label{tab:vision_vitb32_scalability}
\begin{tabular}{l|ccccc}
\toprule
\multicolumn{1}{l|}{\textbf{Method}} & \textbf{8 tasks} & \textbf{14 tasks} & \textbf{20 tasks} & \textbf{30 tasks} & \textbf{50 tasks} \\
\midrule
\multicolumn{1}{l|}{Zero-shot}  & 48.3 & 57.2 & 56.1 & 55.5 & 46.8 \\
\multicolumn{1}{l|}{Individual} & 92.9 & 90.9 & 91.4 & 93.1 & 90.8 \\
\midrule

\rowcolor[HTML]{EFEFEF}
\multicolumn{6}{l}{\textbf{\textit{(Static model merging)}}} \\
\multicolumn{1}{l|}{Weight Averaging}                          & 66.3$_{(72.1)}$ & 64.3$_{(71.1)}$ & 61.0$_{(67.5)}$ & 59.1$_{(64.2)}$ & $47.8_{(52.6)}$ \\
\multicolumn{1}{l|}{Task Arithmetic~\cite{ilharco2023editing}} & 70.6$_{(76.3)}$ & 65.3$_{(72.0)}$ & 60.5$_{(66.7)}$ & 58.0$_{(62.8)}$ & 46.8$_{(51.5)}$ \\
\multicolumn{1}{l|}{TIES-Merging~\cite{yadav2023ties}}        & 74.5$_{(80.3)}$ & 65.1$_{(71.9)}$ & 62.3$_{(68.8)}$ & 59.6$_{(64.7)}$ & 47.5$_{(52.3)}$ \\
\multicolumn{1}{l|}{Consensus TA~\cite{wang2024localizing}}     & 75.0$_{(80.8)}$ & 70.4$_{(77.4)}$ & 65.4$_{(72.0)}$ & 63.4$_{(68.5)}$ & $52.1_{(57.4)} $ \\
\multicolumn{1}{l|}{AdaMerging~\cite{yang2024adamerging}}     & 75.9$_{(81.7)}$ & 70.5$_{(77.3)}$ & 66.2$_{(72.7)}$ & 56.8$_{(61.3)}$ & 49.4$_{(54.4)}$ \\
\midrule

\rowcolor[HTML]{EFEFEF}
\multicolumn{6}{l}{\textbf{\textit{(Dynamic model merging)}}} \\
\multicolumn{1}{l|}{Twin-Merging~\cite{lu2024twin}}           & 84.0$_{(90.3)}$ & 70.0$_{(76.7)}$ & 57.5$_{(61.8)}$ & 60.1$_{(65.2)}$ & 36.4$_{(40.1)}$ \\
\multicolumn{1}{l|}{DaWin~\cite{oh2025dawin}}                 & 89.0$_{(95.3)}$ & 73.8$_{(80.3)}$ & 52.8$_{(57.7)}$ & 40.3$_{(42.9)}$ & 33.1$_{(36.5)}$ \\
\multicolumn{1}{l|}{MoW-Merging~\cite{ye2025dynamic}}         & 88.1$_{(94.8)}$ & 83.2$_{(91.5)}$ & 79.3$_{(86.8)}$ & 56.4$_{(60.6)}$ & 30.1$_{(33.1)}$ \\
\multicolumn{1}{l|}{WEMoE~\cite{tang2024merging}}             & 90.4$_{(97.3)}$ & 83.1$_{(90.7)}$ & 74.4$_{(81.2)}$ & 67.1$_{(72.4)}$ & 51.6$_{(56.8)}$ \\
\midrule

\rowcolor[HTML]{FFF2CC}
\multicolumn{1}{l|}{\textbf{\method{} (Ours)}} 
& $\mathbf{92.2}_{(99.3)}$ 
& $\mathbf{89.8}_{(98.8)}$ 
& $\mathbf{89.8}_{(98.3)}$
& $\mathbf{91.2}_{(97.9)}$ 
& $\mathbf{81.5}_{(89.8)}$ \\
\bottomrule
\end{tabular}
\end{table}
\cref{tab:vision_vitb32_scalability} summarizes performance as the number of vision tasks increases from 8 to 50.
The 30-task results are consistent with the trend reported in \cref{sec:vision_tasks}: static and dynamic merging baselines deteriorate as the task suite grows, while \method{} remains substantially stronger than all available baselines even at larger scale.
This trend becomes even more pronounced at 50 tasks.
% Among the available baselines, the strongest 50-task result reaches 59.1 for Task Arithmetic in the static group and 51.6 for WEMoE in the dynamic group, whereas \method{} attains 81.5, corresponding to 89.8\% normalized accuracy.
%-- 구체적 성능 언급 수정
While the strongest baselines in both static and dynamic groups suffer severe performance drops at this scale, \method{} exhibits minimal degradation, maintaining a striking performance gap over all other methods.

These results indicate that offset-based expert recovery continues to scale effectively even when both the number of tasks and task heterogeneity increase substantially.

\subsection{Fine-grained tasks with high inter-task similarity}
\label{app:fine_grained_similarity}

\begin{figure*}[t]
    \centering
    \begin{minipage}{0.49\linewidth}
        \centering
        \includegraphics[width=\linewidth]{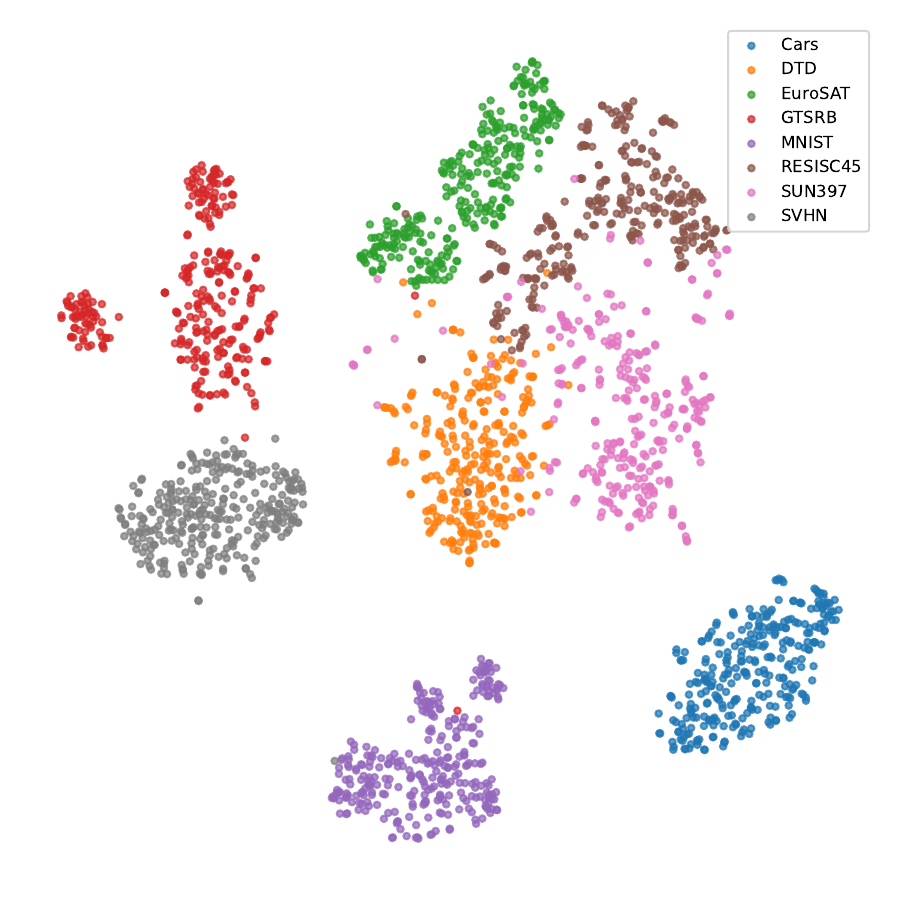}
    \end{minipage}
    \hfill
    \begin{minipage}{0.49\linewidth}
        \centering
        \includegraphics[width=\linewidth]{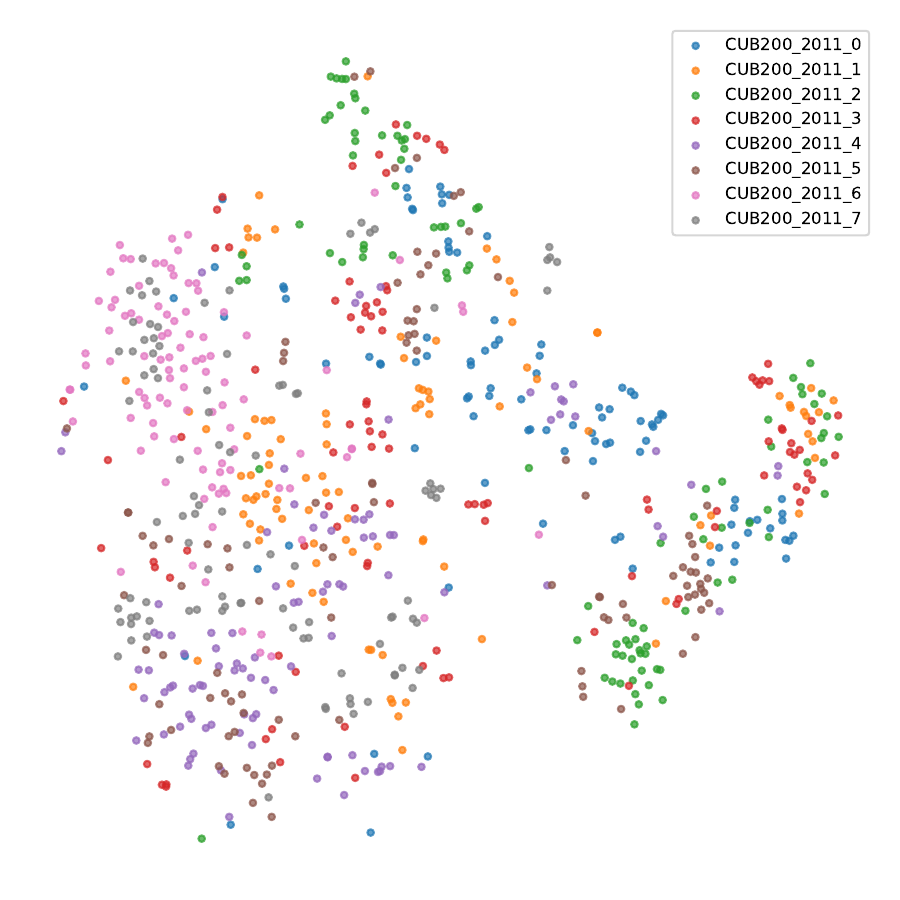}
    \end{minipage}
    \caption{\textbf{Feature overlap comparison between the original vision benchmark and the fine-grained bird benchmark.}
    The visualization is obtained by applying t-SNE to the final merged-model representations.
    The bird benchmark exhibits substantially stronger feature overlap across tasks, indicating a more challenging setting for task identification and task expert recovery.}
    \label{fig:finegrained_feature_overlap}
\end{figure*}

\cref{fig:finegrained_feature_overlap} compares feature distributions between the original vision benchmark and the fine-grained bird benchmark.
Compared with the original benchmark, the bird setting shows substantially stronger feature overlap across tasks, making task identification more challenging.
To rigorously assess robustness under such high inter-task similarity, we construct a fine-grained benchmark from the CUB-200-2011 birds dataset~\cite{wah2011cub}.
Specifically, we partition the bird classes into eight disjoint subsets, treat each subset as an individual task, and fine-tune one expert per task using a CLIP ViT-B/32 backbone.
Since all tasks belong to the same visual domain and differ only through subtle fine-grained distinctions between bird categories, this setting provides a stronger stress test than the heterogeneous vision benchmark used in the main paper.

\begin{table*}[t!]
\centering
\caption{
\textbf{Multi-task performance for 8 fine-grained bird classification tasks (CUB-200 dataset~\cite{wah2011cub}).} We report experimental results for ViT-B/32 on 8 fine-grained bird classification tasks from the CUB-200 dataset. Bold values represent the best performance among all methods, excluding the individual task-specific baselines.}

\label{tab:bird}
\setlength\tabcolsep{4.0pt}
\resizebox{0.5\textwidth}{!}{
\begin{tabular}{l|c}
\toprule
\textbf{Method} 
& \textbf{Mean Accuracy}  \\
\midrule
Fine-tuned & 78.1 \\
Zero-shot & 57.3 \\
\midrule
\rowcolor[HTML]{EFEFEF}
\multicolumn{2}{l}{\textbf{\textit{(Static model merging)}}} \\
Weight Averaging~\cite{utans1996weight} & 60.9   \\
Task Arithmetic~\cite{ilharco2023editing} & 62.1\\
TIES-Merging~\cite{yadav2023ties}    & 61.4\\
Consensus TA~\cite{wang2024localizing} & 60.1 \\
AdaMerging~\cite{yang2024adamerging} & 62.5 \\

\midrule
\rowcolor[HTML]{EFEFEF}
\multicolumn{2}{l}{\textbf{\textit{(Dynamic model merging)}}} \\
Twin-Merging~\cite{lu2024twin} & 64.6\\
DaWin~\cite{oh2025dawin}  & 61.0\\
WEMoE~\cite{tang2024merging} & 57.0 \\
MoW-Merging~\cite{ye2025dynamic}   & 58.9 \\
\midrule
\rowcolor[HTML]{FFF2CC}
\textbf{\method{} (Ours) }    &  $\mathbf{ 71.3 }$ \\

\bottomrule
\end{tabular}
}
\end{table*}

\cref{tab:bird} shows that \method{} outperforms all compared static and dynamic model merging baselines even in this fine-grained setting.
These results suggest that the proposed task identification and task expert recovery pipeline remains effective even when task similarity is high and feature overlap is substantial.

\begingroup
\color{black}
\subsection{Robustness under corrupted inputs}
\label{app:robustness_under_corrupted_input}

\noindent\textbf{Setting.}
We further evaluate the robustness of \method{} under input corruptions.
This experiment tests whether the task identification and recovery pipeline remains reliable when the input distribution is perturbed at test time.
We consider the clean test set and seven corruption types: motion blur, impulse noise, Gaussian noise, pixelation, spatter, contrast change, and JPEG compression.
For each corrupted test set, we report the task identification accuracy, the performance of \method{} with oracle task IDs, the performance of \method{} with predicted task IDs, and the performance of representative dynamic merging baselines.

\begin{table}[t]
\centering
\caption{\textbf{Robustness under corrupted inputs.}
We report task identification accuracy and downstream accuracy on clean and corrupted test sets.
\method{} (oracle) uses the ground-truth task ID and serves as an upper bound without task-identification errors.}
\label{tab:corruption_robustness}
\resizebox{\linewidth}{!}{%
\begin{tabular}{lccccc}
\toprule
Test set & Task-ID Acc. & \method{} (oracle) & \method{} & AdaMerging & WEMoE \\
\midrule
Clean    & 98.90 & 92.60 & \textbf{92.20} & 75.90 & 90.40 \\
Motion   & 96.53 & 85.65 & \textbf{83.73} & 61.16 & 80.64 \\
Impulse  & 81.75 & 80.54 & \textbf{70.61} & 51.51 & 69.90 \\
Gaussian & 85.10 & 77.29 & \textbf{71.57} & 53.85 & 70.87 \\
Pixelate & 98.48 & 91.16 & \textbf{90.26} & 62.78 & 85.24 \\
Spatter  & 98.49 & 91.62 & \textbf{90.96} & 69.73 & 87.74 \\
Contrast & 97.21 & 88.59 & \textbf{87.18} & 65.48 & 84.92 \\
JPEG     & 96.48 & 84.37 & \textbf{82.49} & 61.76 & 79.73 \\
\bottomrule
\end{tabular}%
}
\end{table}

\noindent\textbf{Results.}
As shown in \cref{tab:corruption_robustness}, \method{} remains consistently stronger than the compared baselines across all corrupted test sets.
Although corruptions can reduce task identification accuracy, especially under impulse and Gaussian noise, the recovered model still maintains higher accuracy than AdaMerging and WEMoE in all cases.
The gap between \method{} and \method{} (oracle) further indicates that part of the remaining degradation comes from task identification errors rather than from the recovery module itself.
These results suggest that the proposed task identification and expert recovery pipeline is robust to input-level distribution shifts.
\endgroup
\FloatBarrier

\section{More ablation study}
\label{app:more_ablation_study}

\subsection{Representation source and feature layer}
\label{app:representation_source_and_layer}

\noindent\textbf{Setting.}
This subsection studies how task identification depends on
(i) the representation source used to build task signatures and extract test features, and
(ii) the layer index from which the feature is taken.
We evaluate the 8-task vision setting with a CLIP ViT-B/32 backbone.
We consider five representation sources: representations extracted from the pre-trained model (PRE), and representations extracted from merged models obtained by Weight Averaging (WA), Task Arithmetic (TA), TIES-Merging (TIES), and AdaMerging (ADA).
For each source, we vary the feature layer index used both for task signatures and for test-time feature extraction, and report task identification accuracy under the SVD projection-residual rule in \cref{sec:task_identification}.

\begin{figure}[t]
    \centering
    \includegraphics[width=\linewidth]{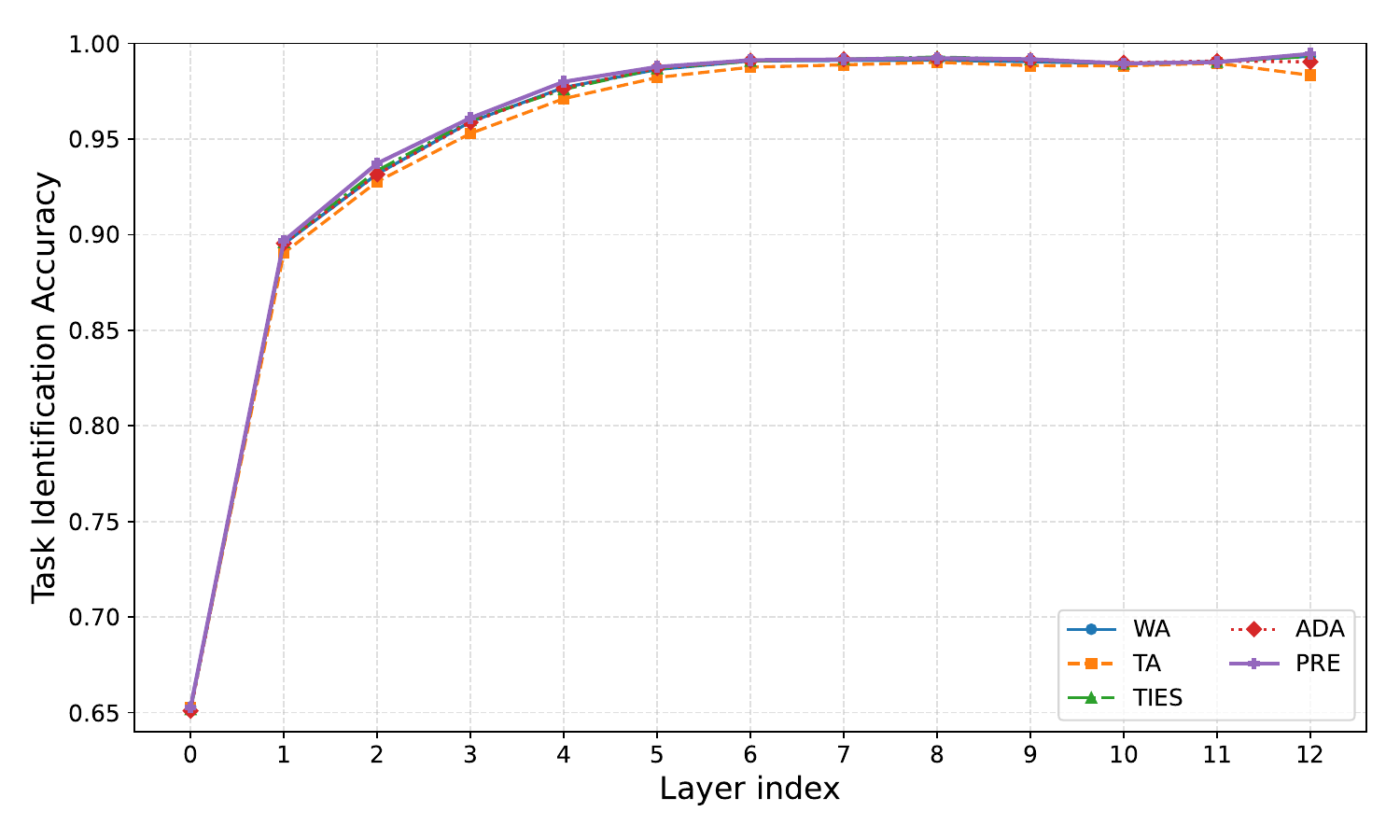}
    \caption{\textbf{Task identification accuracy across representation sources and feature layers.}
    We compare five representation sources: the pre-trained model (PRE) and merged models obtained by WA, TA, TIES, and ADA, in the 8-task vision setting with ViT-B/32.
    Task identification becomes highly accurate from middle layers onward, and all representation sources remain competitive across deeper layers.}
    \label{fig:layer_index_task_identification}
\end{figure}

\noindent\textbf{Results.}
\cref{fig:layer_index_task_identification} shows two consistent trends.
First, task identification remains robust across all representation sources once the feature layer is sufficiently deep.
In particular, strong task separation already emerges at middle layers, and performance becomes highly accurate from the middle-to-late layers for all five sources.
This indicates that, although the main experiments use the final-layer representation, task identification does not require forwarding the merged model all the way to the last layer; intermediate features are already sufficient for accurate task prediction.

Second, representations extracted from the pre-trained model also provide competitive task signatures and test features.
The importance of PRE here is primarily practical rather than absolute performance.
Unlike merged-model features, pre-trained features can be extracted and stored at the beginning of task-specific fine-tuning.
This helps alleviate the practical concern that task signatures must otherwise be recomputed only after merged-model construction.

\subsection{Number of reference inputs and subspace dimension}
\label{app:bank_size}

\noindent\textbf{Setting.}
This subsection studies two design choices in task identification:
the number of reference inputs used to construct task signatures, and the retained singular-vector ratio $k$ used to define each task subspace.
All experiments use the 8-task vision setting with a CLIP ViT-B/32 backbone.
We report both task identification accuracy and normalized accuracy after recovery.

\begin{figure}[t]
\centering
\begin{minipage}{0.49\linewidth}
    \centering
    \includegraphics[width=\linewidth]{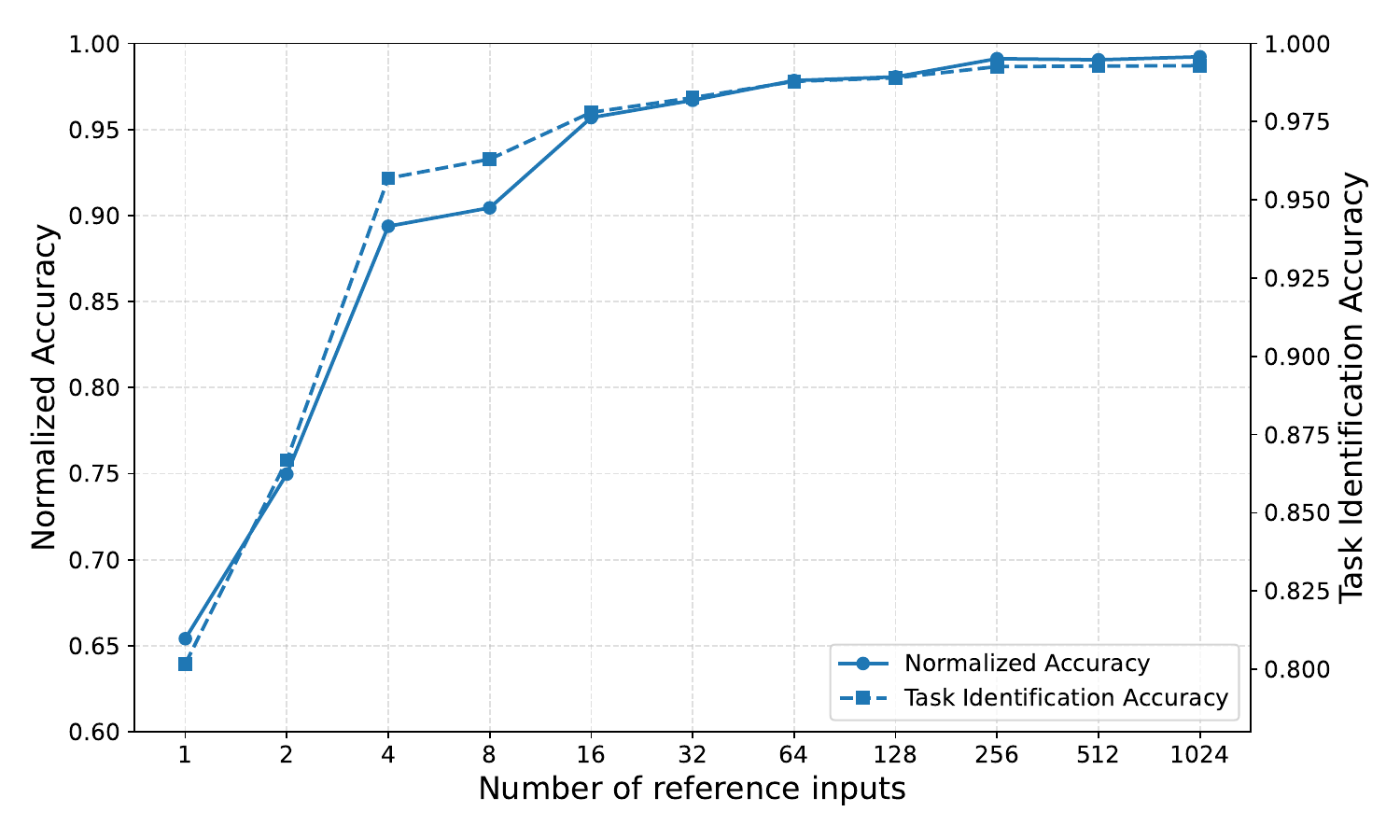}
\end{minipage}
\hfill
\begin{minipage}{0.49\linewidth}
    \centering
    \includegraphics[width=\linewidth]{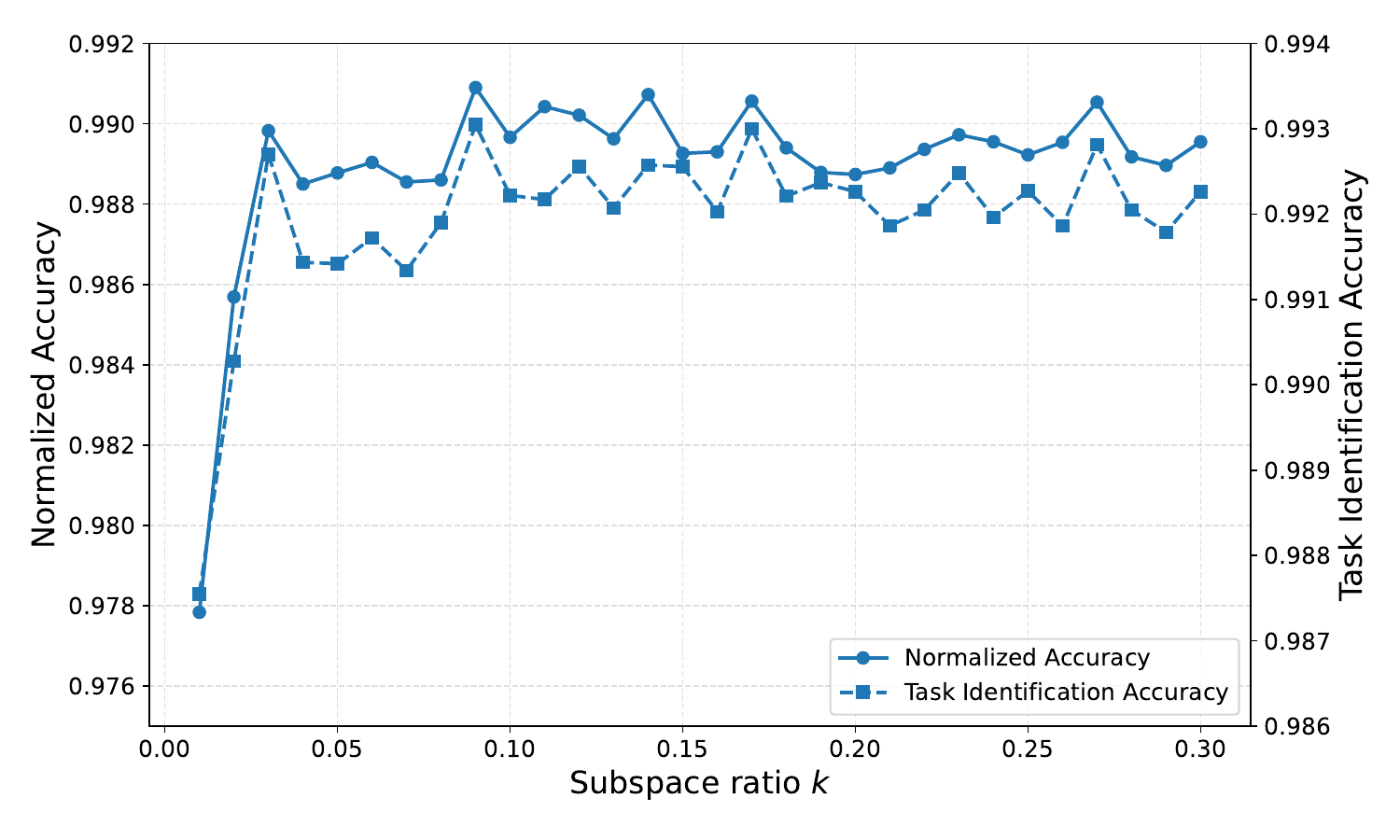}
\end{minipage}
\caption{\textbf{Task identification design choices.}
Left: effect of the number of reference inputs per task.
Right: effect of the singular-vector ratio $k$ used to define each task subspace.
Both plots report task identification accuracy and normalized accuracy in the 8-task vision setting with ViT-B/32.}
\label{fig:bank_size_and_subspace_ratio}
\end{figure}

\noindent\textbf{Results.}
\cref{fig:bank_size_and_subspace_ratio} shows that task identification improves as the number of reference inputs increases, but quickly reaches a regime of diminishing returns.
These results indicate that task signatures do not require a large reference set, and that even substantially fewer reference inputs than the default setting can already provide strong task identification and recovery.
In the remaining experiments, we use 64 reference inputs per task as a stable default choice.

The right plot shows that performance remains stable across a broad range of subspace ratios $k$.
Very small values degrade performance, but once $k$ is not too small, both task identification accuracy and normalized accuracy remain nearly constant.
Based on this trade-off, we use $k=0.1$ in the main experiments, since it provides strong and stable performance while keeping task signatures compact.

\subsection{Rank}
\label{app:rank_ablation}

We study how the low-rank dimension $r$ affects recovery quality and memory.
\cref{fig:vitb32_rank_tradeoff} reports normalized accuracy (mean accuracy divided by the corresponding individual-expert accuracy) and required memory ratio (memory relative to the base model) on ViT-B/32.
Across a sweep of $r$, \method{} exceeds $99\%$ normalized accuracy once $r\!\ge\!128$, and the gains saturate beyond $r\!=\!256$.
Since the parameter and activation costs grow roughly linearly with $r$ through the factors $a\times r$ and $r\times b$, we set $r\!=\!256$ in the main experiments to balance accuracy and memory.

\begin{figure}[t!]
    \centering
    \includegraphics[width=0.7\linewidth]{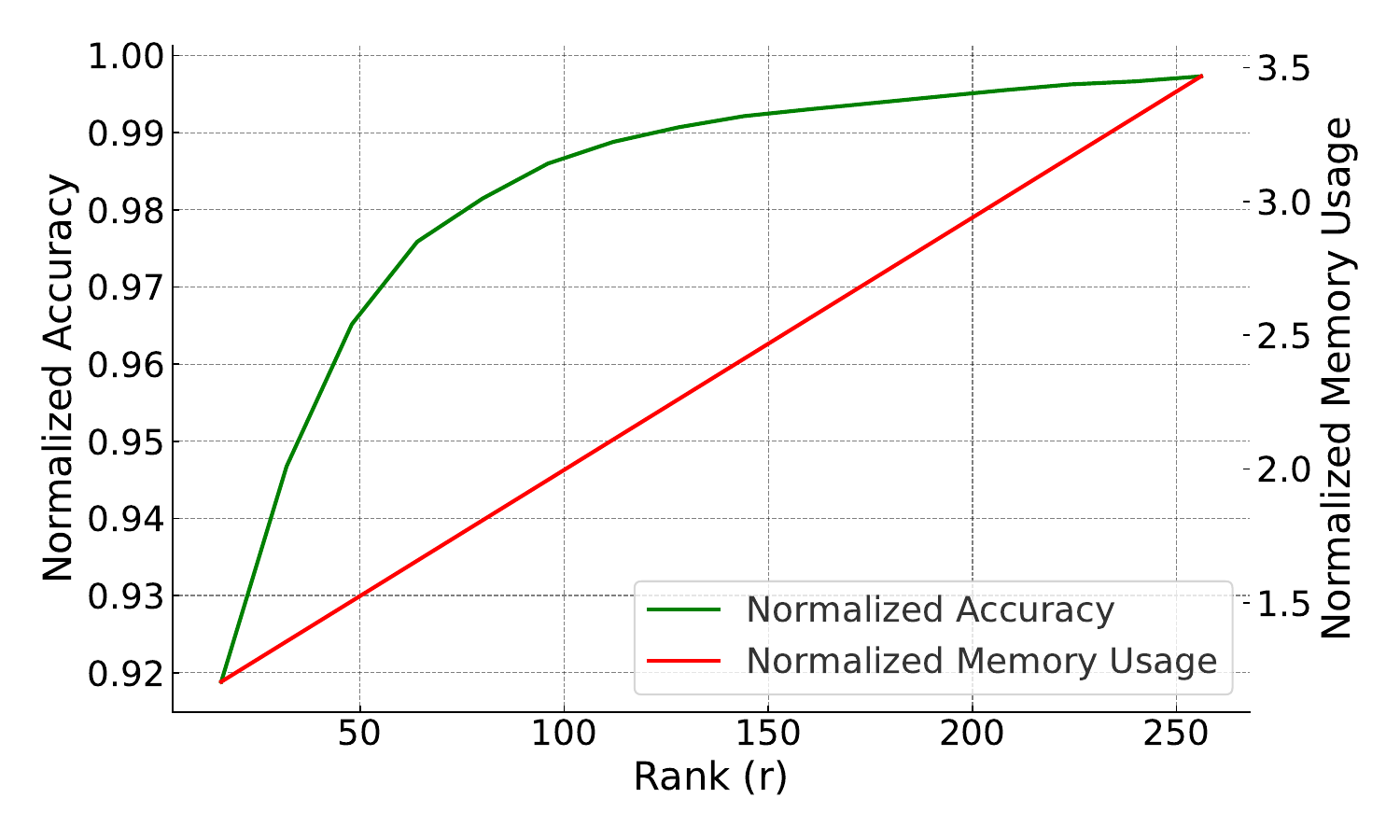}
    \caption{
    Normalized accuracy and required memory ratio as a function of rank $r$ on ViT-B/32.
    }
    \label{fig:vitb32_rank_tradeoff}
\end{figure}

\subsection{Random seed}
\label{app:random_seed}

\begin{table}[h]
\centering
\caption{Normalized accuracy for each random seed and number of tasks across all tasks in the computer vision task with ViT-B/32 CLIP backbone.
The bottom row reports the sample mean and its standard deviation (Mean $\pm$ std) over the five seeds.}

\resizebox{0.8\textwidth}{!}{
\begin{tabular}{cccc}
\toprule
\textbf{Seed} & \textbf{8 tasks} & \textbf{14 tasks} & \textbf{20 tasks} \\
\midrule
0 & 99.2251 & 98.7085 & 98.0569 \\
1 & 99.2021 & 98.6676 & 98.1559 \\
2 & 99.2138 & 98.6494 & 98.0502 \\
3 & 99.1702 & 98.7632 & 97.9032 \\
4 & 99.2792 & 98.5788 & 98.0546 \\
\midrule
\textbf{Mean $\pm$ std} & $99.2181 \pm 0.0399$ & $98.6735 \pm 0.0687$ & $98.0442 \pm 0.0904$ \\
\bottomrule
\end{tabular}
}

\label{tab:seed_results}
\end{table}

To assess the stability and robustness of \method{} with respect to initialization and other sources of randomness in the training process, we conducted experiments across multiple random seeds.
\cref{tab:seed_results} presents the normalized accuracy of \method{} on the ViT-B/32 CLIP backbone for computer vision task suites of 8, 14, and 20 tasks, evaluated over five different random seeds (0 through 4).
The results demonstrate a high degree of consistency across seeds.
This low variance across different seeds indicates that the performance of \method{} is not highly sensitive to the specific random initialization used, suggesting reliable and reproducible outcomes.

\subsection{Cosine similarity as an objective}
\label{app:cosine_objective}

Prior work in model merging sometimes uses cosine similarity to measure alignment between task-dependent directions in parameter space~\cite{yang2024adamerging, huang2024emr, xiong2024multi, davari2024model}.
Motivated by this, we study an alternative cosine-based objective for \method{} using the same notation as the main paper.
Since the reconstruction objective in \cref{eq:retex_training_loss_offset} compares the recovered parameters \(\vtheta_{\text{merge}}+\hat{\bm{\beta}}_t\) with the target expert \(\vtheta_t\), the corresponding target direction is given by \(\vtheta_{\text{merge}}-\vtheta_t\).
To keep the cosine objective aligned with the recovery direction of \(\hat{\bm{\beta}}_t\), we compare \(-\hat{\bm{\beta}}_t\) with \(\vtheta_{\text{merge}}-\vtheta_t\):
\begin{equation}
\mathcal{L}_{\mathrm{cos}}
=
\mathbb{E}_{t\sim\mathcal{U}(1,T)}
\left[
1-
\frac{
(-\hat{\bm{\beta}}_t)\cdot(\vtheta_{\text{merge}}-\vtheta_t)
}{
\|\hat{\bm{\beta}}_t\|_2\,\|\vtheta_{\text{merge}}-\vtheta_t\|_2
}
\right].
\label{eq:cosine_similarity_theta_loss}
\end{equation}
We also study the combined objective
\(
\mathcal{L}_{\mathrm{cos}}+\lambda \mathcal{L},
\)
where \(\mathcal{L}\) denotes the reconstruction loss in \cref{eq:retex_training_loss_offset}.

\cref{fig:cosine_vs_l2_loss_ablation} shows that cosine similarity alone is not sufficient for high-fidelity task expert recovery.
Performance improves substantially once the reconstruction loss is included, and quickly approaches the behavior of the standard objective.
This indicates that the reconstruction loss in \cref{eq:retex_training_loss_offset} is the main driver of effective recovery, while cosine similarity provides limited additional benefit.

\begin{figure}[h!]
    \centering
    \includegraphics[width=0.7\columnwidth]{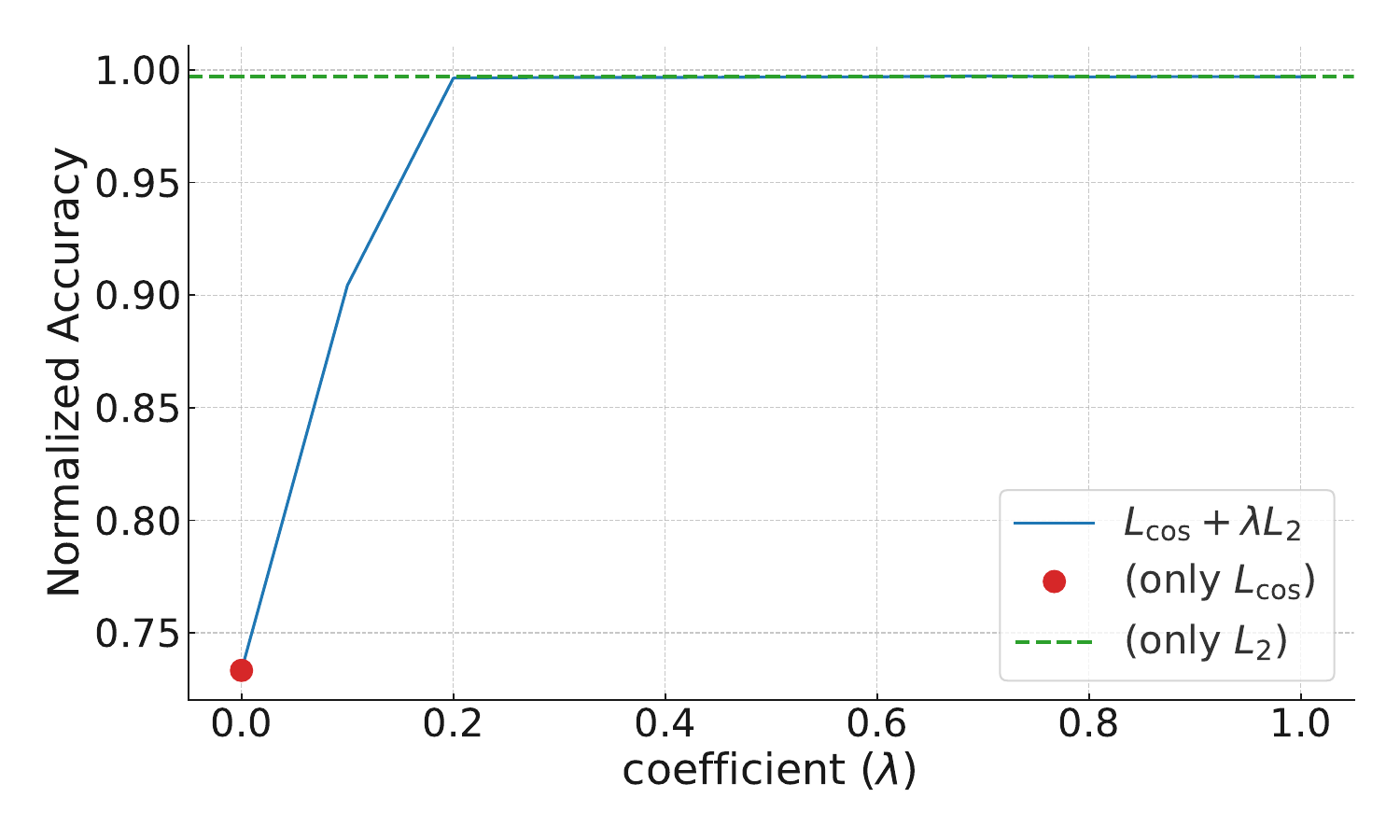}
    \caption{
    Normalized accuracy of \method{} when trained with cosine similarity, the standard reconstruction objective, and their combination.
    }
    \label{fig:cosine_vs_l2_loss_ablation}
\end{figure}

\subsection{Impact of factorization order in offset generation}
\label{app:factorization_order}

In the main paper, the recovered layer-wise offset is factorized as
\(
\hat{\bm{\beta}}_t^{(l)}=\bm{\beta}_{t,g}^{(l)}\bm{\beta}_{s}^{(l)},
\)
where the task-conditioned factor \(\bm{\beta}_{t,g}^{(l)}\) is generated by \(h^{(l)}\) and \(\bm{\beta}_{s}^{(l)}\) is shared across tasks.
Here we study the reverse order,
\(
\hat{\bm{\beta}}_t^{(l)}=\bm{\beta}_{s,\mathrm{alt}}^{(l)}\bm{\beta}_{t,g,\mathrm{alt}}^{(l)},
\)
where the shared factor is placed first and the task-conditioned factor is placed second.
This ablation tests whether the ordering of the two low-rank factors affects the trade-off between recovery quality and memory.

\cref{fig:factorization_order_comparison} shows that reversing the factorization order can reduce memory, but generally leads to lower normalized accuracy.
The standard formulation used in \method{} provides a better accuracy--memory trade-off across ranks, indicating that the ordering of the task-conditioned and shared factors is important for effective recovery.

\begin{figure}[h!]
    \centering
    \includegraphics[width=0.7\columnwidth]{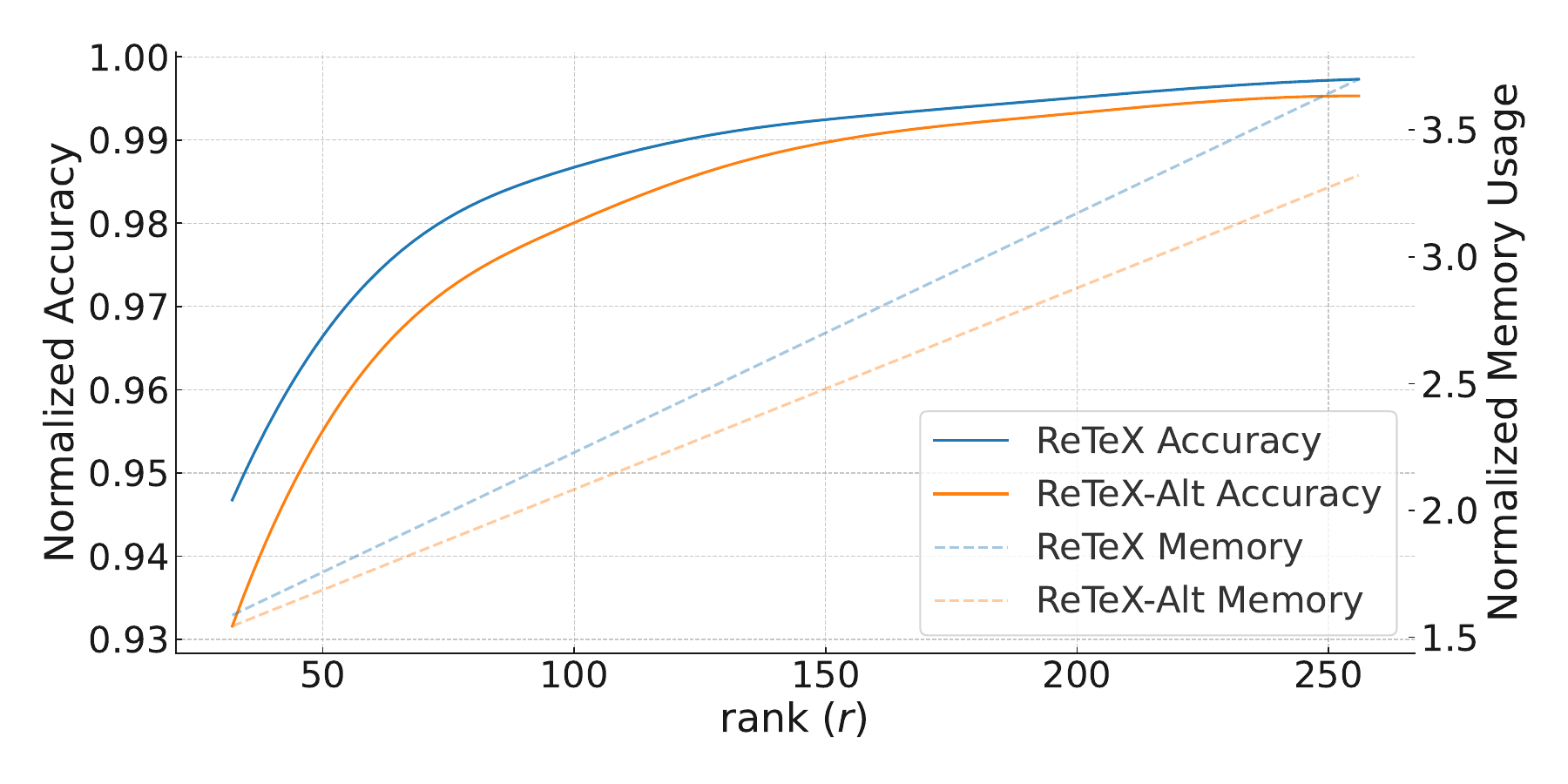}
    \caption{
    Comparison of normalized accuracy and normalized memory usage for the standard \method{} factorization and the reversed-order alternative.
    }
    \label{fig:factorization_order_comparison}
\end{figure}

\subsection{\method{} application with shared layers}
\label{app:shared_layers}

The default \method{} formulation predicts offsets independently for each layer.
To study whether this layer-wise design is necessary, we construct a shared-layer variant in which layers with the same parameter shape reuse the same offset generator and shared offset factor.
More concretely, for layers that belong to the same shape group \(g\), the recovered offset is written as
\(
\hat{\bm{\beta}}_t^{(l)}=h^{(g)}(\ve_t)\bm{\beta}_{s}^{(g)},
\)
so that both the generator and the shared factor are tied across all layers in the group.
This variant tests whether recovery can be simplified by sharing parameters across structurally similar layers.

\cref{fig:retex_compare_layer_share} shows that sharing across layers consistently weakens recovery quality.
Even though parameter sharing can reduce model complexity, it prevents the recovery module from modeling layer-specific interference patterns.
These results support the design choice of estimating offsets independently for each layer in the main method.

\begin{figure}[h!]
    \centering
    \includegraphics[width=0.7\linewidth]{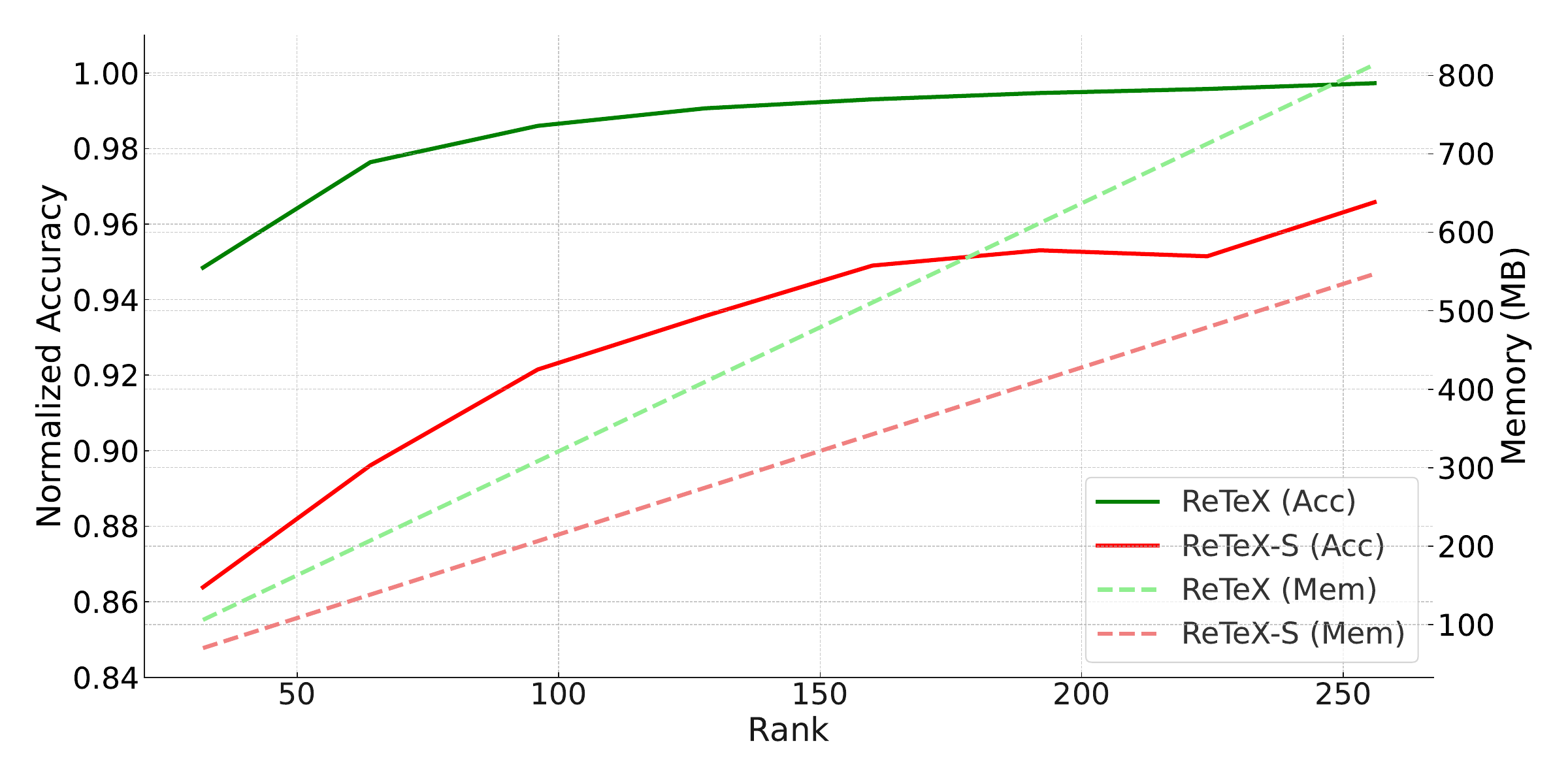}
    \caption{
    Comparison between the default layer-wise \method{} and the shared-layer variant.
    }
    \label{fig:retex_compare_layer_share}
\end{figure}
\FloatBarrier

\section{Additional analyses}
\label{app:additional_analyses}

This section provides additional analyses that complement the main results and clarify several empirical aspects of \method{}.

\subsection{Soft-recovery interpolation mechanism}
\label{app:soft_recovery_interpolation}

The unseen-task setting uses soft recovery over seen-task experts rather than committing to a single hard task ID.
This design allows \method{} to interpolate knowledge from multiple seen experts when the input does not clearly belong to one of the seen tasks.
To check whether the improvement depends on a particular interpolation strategy, we compare three interpolation mechanisms under the same training budget.

\begin{table}[h!]
\caption{\textbf{Comparison of Stage 2 interpolation mechanisms under the same training budget.}
\label{tab:soft_recovery_interpolation}
}
\centering
\setlength{\tabcolsep}{4pt}
\begin{tabular}{llcc}
\toprule
\textbf{Sampler} & \textbf{SEEN} & \textbf{UNSEEN} \\ 
\midrule
Uniform + softmax & 90.4 & 62.5 \\
Gaussian + softmax & 90.4 & 62.5 \\
Dirichlet ($\alpha = 1$) & $\mathbf{90.4}$ & $\mathbf{63.4}$ \\
\bottomrule
\end{tabular}
\end{table}

As shown in \cref{tab:soft_recovery_interpolation}, the Dirichlet sampler achieves the best unseen-task accuracy while preserving seen-task accuracy.
This result supports using Dirichlet-sampled soft targets to train recovery over the simplex of seen experts.
It also suggests that the benefit of soft recovery comes from learning smooth expert interpolation, rather than from a specific hard assignment to a single seen task.

\subsection{Task identification under small reference sets}
\label{app:small_reference_sets}

We next study how the number of reference inputs affects task-agnostic inference.
The task identification module builds a task memory bank from a small number of reference inputs per task.
Therefore, it is important to verify whether \method{} remains reliable when only a limited reference set is available.

We compare \method{} with Twin-Merging, which also handles task-ID-unknown inputs but relies on an input-dependent routing mechanism.
This comparison highlights whether the proposed task identification and recovery pipeline can remain competitive even with a small reference set.

\begin{figure}[t]
\centering
\includegraphics[width=0.78\linewidth]{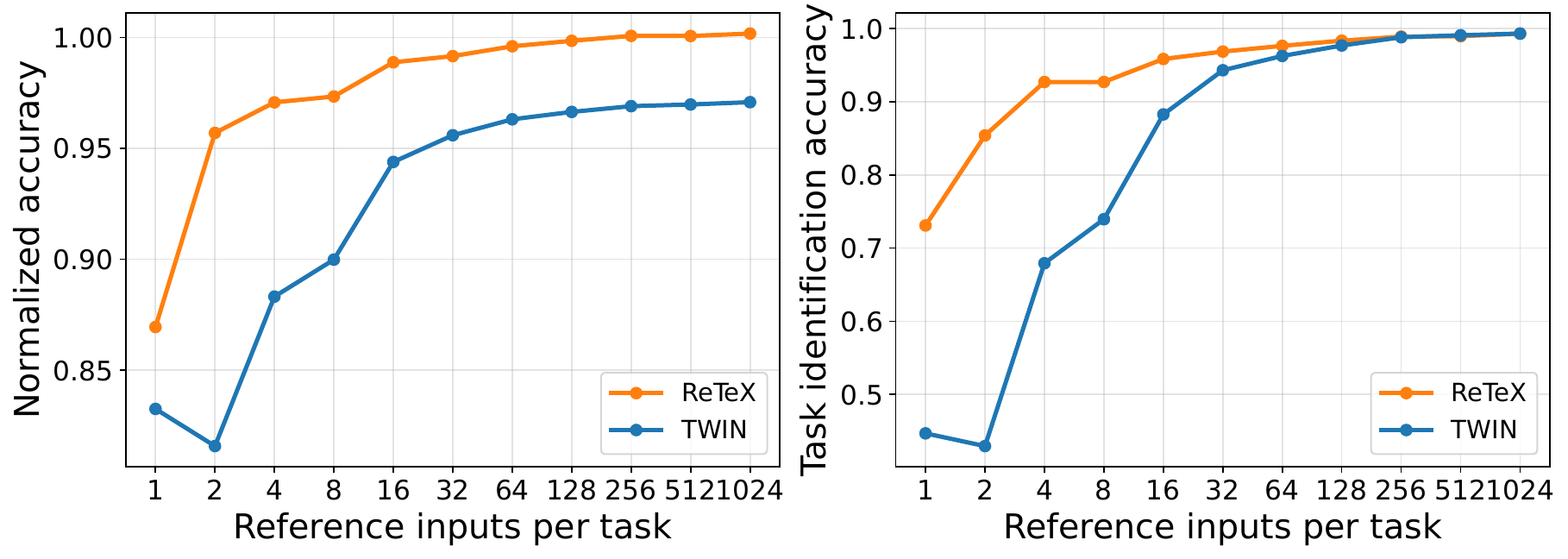}
\caption{Reference-input sweep for task-agnostic inference. \method{} remains competitive with few reference inputs and improves as the task memory bank becomes more reliable.}
\label{fig:reference_input_sweep}
\end{figure}

As shown in \cref{fig:reference_input_sweep}, \method{} remains competitive with Twin-Merging even when the task memory bank is constructed with only a small number of reference inputs.
As the number of reference inputs increases, task signatures become more reliable, leading to improved task identification and final accuracy.
These results indicate that \method{} does not require a large reference set to operate effectively under task-ID-unknown inference.

\subsection{Prediction-flip analysis for WSC}
\label{app:wsc_flip_analysis}

To better understand why WSC remains challenging in the NLP suite, we analyze prediction stability around the task expert.
Prediction-flip rate measures the fraction of validation samples whose predictions differ from the task expert when parameters are perturbed around the expert solution.
A higher flip rate indicates that small parameter deviations from the expert can cause larger prediction-level changes.

\begin{figure}[t]
\centering
\includegraphics[width=0.88\linewidth]{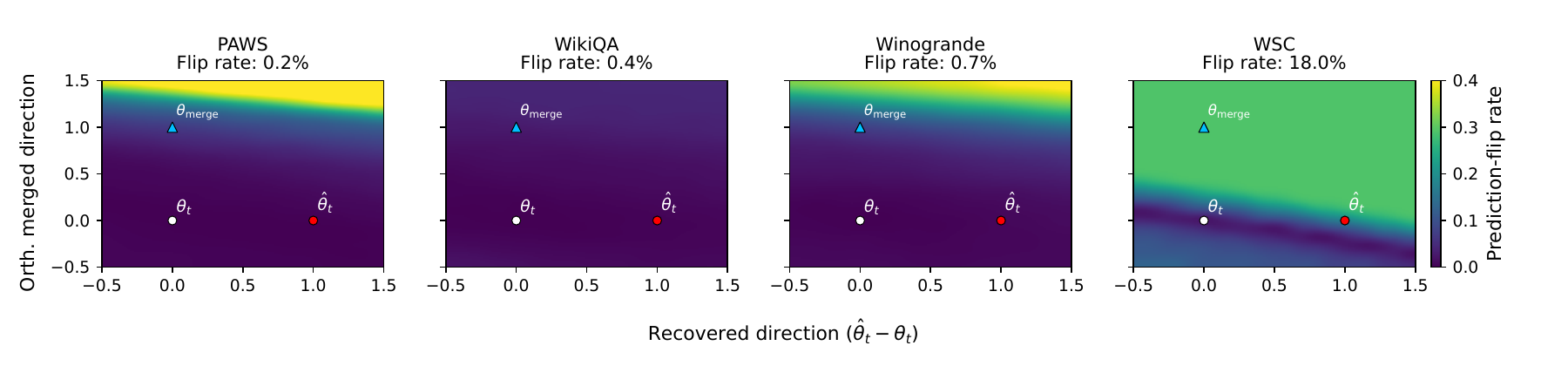}
\caption{Prediction-flip landscapes around task experts. Compared with high-recovery tasks such as PAWS, WikiQA, and Winogrande, WSC exhibits a much larger flip rate, indicating higher sensitivity to small parameter recovery errors.}
\label{fig:wsc_flip_landscape}
\end{figure}

As shown in \cref{fig:wsc_flip_landscape}, WSC exhibits a substantially higher prediction-flip rate than the other analyzed NLP tasks.
This suggests that WSC is more sensitive to small parameter recovery errors.
Therefore, even when the recovered parameters are close to the task expert in parameter space, small deviations can lead to larger changes in final predictions, explaining the remaining performance gap on WSC.
\FloatBarrier

\section{Limitations}
\label{app:limitations}

Despite its effectiveness, \method{} still has several limitations.
First, unlike purely static merging methods that directly deploy a single merged checkpoint, \method{} requires an additional recovery-training stage after the merged checkpoint is constructed.
Although this training is substantially lighter than full multi-task retraining and uses parameter supervision only, it still makes the overall pipeline more involved than directly deploying a merged checkpoint.
At the same time, this additional cost can still be more practical than many dynamic merging approaches that train routers or gating modules with task data and maintain multiple task-specific components at inference~\cite{lu2024twin,tang2024merging,ye2025dynamic}.
In particular, \method{} does not require storing or revisiting full task datasets during recovery training, which keeps the post-merging training stage relatively lightweight.

Second, task-agnostic deployment in \method{} still requires a small set of reference inputs to construct the task memory bank.
This means that \method{} is not fully data-free when task identification is required.
However, this requirement is less restrictive than in many router-based dynamic merging methods, which typically require task data during router training or model-merging stages~\cite{lu2024twin,tang2024merging,ye2025dynamic}.
In contrast, the reference inputs in \method{} are only used offline to compute compact task signatures, and these signatures can be prepared together with task experts before model merging.
As a result, \method{} does not need to retain raw data after expert release or during the merging process itself.
Nevertheless, dependence on reference inputs remains a practical limitation, and removing the need for such data while preserving reliable task identification is an important direction for future work.
\FloatBarrier

\end{document}